\documentclass[10pt,peerreview,draftclsnofoot,onecolumn,a4paper]{IEEEtran}
\usepackage{amsmath,amssymb}
\usepackage{microtype,charter}
\usepackage{graphicx,epsfig,url,epstopdf}
\usepackage[caption=false,font=footnotesize]{subfig}

\usepackage{multirow,bigdelim}
\usepackage{url}
\usepackage{xcolor}
\usepackage[T1]{fontenc}
\usepackage{palatino}
\definecolor{Red}{rgb}{1,0,0}
\definecolor{Blue}{rgb}{0,0,1}
\usepackage{amsmath,amsthm,amssymb}
\usepackage{graphicx,epstopdf,epsfig}
\usepackage{url,framed}
\usepackage{bm}
\usepackage{microtype}
\usepackage{mathtools}
\usepackage{multirow}
\usepackage[caption=false,font=footnotesize]{subfig}
\usepackage[linesnumbered,ruled]{algorithm2e}
\usepackage{array}
\usepackage{float}
\restylefloat{table}
\usepackage{breqn}
\usepackage{siunitx}
\usepackage{algpseudocode}
\usepackage{lipsum}
\usepackage{placeins}
\usepackage{blindtext}
\usepackage{tikz}
\usepackage[round]{natbib}

\newcommand*\circled[1]{\tikz[baseline=(char.base)]{
            \node[shape=circle,draw,inner sep=2pt] (char) {#1};}}

\newtheorem{theorem}{Theorem}[section]
\newtheorem{proposition}[theorem]{Proposition}

\def \x {\mathbf{x}}

\def \t {\mathbf{t}}

\def \e {\mathbf{e}}

\def \del {\bm{\delta}}

\def \bX {\mathbf{X}}

\def \bA {\mathbf{A}}

\def \bV {\mathbf{V}}
\def \bU {\mathbf{U}}

\def \Lam {\bm{\Lambda}}

\def \R {\mathbb{R}}

\def \P {\mathcal{P}}

\def\x{\boldsymbol{x}}
\def\t{\boldsymbol{t}}
\def\e{\boldsymbol{e}}
\def\del{\boldsymbol{d}}

\def\R{\mathbf{R}}
\def\T{\mathbf{T}}
\def\B{\mathbf{B}}
\def\L{\mathcal{L}}
\def\C{\mathbf{C}}
\def\D{\mathbf{D}}
\def\G{\mathbf{G}}
\def\H{\mathbf{H}}
\def\Lam{\bm{\Lambda}}
\def\A{\mathbf{A}}
\def\X{\mathbf{X}}

\title{Least-squares registration of point sets over $\mathbb{SE}(d)$ using closed-form projections}

\author{\IEEEauthorblockN{Sk. Miraj Ahmed\IEEEauthorrefmark{1},
Niladri Ranjan Das\IEEEauthorrefmark{2} and Kunal Narayan Chaudhury\IEEEauthorrefmark{2} \IEEEauthorrefmark{1}\IEEEauthorrefmark{1}}

\IEEEauthorblockA{
\IEEEauthorrefmark{1}University of California, Riverside, USA.,
\IEEEauthorrefmark{2}Indian Institute of Science, India.}
\thanks{\IEEEauthorrefmark{1}\IEEEauthorrefmark{1}Correspondence: kunal@ee.iisc.ac.in (Kunal Narayan Chaudhury). This work was partially supported by EMR grant SERB/F/6047/2016-2017 from DST-SERB, Government of India.}}

\begin{document}
\maketitle
\begin{abstract}
Consider the problem of registering multiple point sets in some $d$-dimensional space using rotations and translations. Assume that there are sets with common points, and moreover the pairwise correspondences are known for such sets. We consider a least-squares formulation of this problem, where the variables are the transforms associated with the point sets. The present novelty is that we reduce this nonconvex problem to an optimization over the positive semidefinite cone, where the objective is linear but the constraints are nevertheless nonconvex. We propose to solve this using variable splitting and the alternating directions method of multipliers (ADMM). Due to the linearity of the objective and the structure of constraints, the ADMM subproblems are given by projections with closed-form solutions. In particular, for $m$ point sets, the dominant cost per iteration is the partial eigendecomposition of an $md \times md$ matrix, and $m-1$ singular value decompositions of $d \times d$ matrices. We empirically show that for appropriate parameter settings, the proposed solver has a large convergence basin and is stable under perturbations. As applications, we use our method for 2D shape matching and 3D multiview registration. In either application, we model the shapes/scans as point sets and determine the pairwise correspondences using ICP. In particular, our algorithm compares favorably with existing methods for multiview reconstruction in terms of  timing and accuracy.
\end{abstract}


\section{Introduction}

Surface reconstruction from multiview scans has applications in computer-aided design, computer graphics, computer vision, medical imaging, and virtual reality. A range scanner is typically used to scan the object from different views. The scans are extracted by fixing the object and moving the scanner around, or by fixing the scanner and rotating the object on a turntable. Each scan is represented using a mesh, composed of vertices, faces, and normals (\cite{3Dscanrep}). The vertices are simply the sampled points, the faces encode the connectivity between vertices, and the normals can be used to estimate the curvature (\cite{MLSalgo}). As with most existing reconstruction methods, we will use just the vertex information, i.e., each scan is represented as a \textit{point set}. The computational problem in this regard is to register the scans using translations and rotations.
The problem has two components, namely, finding the point-to-point correspondences between scans and determining the alignment of each scan using the found correspondences. The former is harder to deal with due to its combinatorial nature (\cite{Xing2013}).

On the other hand, if the point-to-point correspondence is known for two scans, then they can be registered simply using singular value decomposition (\cite{Umeyama1991}). However, the correspondences are not known in practice and have to be inferred from the scan data. A natural approach is to solve the correspondence and registration problems in an alternating manner.
This approach is used in classical ICP (Iterative Closest Point) and its variants (\cite{Besl1992,Rusinkiewicz2001,Zinsser2003}), where the correspondences are obtained by aligning the scans and matching nearest points. The two steps are repeated until convergence. However, ICP requires a good estimate for the initial alignment. Deformations (scaling/stretching) in the scans further deteriorates the performance of ICP. To address this, \cite{lieicp2009,scaledICPlie2009} proposed an unified mathematical model for registration, by combining ICP with optimization over Lie groups. Classical ICP is also sensitive to outliers and partial overlaps between scans. Robust variants of ICP have been proposed that can reject outliers (\cite{TrimmedICP,Zinsser2003,RobustICP2007,Robusticp2011,Zhang2011,DONG201467,Yang2016}). However, this involves pruning or reweighting the correspondences, which can be computation intensive. A more efficient alternative based on sparsity inducing norms was explored in \cite{SparseICP2013}.

The situation is more difficult with multiple scans. A straightforward approach is \textit{sequential registration} in which multiple scans are aligned one at a time using ICP. However, sequential registration is prone to error propagation, and the situation gets worse in the presence of outliers. A more robust alternative is to take into account the pairwise correspondences between scans (obtained using ICP), integrate them in a single objective (cost function), and jointly optimize the transforms with respect to the objective.  This way we can prevent error propagation and distribute the registration error across the scans. Our approach is based on this so-called \textit{global registration}.

\subsection{Global Registration}

There exist several methods for global registration in the literature. These can be divided into two broad classes depending on whether the registration is performed in frame or point space. Frame-space methods use the relative transform between scans (\cite{Sharp2002,Fusiello2002,Torsello2011,Govindu2004,Pooja2014}), whereas point-space methods work directly with the local coordinates (\cite{Pennec1996,Bergevin1996,Benjemaa1999,Bennamoun2001}). In one of the earliest frame-space methods (\cite{Sharp2002}), the view graph is decomposed into cycles such that the optimal transforms for each cycle can be computed in closed-form. In \cite{Fusiello2002}, the rotations are parameterized using quaternions and the registration error is distributed evenly across all scans. In \cite{Torsello2011}, dual quaternions are used to represent the transforms and geodesic averaging is used to denoise the relative transforms. In \cite{Govindu2004} and \cite{Pooja2014}, the authors used Lie-algebraic methods for averaging relative transforms. In \cite{Arrigoni2015} and \cite{Bernard2015}, it is shown that the null space of an appropriate matrix (constructed from the relative transforms) can be used to extract cycle-consistent transforms. In \cite{Arrigoni2016}, the authors proposed to use low-rank matrix completion for  global registration.
The pioneering point-space methods include \cite{Pennec1996,Bergevin1996,Benjemaa1999,Bennamoun2001}. 
The authors in \cite{Pennec1996} alternately computed an average shape from the scans and aligned the scans against this shape.
In \cite{Bergevin1996}, the scans are repeatedly selected and registered against other scans in a global reference frame.
The method in \cite{Bergevin1996} is enhanced in \cite{Benjemaa1999} by accelerating the correspondence search using multi-z buffer technique, and by updating the surface positions immediately after the transform computation in each iteration. 
In \cite{Raghuramu2015}, the registration error is measured using the $\ell_1$ norm instead of the $\ell_2$ norm. 
A generalization of two-view ICP for multiple views was proposed in \cite{Fantoni2012}. In \cite{Toldo2010}, the authors combined ICP with general procrustes analysis for multiset registration. In \cite{Evangelidis2014}, the point sets are modeled using the Gaussian mixture model and registration is performed using the EM algorithm.

\subsection{Contribution}

In this paper, we consider the abstract problem of registering $m$ point sets over the group of rotations and translations, the special Euclidean group $\mathbb{SE}(d)$, given the local coordinates of points in each point set. We assume that there exist sets with common points, and that the point-to-point correspondences are known for such sets.
We consider a least-squares formulation for this problem, where the variables are the $\mathbb{SE}(d)$-transforms associated with the point sets. Remarkably, while this problem is inherently nonconvex, its solution (global optimum) can be computed using SVD when $m=2$ (\cite{Umeyama1991}). Unfortunately, such a solution is not available when $m\geq3$, and one must resort to some iterative method. In this regard, the precise contributions are as follows:
\begin{itemize}
\item We observe that the constrained least-squares optimization can be reduced to an SDP (semidefinite program), where the objective is linear but the constraints are nonconvex. The solution of the original least-squares problem can be derived from the SDP solution via a linear map. 

\item We propose to solve the SDP using variable splitting and ADMM (alternating directions method of multipliers). In particular, the variable splitting is such that the resulting subproblems are given by matrix projections with closed-form solutions. For $m$ point sets, the per-iteration cost is essentially the partial eigendecomposition of an $md \times md$ matrix, and $m-1$ SVD-based projections onto the rotation group $\mathbb{SO}(d)$.

\item We apply the proposed algorithm to the motivating problem, namely, the registration of multiview scans extracted from a three-dimensional surface. We model each scan as a point set and determine the pairwise correspondences using Picky-ICP \cite{Zinsser2003}. Since the dominating computation per iteration is a partial eigendecomposition, we can scale our iterative solver to practical problems involving large number of scans. To demonstrate the applicability of the algorithm beyond multiview registration, we also use it for 2D shape matching.

\end{itemize}

Based on toy examples and simulated multiview data, we empirically demonstrate that the proposed solver has a large basin for global convergence and exhibits fast convergence for appropriate parameter settings. We present simulation results on scans from the Stanford repository and compare them with existing multiview methods.  We observe that our algorithm is generally more robust to noise (both in the coordinates and the correspondences) and is quite fast. 

\subsection{Related Work}

The present contribution is an extension of \cite{Ahmed2017} where the registration is performed over the Euclidean group $\mathbb{E}(d)$, which includes reflections along with translations and rotations. 
Since the optimization is performed over a subgroup of $\mathbb{E}(d)$ in this paper, namely $\mathbb{SE}(d)$, we are required to introduce additional constraints. The novelty in this regard is the proposed variable splitting, whereby we are able to add more constraints and yet retain the efficiency of the ADMM solver in \cite{Ahmed2017}. 
Similar to \cite{Ahmed2017}, the subproblems of the proposed ADMM solver are given by closed-form projections. 
In the context of multiview registration, the ADMM algorithm in \cite{Ahmed2017} works well for small noise. However, when the noise is large, the  optimal transforms returned by the algorithm often include reflections. This is possible since the domain is $\mathbb{E}(d)$, whereby both rotations and reflections are allowed. The present proposal fixes this problem by forcing the optimal transforms to be in $\mathbb{SE}(d)$. 

We note that least-squares formulations of global registration have been considered in \cite{Bennamoun2001,Krishnan2005,Chaudhury2015}. The optimization is performed over  $\mathbb{SE}(d)$ in \cite{Bennamoun2001,Krishnan2005}, and over  $\mathbb{E}(d)$ in 
\cite{Chaudhury2015}. The optimal rotations are iteratively computed using SVD in \cite{Bennamoun2001}, while Gauss-Newton iterations are performed on the $\mathbb{SO}(3)$ manifold to compute the optimal rotations in \cite{Krishnan2005}. The solver in \cite{Krishnan2005} was computationally enhanced  in \cite{Bonarrigo2011}, and later it was shown in \cite{Mateo2014} that the problem can be posed within a Bayesian framework. In a different direction, it was observed in \cite{Chaudhury2015} that least-squares registration over $\mathbb{E}(d)$ can be reduced to a rank-constrained SDP; this can further be relaxed to a convex SDP with provable tightness and stability guarantees. 

We note that though ADMM has found wide applications in convex programming \cite{Boyd2011}, the fact that it works well with nonconvex problems was reported more recently; e.g., see \cite{Chartrand2013,Miksik2014,Wang2015,Diamond2017}. Moreover, while ADMM comes with with strong convergence guarantees for convex problems, analysis of nonconvex ADMM is still in its infancy (\cite{Wang2015,Hong2016}). 

\subsection{Organization}

The paper is organized as follows. In Section \ref{sec:Reg}, we give the mathematical description of the registration problem and the algorithmic solution. 
We empirically analyze the algorithm in Section \ref{sec:Numerics} using toy examples and simulated data. 
In Section \ref{sec:MR}, we present results on multiview registration using scans from the Stanford repository and compare our results with existing methods.  We conclude the paper in Section \ref{sec:Conc}.

\subsection{Notation}

We use $[m]$ to denote $\{1,\ldots,m\}$. The Euclidean norm is denoted by $\lVert \x \rVert_2$.
Symmetric matrices in $\mathbb{R}^{k \times k}$ are represented by $\mathbb{S}^k$, and the subset of positive semidefinite  matrices by $\mathbb{S}^k_+$. We use the standard inner-product on $\mathbb{R}^{k \times k}$ given by $\langle \A, \B\rangle=\text{trace}(\A^\top \!\B)$; the norm induced by this inner-product is  $\Vert \A\Vert_{\mathrm{F}}=\langle \A, \A\rangle^{1/2}$. The orthogonal group $\mathbb{O}(d)$ consists of matrices $\bU \in \mathbb{R}^{d \times d}$ such that $\bU^\top \! \bU = \mathbf{I}$, where $\mathbf{I} \in \mathbb{R}^{d \times d}$ is the identity matrix. The special orthogonal (or rotation) group $\mathbb{SO}(d)$ consists of matrices $\R \in \mathbb{R}^{d \times d}$ such that
\begin{equation}
\label{SOd}
 \R \in  \mathbb{O}(d) \quad  \text{ and } \quad \mathrm{det}(\R)=1.
\end{equation}
The special Euclidean group  $\mathbb{SE}(d)$ consists of orientation-preserving rigid motions in $d$-dimensions, which are simply rotations and translations. An element of $\mathbb{SE}(d)$ is represented using a pair $(\R,\t)$, where $\t \in \mathbb{R}^d$ and $\R \in \mathbb{SO}(d)$.

\section{Registration over $\mathbb{SE}(d)$}
\label{sec:Reg}

We generally consider the global registration of point sets over $\mathbb{SE}(d)$, though we are mainly interested in $d=2$ and $d=3$. 
Suppose we have $m$ point sets $\P_1,\ldots,\P_m \subset \mathbb{R}^d$.
Two point sets are said to overlap if they have at least one common point.
In general, not every pair of sets overlap. 
For $i,j \in [m]$, we use $i \sim j$ to mean that $\P_i \cap \P_j$ is not empty, and we denote the number of common points in this case by $n_{ij}$.
Moreover, we denote the local coordinates of the common points in $\P_i$ and $\P_j$ as
\begin{equation*}
\{\x^k_{ij} : 1 \leq k \leq n_{ij}\} \quad \text{and} \quad \{\x^k_{ji} : 1 \leq k \leq n_{ij}\}.
\end{equation*}
In the ideal case, the hypothesis is that these points are related to each other via a transform from $\mathbb{SE}(d)$. In particular, if we associate the transform $(\R_i,\t_i) \in \mathbb{SE}(d)$ with $\P_i$, then 
\begin{equation}
\label{rel}
\R_i \x^k_{ij} + \t_i  =  \R_j \x^k_{ji} + \t_j \qquad (i \sim j).
\end{equation}
Needless to say, the transforms are specified relative to some global reference frame.

\subsection{Least-Squares Formulation}

 In practice, the consistency relations \eqref{rel} only hold approximately due to various imperfections. The task of computing the transforms can be posed as an optimization problem in this case. In particular, we consider the least-squares formulation 
 \begin{equation}
\label{LSopt}
\min \ \sum_{i \sim j} \sum_{k=1}^{n_{ij}} \  \lVert \R_i \x^k_{ij} + \t_i -   \R_j \x^k_{ji} - \t_j \rVert_2^2.
\end{equation}
By introducing the matrix variables 
\begin{equation}
\label{defRT}
\R = \left[ \R_1 \cdots \R_m \right]  \quad \text{and} \quad  \T = \left[ \t_1 \cdots \t_m \right],
\end{equation}
we can express \eqref{LSopt} as
\begin{equation}
\label{obj}
\min \ \sum_{i \sim j} \sum_{k=1}^{n_{ij}} \  \lVert \R \del^k_{ij} + \T \e_{ij} \rVert_2^2,
\end{equation}
where $\e_{ij} = \e_i - \e_j$, $\e_i \in \mathbb{R}^m$ is the all-zeros vector with one at the $i$-th position, and 
\begin{equation*}
\del^k_{ij} = (\e_i  \otimes \mathbf{I})\  \x^k_{ij} - (\e_j  \otimes \mathbf{I})\  \x^k_{ji},
\end{equation*}
where  $\otimes$ is the Kronecker product. For fixed $\R$, the minimum of \eqref{obj} is attained when $\T = - \R \B \mathbf{L}^{\dagger}$, where 
\begin{equation*}
\mathbf{L} = \sum_{i \sim j} n_{ij} \e_{ij}\e_{ij}^\top \quad  \text{and} \quad \B = \sum_{i \sim j} \sum_{k=1}^{n_{ij}}  \del^k_{ij} \e_{ij}^\top,
\end{equation*}
and $\mathbf{L}^{\dagger}$ is the Moore-Penrose pseudo-inverse of $\mathbf{L}$. Substituting $\T = - \R \B \mathbf{L}^{\dagger}$, the objective in \eqref{obj} becomes
\begin{equation}
\label{exp}
\langle \C, \R^\top \R \rangle=  \sum_{ i,j=1}^m \text{trace}(\C_{ij} \R_i^\top \R_j),
\end{equation}
where 
\begin{equation*}
\D = \sum_{i \sim j}\sum_{k=1}^{n_{ij}}   \del_{ij}\del_{ij}^\top \quad \text{and}\quad  \C=\D-\B \mathbf{L}^{\dagger}\B^\top.
\end{equation*}
In \eqref{exp}, we use $\C_{ij} \in \mathbb{R}^{d \times d}$ to denote the $(i,j)$-th block (sub-matrix) of $\C$:
\begin{equation*}
\C_{ij} (p,q) = \C\left((i-1)d + p,(j-1)d + q\right),
\end{equation*}
where $ i,j \in [m]$ and $p,q \in [d]$. In summary, we have reduced \eqref{obj} to 
\begin{equation}
\label{manopt}
\underset{\R}{\text{min}} \ \ \langle \C, \R^\top \R \rangle \quad  \text{s.t.} \quad  \R_1,\ldots,\R_m \in \mathbb{SO}(d). 
\end{equation}
The reduction is in the following sense: if the minimizer of \eqref{manopt} is $\R^*$, then the minimizer of \eqref{obj} is $(\R^*,\T^*), \ \T^* = - \R^* \B \mathbf{L}^{\dagger}$. Note that the domain of \eqref{manopt} is $\mathbb{SO}(d) \times \cdots \times \mathbb{SO}(d)$, which is  not convex. However, $\mathbb{SO}(d)$ has the structure of a smooth Riemannian manifold (\cite{Absil2009}), and this is be used to design efficient numerical solvers in \cite{Krishnan2005}.

\subsection{Nonconvex SDP}

By a change-of-variables, we can express \eqref{manopt} as an SDP, where the objective is linear but the constraints are nonconvex. In particular, this is done using the Gram matrix $\G \in \mathbb{R}^{dm \times dm}$  of the variables $\R_1,\ldots,\R_m \in \mathbb{SO}(d)$, given by  
\begin{equation}
\label{defG}
\G_{ij}=\R_i^\top\! \R_j \qquad (i,j \in [m]).
\end{equation}
We reiterate that $\G_{ij}$ denotes the $(i,j)$-th block of $\G$, i.e.,
\begin{equation*}
\G = \begin{bmatrix} \G_{11} & \G_{12} & \dots & \G_{1m}\\ \G_{21} & \G_{22} & \dots & \G_{2m}\\ \vdots & \vdots & \ddots & \vdots \\ \G_{m1} & \G_{m2} & \dots & \G_{mm} \end{bmatrix}.
\end{equation*}
\begin{proposition}
The Gram matrix has the following properties:
\begin{description}
\item $(\mathrm{P1})$ $\G \in \mathbb{S}^{dm}_+$,
\item $(\mathrm{P2)}$ $\mathrm{rank}(\G) \leq d$,
\item $(\mathrm{P3})$ $\G_{ii}=\mathbf{I}, \ i \in [m]$, and
\item $(\mathrm{P4})$ $\G_{i,i+1} \in  \mathbb{SO}(d), \ i \in [m-1]$.
\end{description}
\end{proposition}
\begin{proof}
Since we can write $\G=\R^\top\!\R$, $(\mathrm{P1})$ is clear. Moreover, since $\R$ has full rank, and $\mathrm{rank}(\G)= \mathrm{rank}(\R)$, we in fact have equality in $(\mathrm{P2})$. Finally, note that $\G_{i,i+1} = \R_i^\top\R_{i+1}$, i.e., $\G_{i,i+1}$ is the product of rotations. In particular, $\G_{i,i} = \R_i^\top\R_{i}= \mathbf{I}$. This establishes $(\mathrm{P3})$ and $(\mathrm{P4})$.
\end{proof}

Apart from ($\mathrm{P1}$)-($\mathrm{P4}$), we can of course list other properties of $\G$. But the key observation is that ($\mathrm{P1}$)-($\mathrm{P4}$) are its essential properties, namely, they are sufficient to characterize $\G$ as a Gram matrix of rotations (see Appendix for the proof).
\begin{theorem}
\label{thm1}
If $\G$ satisfies $\mathrm{(P1)}\mbox{-}\mathrm{(P4)}$, then $\G$ is given by \eqref{defG} for some $\R_1,\ldots,\R_m \in \mathbb{SO}(d)$.
\end{theorem}
At this point, we note that Theorem \ref{thm1} remains valid if we use $\mathrm{rank}(\G) = d$ in ($\mathrm{P2}$). The reason why we use the present formulation is that the set $\{\G  \in \mathbb{S}^{dm}: \mathrm{rank}(\G) \leq d\}$ is closed in $\mathbb{S}^{dm}$, i.e., it contains all its limit points. In contrast, the set $\{\G \in \mathbb{S}^{dm}: \mathrm{rank}(\G) = d\}$ is not closed. For example, the sequence of matrices  $\mathbf{I},(1/2)\mathbf{I}, (1/3)\mathbf{I}, \ldots$ are of rank $d$, but their limit is the zero matrix which has zero rank. The algorithmic implication of this technical point will be evident in the next section, where we will be required to compute the projection on a set. If the set is not closed, then the projection might not be defined.

Based on Theorem \ref{thm1}, we substitute $\G=\R^\top\!\R$ in \eqref{manopt} and consider the following problem:
\begin{equation}
\label{ncvxSDP}
\begin{aligned}
& \underset{\G}{\text{min}} & & \langle \C, \G \rangle  \\
& \text{s.t.} & & \G \in \mathbb{S}^{dm}_+, \ \mathrm{rank}(\G) \leq d, \\
&&& \G_{ii}=\mathbf{I}, \ i \in [m], \\
&&& \G_{i,i+1} \in \mathbb{SO}(d), \ i \in [m-1].
\end{aligned}
\end{equation}
Note that the objectives in \eqref{manopt} and \eqref{ncvxSDP} are identical. Moreover, following Theorem \ref{thm1}, there exists a one-to-one correspondence between the domains of \eqref{manopt} and \eqref{ncvxSDP}. Therefore, we have the following important observation.
\begin{theorem} 
\label{thm2} 
Problems \eqref{manopt} and \eqref{ncvxSDP} are equivalent, that is, $\G^{\star}$ is a minimizer of \eqref{manopt}  if and only if $\G^{\star}$ a minimizer of \eqref{ncvxSDP}.
\end{theorem}
The optimization in \eqref{ncvxSDP} is a nonconvex SDP (\cite{Diamond2017}), where the variable is positive semidefinite and the objective is linear. The constraint $\G_{ii}=\mathbf{I}$ is affine. However, the constraints
\begin{equation}
\label{ncvx}
\mathrm{rank}(\G) \leq d \quad \text{and} \quad \G_{i,i+1} \in \mathbb{SO}(d)
\end{equation}
are nonconvex. We develop an ADMM-based solver for \eqref{ncvxSDP}. This is based on two crucial observations---the objective in  \eqref{ncvxSDP} is linear, and the projections onto the feasible sets in \eqref{ncvx} can be computed in closed-form.

\subsection{Variable Splitting and ADMM}

Following the success of ADMM for convex programming (\cite{Boyd2011}), the ADMM framework has been extended to several nonconvex problems (\cite{Chartrand2013,Miksik2014,Diamond2017}). Preliminary theoretical results concerning the validity of such formal extensions have also been reported (\cite{Wang2015,Hong2016}). We propose an ADMM solver for  \eqref{ncvxSDP} using variable splitting. In particular, we define the following subsets of $\mathbb{S}^{dm}$:
\begin{equation*}
 \Omega = \Big\{\mathbf{X} \in \mathbb{S}^{dm}_+ ; \ \mathrm{rank}(\mathbf{X}) \leq d \Big\},
\end{equation*}
and
\begin{equation*}
\Theta = \Big\{ \mathbf{X}_{ii}=\mathbf{I}, i \in [m]; \ \mathbf{X}_{j,j+1} \in \mathbb{SO}(d), j \in [m-1] \Big\}.
\end{equation*}
Simply stated, $\Theta$ consists of symmetric matrices whose diagonal blocks are identity and super-diagonal blocks are rotations. Note that we can equivalently write \eqref{ncvxSDP} as
\begin{equation}
\label{split}
\begin{aligned}
& \underset{\G, \H}{\text{min}} & & \langle \C, \G \rangle  \\
& \text{s.t.} & & \G \in \Omega, \\
&&& \H  \in \Theta, \\
&&& \G - \H = 0.
\end{aligned}
\end{equation}
The purpose of this splitting is to group the constraints into two distinct classes, albeit at the expense of a linear constraint. For fixed $\rho > 0$, the augmented Lagrangian for \eqref{split} is 
\begin{equation*}
\label{AL}
\L_{\rho}(\G,\H,\Lam)= \langle\C, \G \rangle + \langle \Lam , \G-\H \rangle + \frac{\rho}{2}\Vert \G-\H\Vert_{\mathrm{F}}^{2},
\end{equation*}
where $\Lam \in \mathbb{S}^{dm}$ is the dual variable for the constraint $ \G-\H=\mathbf{0}$ (\cite{Boyd2011}). Starting with initializations $\H^{0}$ and $\Lam^{0}$, the ADMM solver uses the following sequence of updates for $k \geq 0$:
\begin{equation}
\label{Gupdate} 
\G^{k+1} =\underset{\G \in \Omega}{\text{argmin}}\,\,\,\L_{\rho}(\G,\H^{k},\Lam^{k}),
\end{equation}
\begin{equation}
\label{Hupdate} 
 \H^{k+1} =\underset{\H \in \Theta}{\text{argmin}}\,\,\,\L_{\rho}(\G^{k+1},\H,\Lam^{k}),
 \end{equation}
\begin{equation}
\label{updateLam}
\Lam^{k+1}=\Lam^{k}+\rho(\G^{k+1}-\H^{k+1}).
\end{equation}
Note that we can write the objective in \eqref{Gupdate} as 
\begin{equation*}
 \frac{\rho}{2}\Vert \G- (\H^k - \rho^{-1} (\C+\Lam^k ) ) \Vert_{\mathrm{F}}^{2} \ + \ \text{ constant },
\end{equation*}
where the constant does not depend on $\G$. Therefore,
\begin{equation}
\label{projG}
\G^{k+1} = \Pi_{\Omega} \big(\H^k - \rho^{-1} (\C+\Lam^k ) \big),
\end{equation}
where $\Pi_{\Omega}(\A)$ is the projection of $\A \in \mathbb{S}^{dm}$ onto $\Omega$: 
\begin{equation}
\label{proj}
\Pi_{\Omega}(\A) = \underset{\X \in \Omega}{\text{argmin}} \ \lVert \X - \A \rVert_{\mathrm{F}}^2.
\end{equation}
Similarly, it can be verified that
\begin{equation}
\label{projH}
\H^{k+1} = \Pi_{\Theta} (\G^{k+1} + \rho^{-1} \Lam^k ).
\end{equation}
In summary, we were able to express the primal subproblems as orthogonal projections. This is where the linearity of the objective in \eqref{split} plays an important role. At this point, note that both $\Omega$ and $\Theta$ are closed sets. Hence, $\Pi_{\Omega}$ and $\Pi_{\Theta}$ are well-defined. However, if we had defined $ \Omega$ to be the set of rank-$d$ matrices, then $\Pi_{\Omega}(\A)$ would not be defined if the rank of $\A$ were strictly less than $d$.

The complete algorithm is summarized in Algorithm \ref{algo}. We now explain how steps \ref{updateG} and \ref{updateH} can be computed in closed-form.
\begin{algorithm}
\label{algo}
\caption{ADMM Solver.}
\textbf{Input}: $\C$ and $\rho>0$.\\
Initialize $\H$ and $\Lam$.\\
\While{\textit{the stopping criterion is not met}}{
$\G \leftarrow \Pi_{\Omega} (\H - \rho^{-1} (\C+\Lam ) )$.\\ \label{updateG}
$\H \leftarrow \Pi_{\Theta} (\G + \rho^{-1} \Lam )$.\\ \label{updateH}
$\Lam \leftarrow \Lam+\rho(\G-\H)$.
}
\end{algorithm}

\subsection{Matrix Projections}

Recall that $\Omega$ and $\Theta$ are closed in $\mathbb{S}^{dm}$. As a result, projections \eqref{projG} and \eqref{projH} are well-defined. In particular, by adapting the Eckart-Young theorem \cite{Eckart1936} for positive semidefinite matrices, we can deduce the following.
 \begin{theorem}
\label{thm3}
Let the eigendecomposition of $\A \in \mathbb{S}^{dm}$ be
\begin{equation*}
\A = \sum_{i=1}^{dm} \lambda_i \boldsymbol{u}_i \boldsymbol{u}_i ^{\top},
\end{equation*}
where $\lambda_1 \geq  \cdots \geq  \lambda_{dm}$ are its eigenvalues, and $\boldsymbol{u}_1,\ldots, \boldsymbol{u}_{dm}$ the corresponding eigenvectors. Then
\begin{equation*}
\Pi_{\Omega}(\A) = \sum_{i=1}^d \max(\lambda_i,0) \boldsymbol{u}_i \boldsymbol{u}_i ^{\top}.
\end{equation*}
\end{theorem} 
For completeness, we sketch the proof of Theorem \ref{thm3} in the Appendix.  In practice, we can efficiently compute $\Pi_{\Omega}(\A)$ using the power method (\cite{golub2012matrix}), since we only require the top-$d$ eigenvalues/eigenvectors of $\A$. 
The relevance of property ($\mathrm{P2}$) is now evident. Namely,  $\Pi_{\Omega}(\A)$ might not be defined if instead of requiring the matrices in $\Omega$ to have rank at most $d$, we insist that the rank be exactly $d$. In particular, if $\A$ has less than $d$ positive eigenvalues, then $\Pi_{\Omega}(\A)$ does not exist as the minimum in \eqref{proj} is not attained in this case.

To compute $\Pi_{\Theta}(\A)$, note that if $\X \in \Theta$, then we can write 
\begin{align*}
 \lVert \bX - \bA \rVert_{\mathrm{F}}^2 &=  \sum_{i=j} \lVert \mathbf{I} - \A_{i,j} \rVert_{\mathrm{F}}^2 + \sum_{|i-j| =1} \lVert \X_{i,i+1} -\A_{i,i+1} \rVert_{\mathrm{F}}^2\\
 & +  \sum_{|i-j| \geq 2} \lVert \X_{i,j} -\A_{i,j} \rVert_{\mathrm{F}}^2.
 \end{align*}
Following the definition of $\Theta$, it is then evident that
\begin{equation*}
\Pi_{\Theta}(\bA)_{ij} 
=\begin{cases} 
\mathbf{I} & i=j,\\
\Pi_{\mathbb{SO}(d)} (\A_{ij}) & |i-j|=1,\\
\A_{ij} & |i - j| \geq 2.
\end{cases}
\end{equation*}
Simply stated, we can determine $\Pi_{\Theta}(\A)$ by setting the diagonal blocks of $\A$ to $\mathbf{I}$, projecting the super- and sub-diagonal blocks of $\A$ onto $\mathbb{SO}(d)$, and keeping the remaining blocks of $\A$ unchanged. Of course, since the resulting matrix is required to be symmetric, we simply project  the $m-1$ super-diagonal blocks of $\A$ and set the  sub-diagonal blocks using symmetry.
In this regard, we record the following result.
\begin{theorem}
\label{thm4}
Let the SVD of $\H \in \mathbb{R}^{d \times d}$ be
 \begin{equation*}
 \H=\bU \mathrm{diag}(\sigma_1,\ldots,\sigma_{d-1},\sigma_d ) \bV^\top,
 \end{equation*}
 where $\sigma_1 \geq  \cdots \geq  \sigma_d \geq 0$ and $\bU,\bV \in \mathbb{O}(d)$. Then
  \begin{equation*}
 \Pi_{\mathbb{SO}(d)}(\H)=\bU \mathrm{diag}\left(1,\ldots,1, \mathrm{det}(\bU \bV^\top) \right)\bV^\top.
 \end{equation*}
\end{theorem}
A proof of this result using the theory of Lagrange multipliers can be found in \cite{Umeyama1991,kabsch1976solution}. We provide a somewhat different proof in the Appendix. 

We note that the optimum $\H^\ast$ obtained using Algorithm \ref{algo} has rank $d$ or more (this follows from the definition of $\Theta$). If the rank of $\H^\ast$ is exactly $d$, then it follows from Theorem \ref{thm1} that we can factor it as $\H^\ast = \R^\top\! \R$, where $\R \in \mathbb{R}^{d \times dm}$. This can be done using the eigendecomposition of $\H^\ast$; see \cite{Chaudhury2015}. In particular, if we partition $\R$ into $m$ blocks of size $d \times d$, then each block would be a rotation matrix. However, if the rank of $\H^\ast$ is greater than $d$, we have to use some form of ``rounding'' to extract the rotations (which are no longer uniquely defined). In the present case, we have use spectral rounding (see Section $2.3$ in \cite{Chaudhury2015}).

\subsection{Complexity Analysis}

The breakup of the computation complexity for the proposed method is given in Table \ref{complexitytable}. The one-time cost of building the matrix $C$ is $O(m^2(m+d+d^2)) = O(m^3)$, since $m \gg d$ in practice. The cost per iteration is dominated by the projections steps $4$ and $5$ in Table \ref{complexitytable}. For step $5$, we need to find the top $d$ eigenvectors of an $md \times md$ matrix. This can be done efficiently using Arnoldi iterations (\cite{GoluVanl96}), where the per-iteration cost is $O(m^2d^2)$. On the other hand, we need to compute the SVD of $m-1$ matrices in step $5$, each of size $d \times d$. This can be done efficiently using QR-type algorithms (\cite{GoluVanl96}). Finally, a single SVD is required in step $6$ (one time cost). In summary, the per-iteration cost of Algorithm \ref{algo} is essentially the partial eigendecomposition of a $md \times md$ matrix, and the SVD of $m-1$ matrices of size $d \times d$. As a result, we can scale up our algorithm to problems involving large number of point sets. 

\begin{table}[!ht]
\centering
\footnotesize
\begin{tabular}{|l|l|l}
\cline{1-2}
\multicolumn{1}{|c|}{\textbf{Operation}}                     & \multicolumn{1}{c|}{\textbf{Complexity}} &                                    \\ \cline{1-2}
$1. L^{\dagger}$                                             & $O(m^3)$                                 &                                    \\ \cline{1-2}
$2. B L^{\dagger} B^\top$                                    & $O(m^2d)$                                &                                    \\ \cline{1-2}
$3. C=D-B L^{\dagger} B^\top$                                & $ O(m^2d^2)$                             &                                    \\ \cline{1-2}
$4. G \leftarrow \Pi_{\Omega} (H - \rho^{-1} (C+\Lambda ) )$ & $ O(km^2d^2)$                            & \rdelim\}{2}{0.1mm}[\small{recurring cost}]  \\ \cline{1-2}
$5. H \leftarrow \Pi_{\Theta} (G + \rho^{-1} \Lambda )$      & $(m-1)O(d^3)$                            &                                    \\ \cline{1-2}
$6. $final rounding                                           & $O(d^3)$                                 &                                    \\ \cline{1-2}
\end{tabular}
\caption{Breakup of the complexity of Algorithm \ref{algo} ($k$ in step $4$ is the number of Arnoldi iterations). }
\label{complexitytable}
\end{table}

\section{Numerical Results}
\label{sec:Numerics}

In this section, we present numerical results using simulated point sets for which the point-to-point correspondences are known. This allows us to factor out the task of finding the correspondences, which will be  addressed in Section \ref{sec:MR}. Moreover, this allows us to manipulate the true correspondences and access their impact on the registration. In particular, we empirically investigate the following questions:

\begin{itemize}
\item Are the iterates in Algorithm \ref{algo} convergent? 
\item If so, do they converge to the global minimum of \eqref{ncvxSDP}, or simply to some saddle point?
\item How does the initialization and $\rho$ affect convergence?
\item Is the algorithm stable to perturbations of coordinates and correspondences? 
\end{itemize}
The first three questions can be theoretically resolved for convex problems (\cite{Boyd2011}). However, the situation is more complicated with nonconvex problems, and only some preliminary results have been established (\cite{Wang2015,Hong2016}). Similar to (\cite{Chartrand2013,Miksik2014,Wang2015,Diamond2017}), we will demonstrate the practical utility of the ADMM solver. In fact, as with the above papers, we will show that the solver works well in practice.

\begin{figure}
\center
\subfloat[Point set $1$.]{\includegraphics[width=0.2\linewidth]{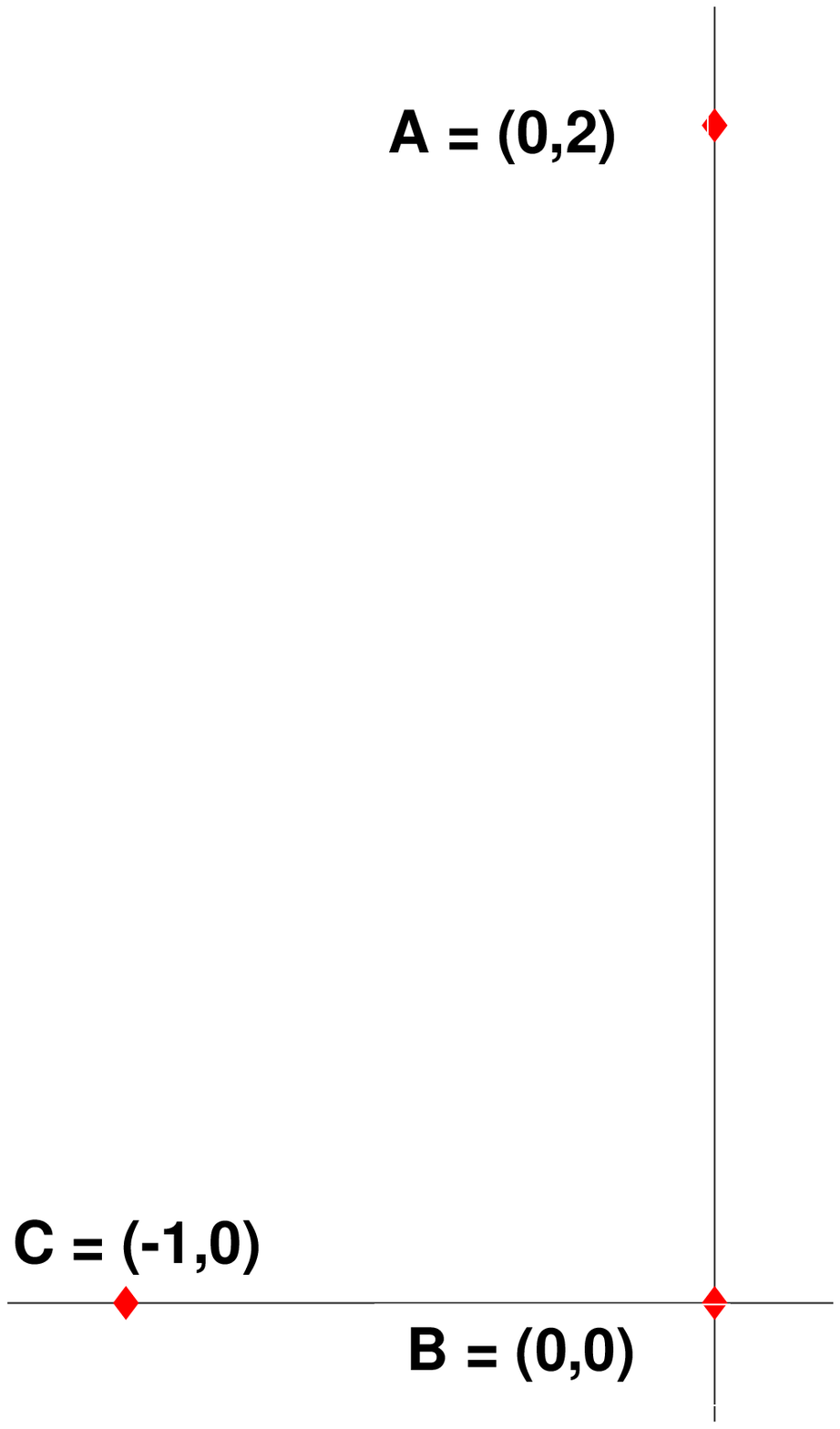}} \hspace{4mm}
\subfloat[Point set $2$.]{\includegraphics[width=0.2\linewidth]{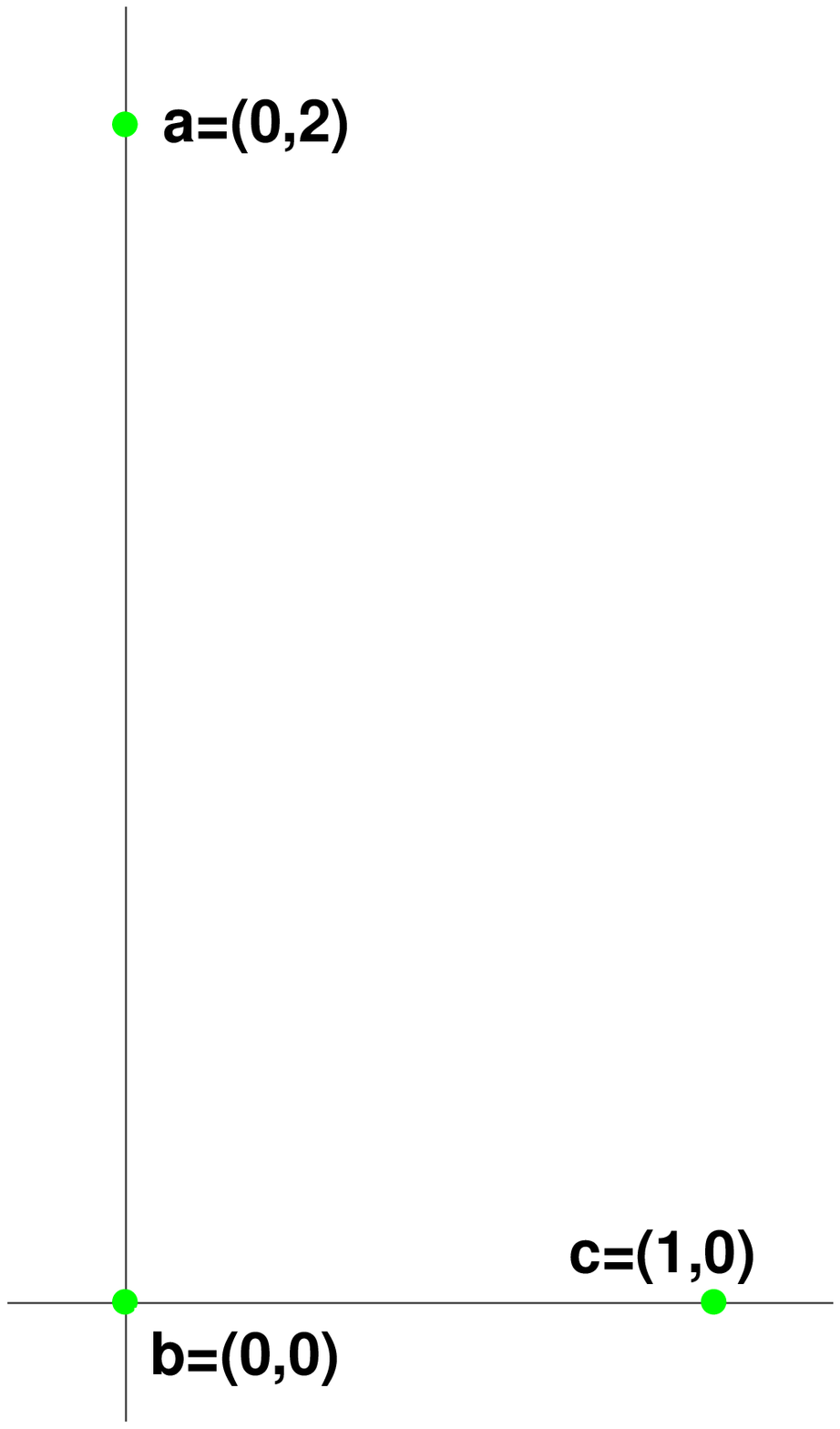}} \hspace{4mm}
\subfloat[After registration.]{\includegraphics[height=0.35\linewidth,width=0.3\linewidth]{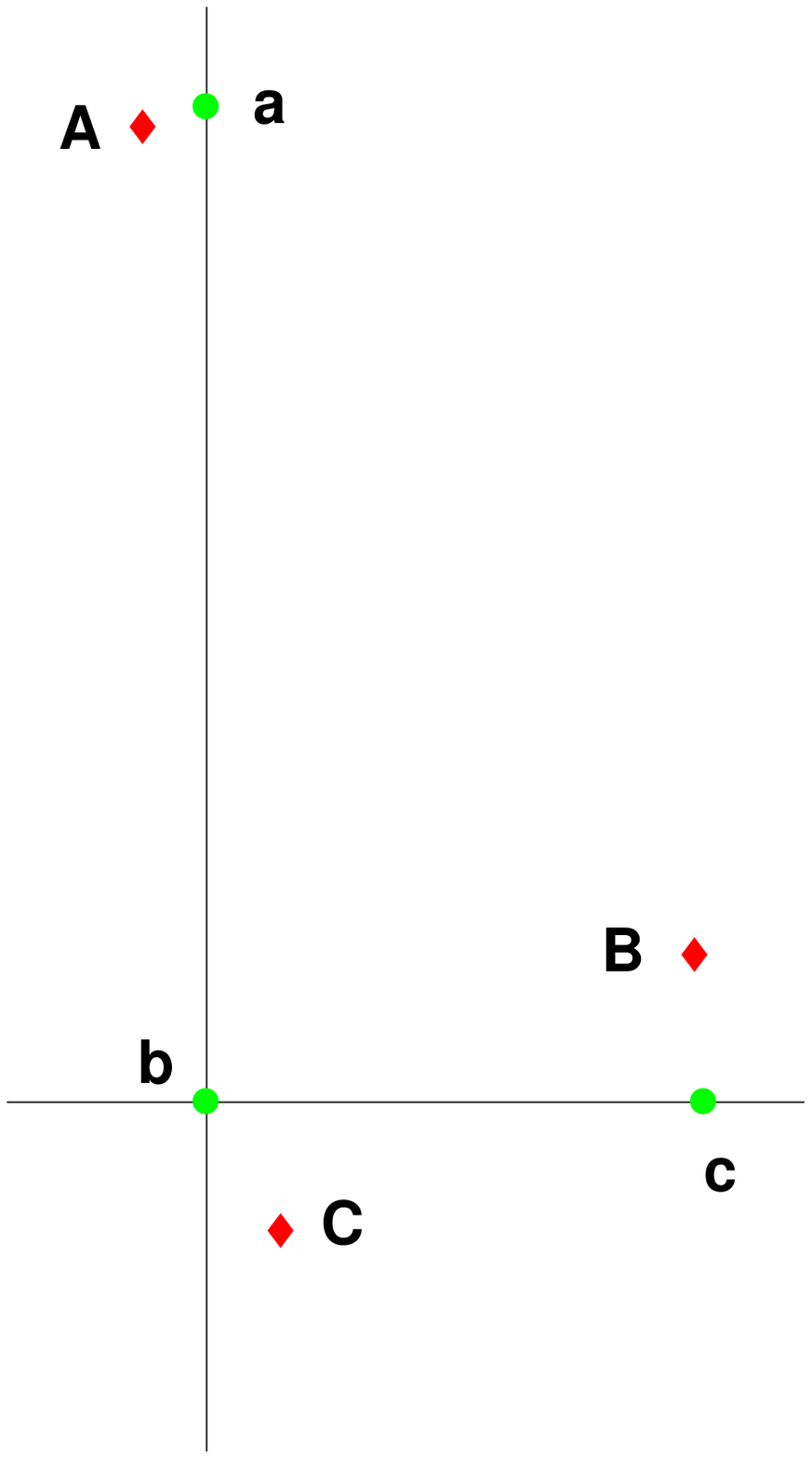}} 
\caption{Registration result for Umeyama's problem using Algorithm \ref{algo}.} 
\label{fig1}
\end{figure}

 \begin{figure}
\center
\includegraphics[width=0.41\linewidth]{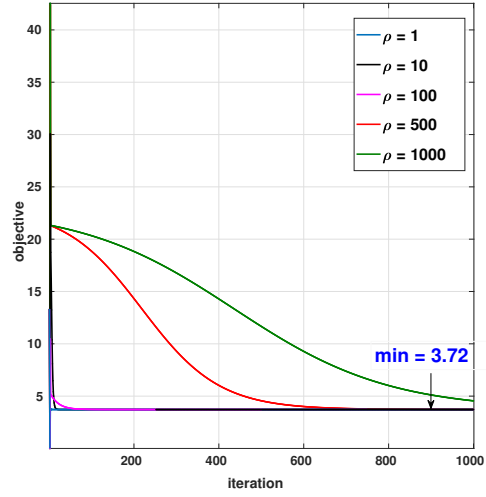}
\caption{Convergence for different $\rho$ for the problem in Figure \ref{fig1}.} 
 \label{fig2}
\end{figure}

To initialize Algorithm \ref{algo}, we set $\Lam$ as the zero matrix. The solution of the spectral relaxation of \eqref{ncvxSDP} is used for $\H$ (\cite{Chaudhury2015}). The latter requires us to compute the smallest $d$ eigenvectors of $\mathbf{C}$. 
We will sometimes initialize $\H$ with the all-identity matrix (a matrix of size $md \times md$ whose $m^2$ blocks are the $d \times d$ identity matrix)  for some experiments, which is weaker than the spectral initialization. 

To address the second question, we somehow need to compute the global minimum of \eqref{ncvxSDP}. 
An useful observation is that this can be done for two point sets. In this case, the minimizer is simply given by the Gram matrix of the optimal rotations obtained using Umeyama's algorithm \cite{Umeyama1991}. This will be particularly useful for analyzing the stability of our algorithm, since it is non-trivial to determine the minimum of \eqref{ncvxSDP} in the presence of noise.
On the other hand, note that the minimum of \eqref{ncvxSDP} is simply zero with simulated data (without noise). In this case, the minimizer is the Gram matrix of the ground-truth  rotations. 

Since Umeyama's original formulation (\cite{Umeyama1991}) includes scaling (along with rotation and translation), we describe the solution over $\mathbb{SE}(d)$ for completeness. 
For $m=2$, we can write \eqref{LSopt} as
 \begin{equation}
\label{opt2}
\min \  \sum_{k=1}^{n} \  \lVert \R_1 \x^k + \t_1 -   \R_2 \boldsymbol{y}^k - \t_2 \rVert_2^2,
\end{equation}
where $n$ is the number of common points, and $(\x^k)$ and $(\boldsymbol{y}^k)$ are the respective local coordinates.

 The problem can be simplified by fixing one point set and computing the relative rotation and translation of the other. In particular, we can take $\R_1=\mathbf{I}$ , $\t_1=\boldsymbol{0}$ and $\R_2=\R$ , $\t_2=t$. That is, we consider the following simplification of \eqref{opt2}:
 \begin{equation}
\label{opt2s}
\min_{\R,\t} \ \sum_{k=1}^{n} \  \lVert  \x^k  -   \R \boldsymbol{y}^k- \t \rVert_2^2.
\end{equation}
As shown in \cite{Umeyama1991}, the minimizers of \eqref{opt2s} are
 \begin{equation}
 \label{minRt}
\R^\star = \Pi_{\mathbb{SO}(d)} (\mathbf{H}) \quad \text{and} \quad \t^\star =\bar{\x}-\R^\star \bar{\boldsymbol{y}},
\end{equation}
where
\begin{equation*}
\bar{\x} = \frac{1}{n}  \sum_{k=1}^n  \x^k, \qquad \bar{\boldsymbol{y}} = \frac{1}{n} \sum_{k=1}^n  \boldsymbol{y}^k, 
\end{equation*}
and
\begin{equation*}
\H =  \sum_{k=1}^n (\x^k - \bar{\x} )( \boldsymbol{y}^k - \bar{\boldsymbol{y}})^\top.
\end{equation*}

For our first experiment, we consider the example in \cite{Umeyama1991} involving two point sets. Each set has three points whose coordinates are $\{(0,0),(1,0),(0,2)\}$ and $\{(0,0),(-1,0),(0,2)\}$; one point set is simply a reflection of the other (see Figure \ref{fig1}). Clearly, the minimum of \eqref{opt2s} cannot be zero in this case, since the points cannot be perfectly aligned using just translations and rotations. In fact, the optimal value corresponding to \eqref{minRt} is $3.7185$ (up to four decimal places). 

We next solve the above problem using Algorithm \ref{algo}, where we work with formulation \eqref{opt2}. Note that the algorithm does not make use of Umeyama's solution; it iteratively computes the solution starting from some initialization. Therefore, this simple problem is nonetheless non-trivial for Algorithm \ref{algo}. For different values of $\rho$, the  objective at each iteration is shown in Figure \ref{fig2}, where we initialize using $\H=\mathbf{I}$.
Note that, while the convergence speed changes with $\rho$ (as in convex ADMM), the objective asymptotically converges to the minimum (indicted in the figure) obtained using Umeyama's formula. The solution obtained using  our algorithm  is depicted in Figure \ref{fig1}, which agrees perfectly with the result in \cite{Umeyama1991}.


\begin{table}
\centering
\small
\caption{Comparison of rotational error of proposed algorithms with that in \cite{Ahmed2017} for $10$ point sets. Rows I, II and III correspond to the ground truth, results from \cite{Ahmed2017}, and results from our method (see text for details).}
\setlength\tabcolsep{2.0pt}
\begin{tabular}{|c|c|c|c|c|c|c|c|c|c|c|c|}
\hline
 \multirow{1}{*}{} & 
\multicolumn{10}{c|}{  Determinant of optimal transforms} &
  {Error} \\
\hline
I &+1 &+1 &+1 &+1 &+1 &+1 &+1 &+1 &+1 &+1 &$\ang{0}$ \\
\hline
II &+1 &+1 &+1 &\circled{-1} &\circled{-1} &+1 &+1 &+1 &+1 &+1 &$\ang{59.4}$\\
\hline
III &+1 &+1 &+1 &+1 &+1 &+1 &+1 &+1 &+1 &+1 &$\ang{5.23}$\\
\hline
\end{tabular}
\label{DC}
\end{table}

%

We next experiment with three-dimensional models from the Stanford  repository \cite{3Dscanrep}. To extract overlapping point sets from a given model (ground-truth), we follow the process in \cite{Evangelidis2014}. We first center the model by subtracting its centroid, and then rotate it about the $x$-axis by $m$ different angles. For a fixed rotation, the points above the $x\mbox{-}y$ plane are formed into a point set. After creating $m$ such point sets, we  randomly rotate and translate them. Obviously, we know the exact correspondences in this case.

We test the robustness of our algorithm by (i) corrupting the coordinates with additive Gaussian noise, and (ii) introducing false correspondences (outliers). The goal is to mimic real world scenarios in which the scanned measurements are invariably noisy and some of the correspondences are wrongly estimated. 
To objectively assess the reconstruction quality, we use the rotation error from \cite{Arrigoni2016}. If $(\R_i)$ are the true rotations and $(\hat{\R}_i)$ are the rotations obtained using our algorithm, we first remove the global rotation by fixing one of the rotations (say, the first one) and performing the corrections
$\R_i' =  \R_1^\top\R_i$ and  $\hat{\R}_i' =  \hat{\R}_1^\top\hat{\R}_i$. The rotation error is defined as 
\begin{equation}
\label{reconErr}
\frac{1}{m} \sum_{i=1}^m d(\R'_i,\hat{\R}_i'),
\end{equation}
where 
\begin{equation*}
d(\R_1,\R_2) = \cos^{ - 1} \left(\frac{1}{2}(\langle \R_1, \R_2 \rangle-1)\right)
\end{equation*}
is the geodesic distance on $\mathbb{SO}(3)$ (\cite{Arrigoni2016}).

\begin{figure}
\center
\subfloat[\textit{Bunny}.]{\includegraphics[height=0.3\linewidth,width=0.3\linewidth]{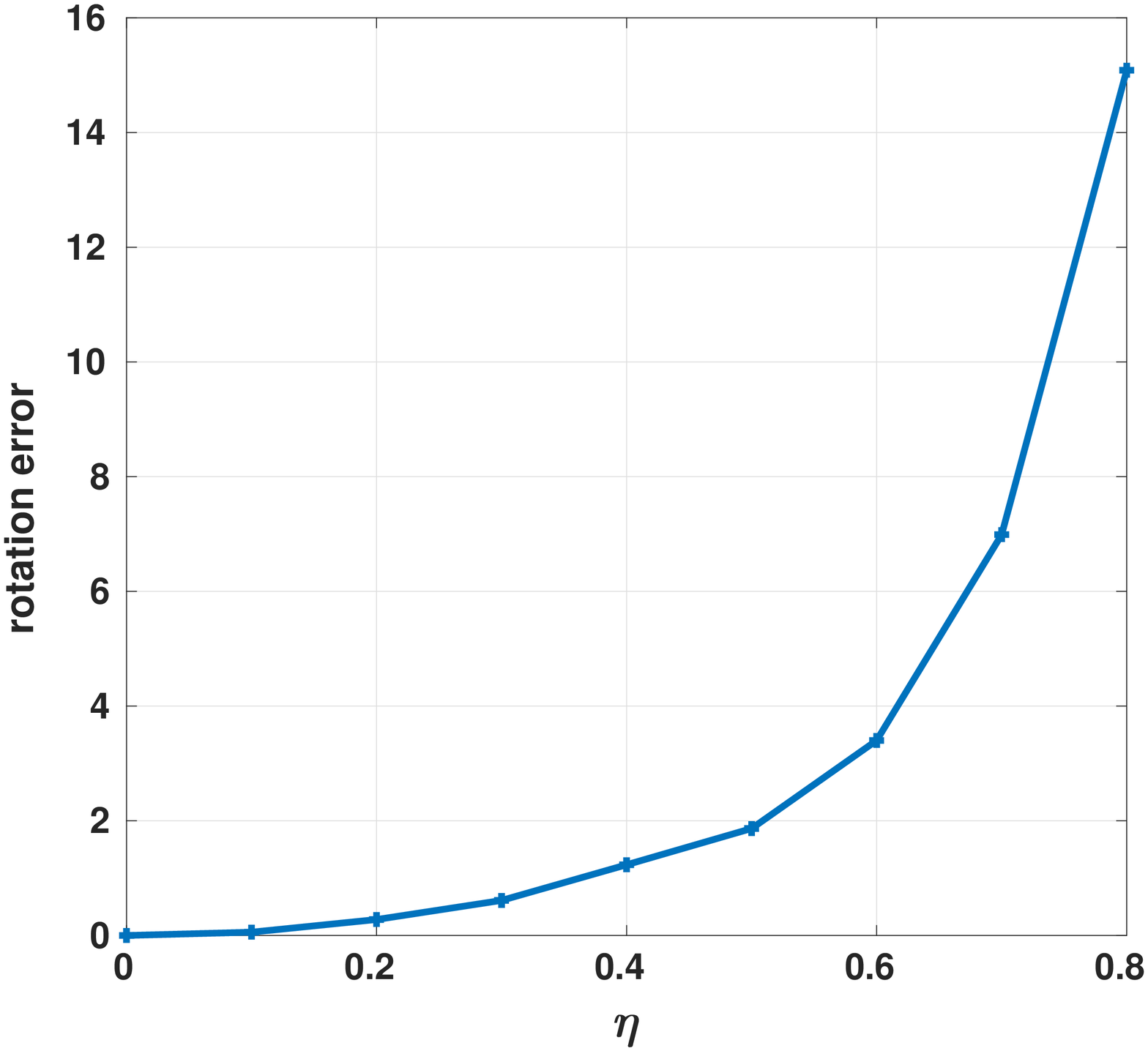}}
\subfloat[\textit{Buddha}.]{\includegraphics[height=0.3\linewidth,width=0.3\linewidth]{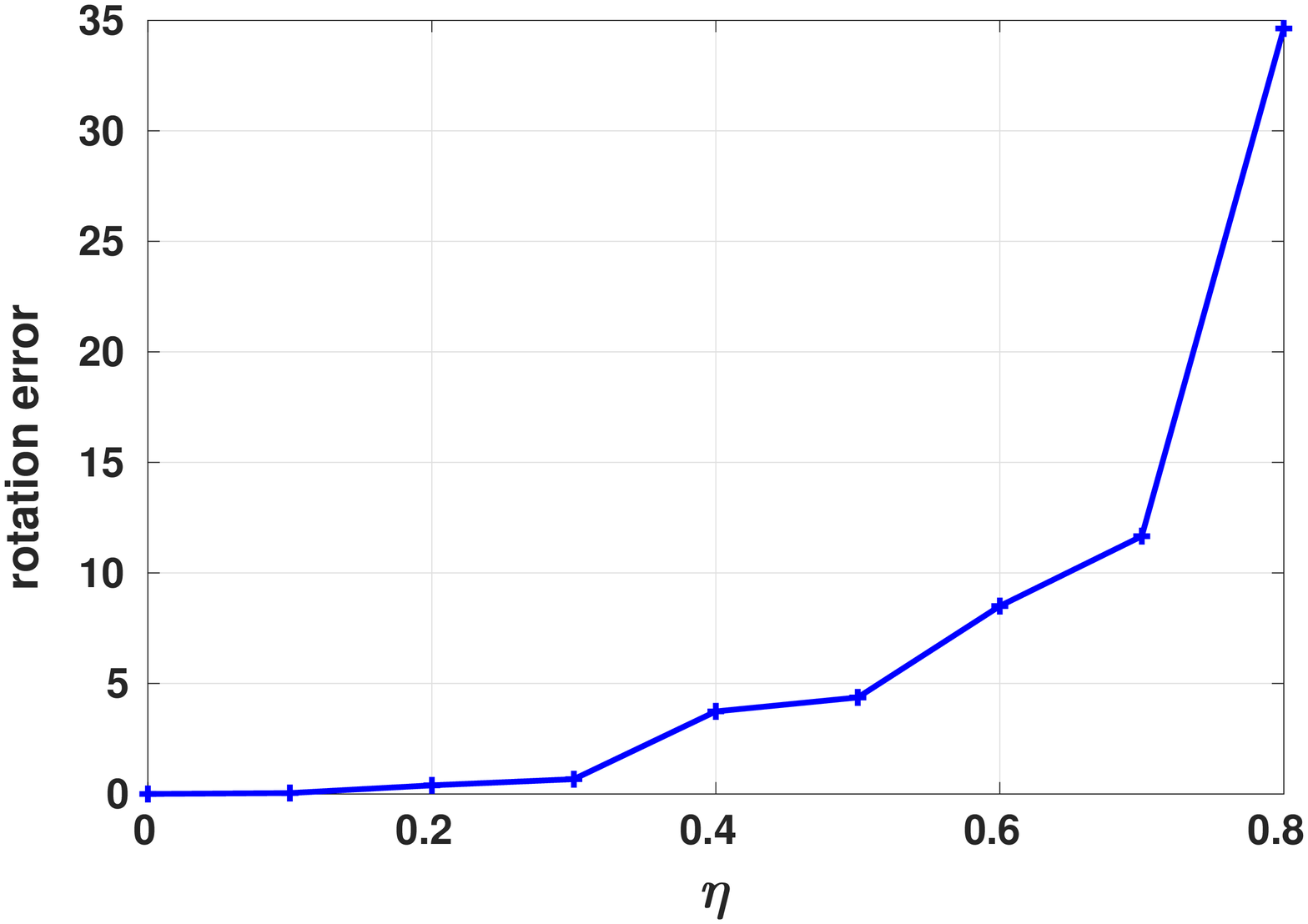}} 
\caption{Rotation error on account of wrong correspondences.} 
\label{fig3}
\end{figure}

\begin{figure}
\center
\subfloat[\textit{Bunny}.]{\includegraphics[height=0.3\linewidth,width=0.3\linewidth]{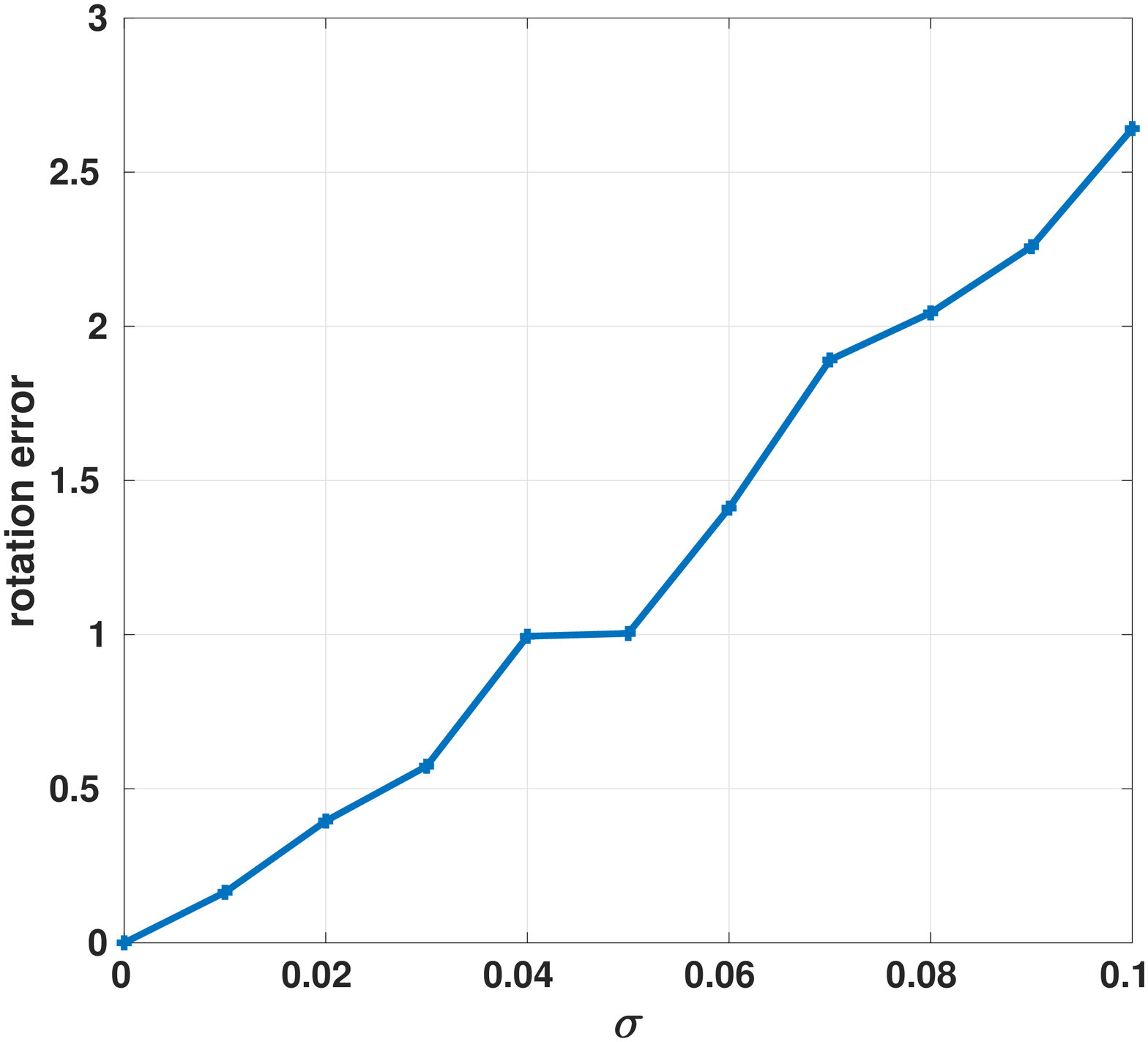}} 
\subfloat[\textit{Buddha}.]{\includegraphics[height=0.3\linewidth,width=0.3\linewidth]{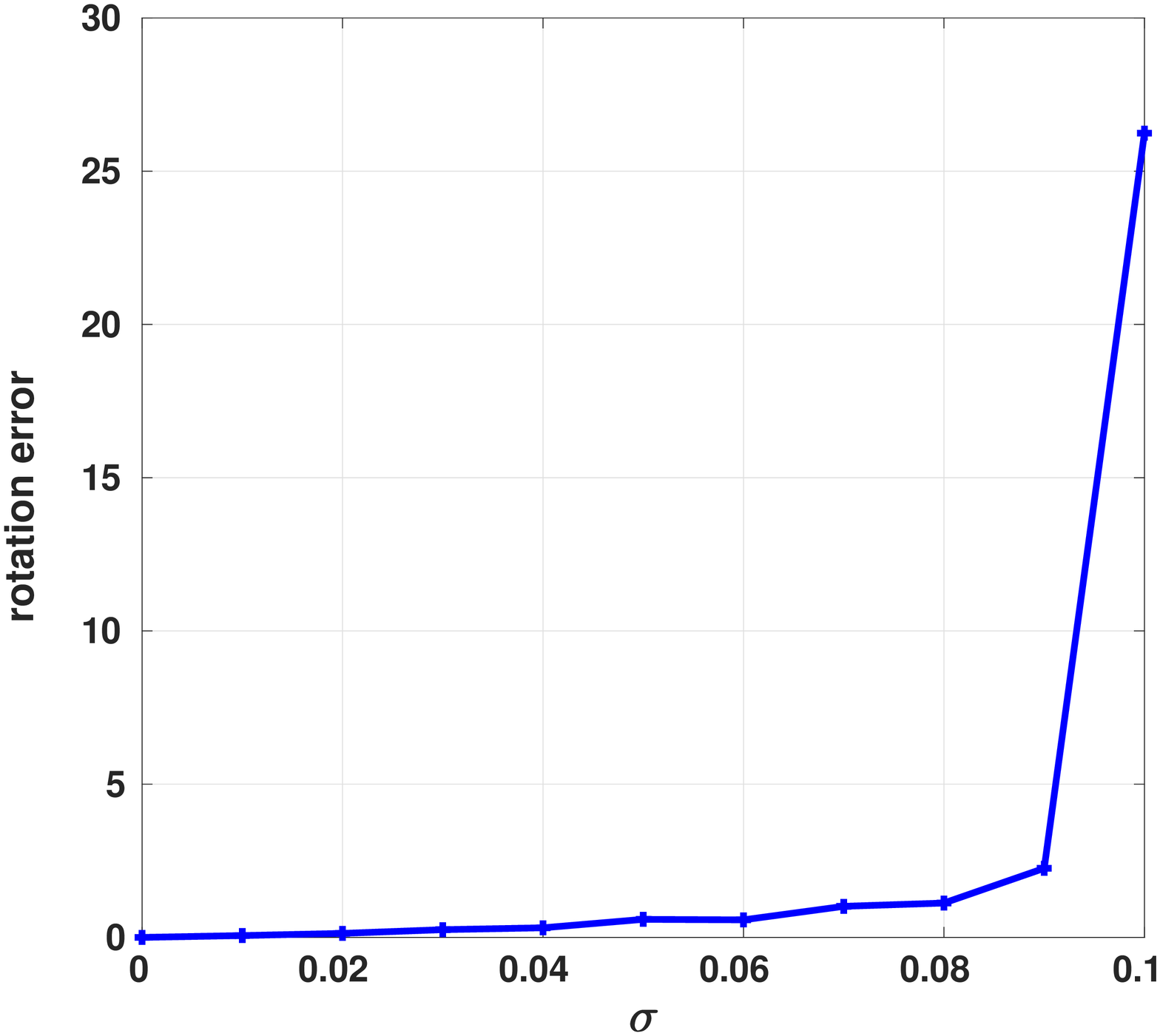}} 
 \caption{Rotation error on account of additive Gaussian noise.} 
\label{fig4}
\end{figure}

\begin{figure*}
\center
\subfloat[$\eta = 0.1$.]{\includegraphics[width=0.23\linewidth]{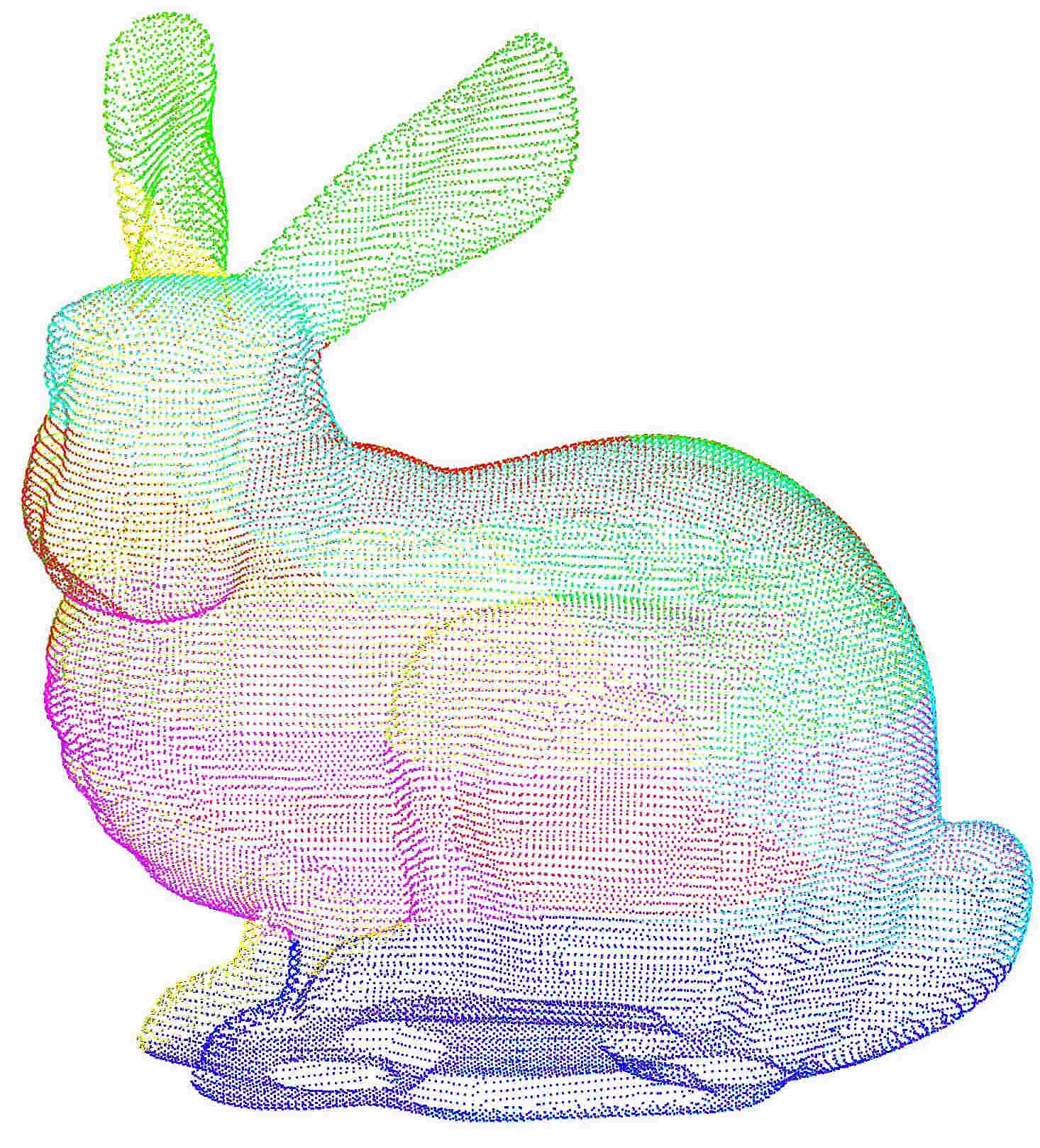}}
\subfloat[$\eta = 0.3$.]{\includegraphics[width=0.25\linewidth]{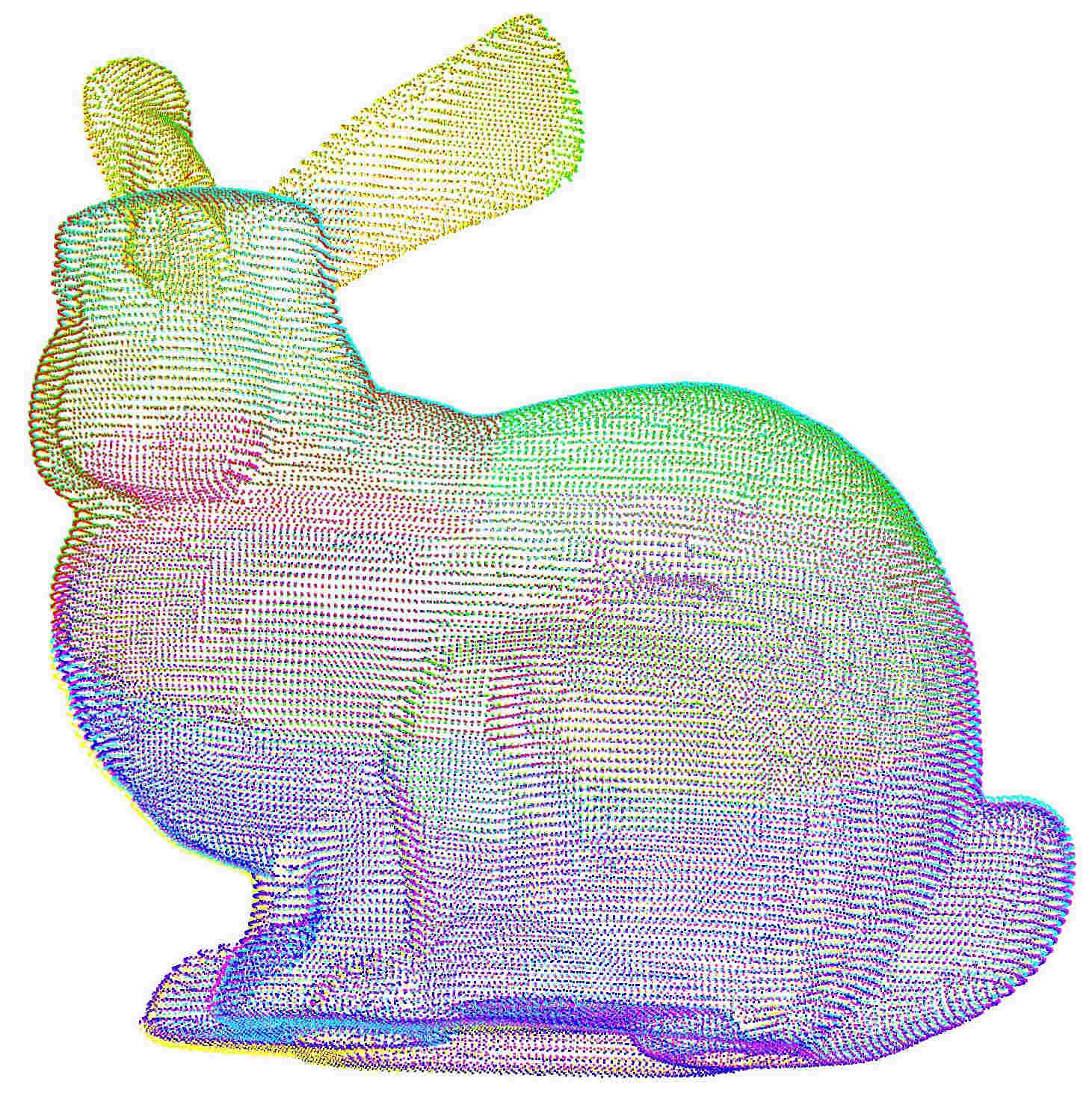}} 
\subfloat[$\eta = 0.5$.]{\includegraphics[width=0.25\linewidth]{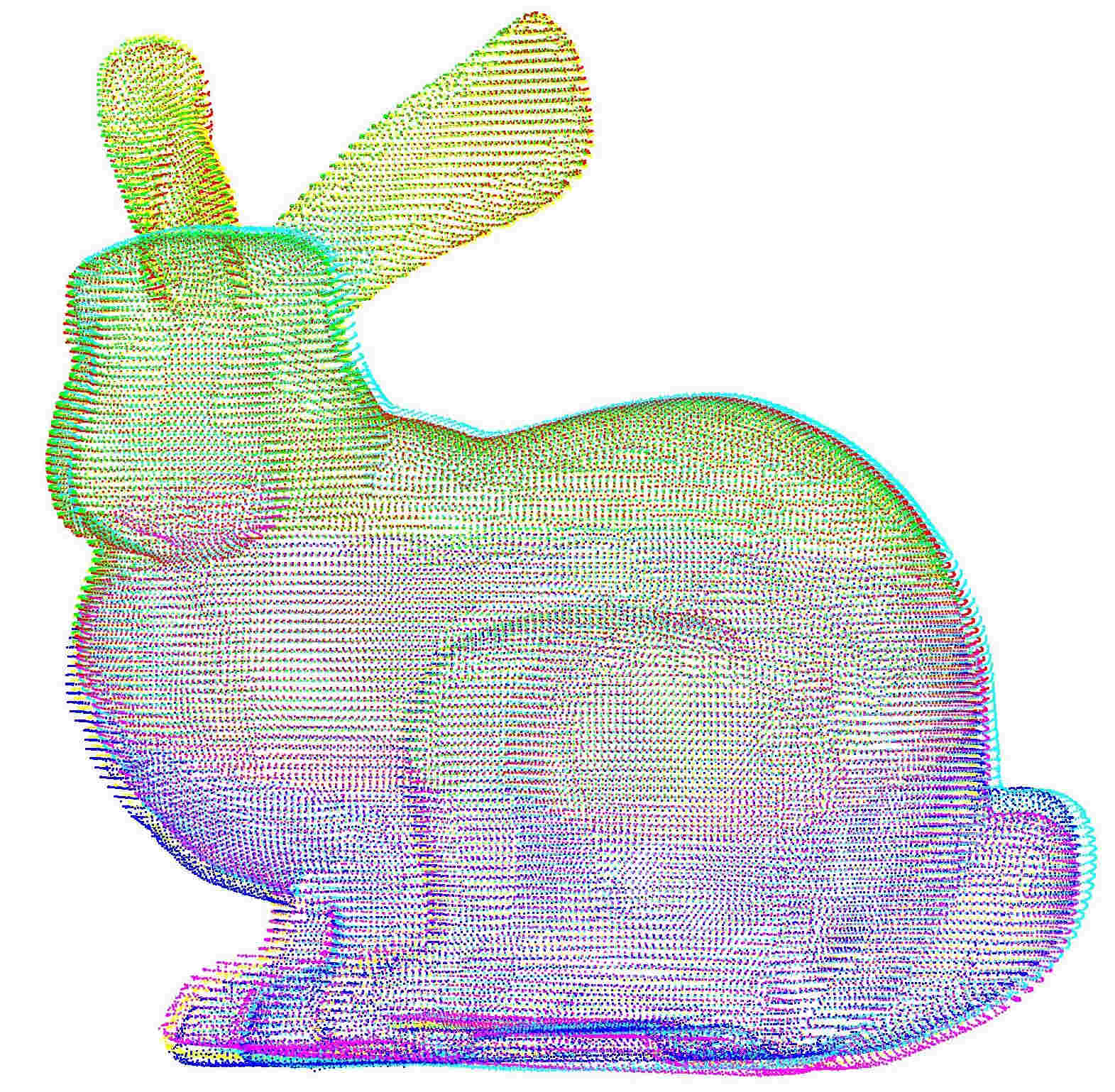}} 
\subfloat[$\eta = 0.7$.]{\includegraphics[width=0.25\linewidth]{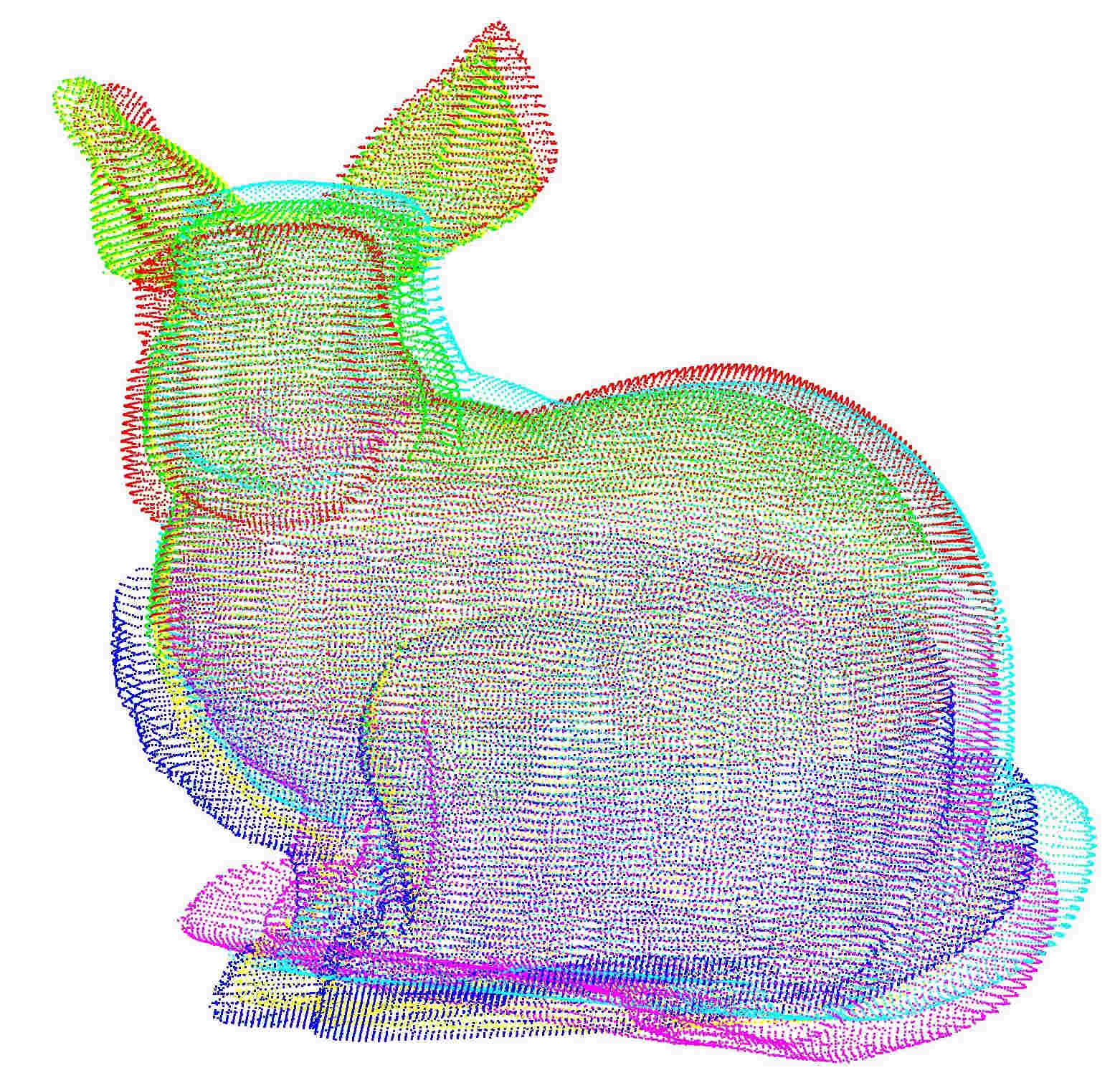}}
\caption{Reconstruction results for \textit{Bunny} with different proportion of outliers. Notice that the reconstruction is quite good even with $70\%$ outliers.} 
\label{fig5}
\end{figure*}

To perturb the coordinates, we add isotropic Gaussian noise of variance $\sigma^2$ to the points in each set. For the correspondence noise, we first fix some $\eta \in [0,1]$ and randomly shuffle $\eta$-fraction of the known correspondences (these are the outliers), keeping other correspondences fixed. The results obtained on the models of \textit{Bunny} and \textit{Buddha} are shown in Figures \ref{fig3} and \ref{fig4}. 
We have used $10$ scans per model for the experiments.
Notice that the rotation error scales gracefully with increase in $\sigma$ and $\eta$. Moreover, thanks to the global registration framework, we can obtain accurate reconstructions even with  $50\%$ outliers. This is further demonstrated using some visual results in Figure \ref{fig5}. 
Some more visual results for \textit{Bunny}, \textit{Buddha}, \textit{Armadillo} and \textit{Dragon} are shown in Table \ref{NRR}. In each case, we have used six point sets. The point sets before and after registration are shown in the figure along with the ground truth.

As in Umeyama's example, the ADMM objective was found to converge to the global minimum (zero) in the noiseless case. Since it is not possible to determine the global minimum in general for multiple point sets, we cannot decide on optimality in the presence of noise. However, the iterates were found to converge to a fixed point in such cases.


\begin{table}
\centering
\caption{Registration results from known correspondences. 
Left to right: ground truth, before registration, and after registration. Six point sets were used in each case (which are color coded). We used $\rho=10$ for all the experiments. }
\begin{tabular}{|m{2.5cm}|m{2.5cm}|m{2.5cm}|}
\cline{1-3}

\raisebox{-0.5\totalheight}{\includegraphics[width=0.9\linewidth]{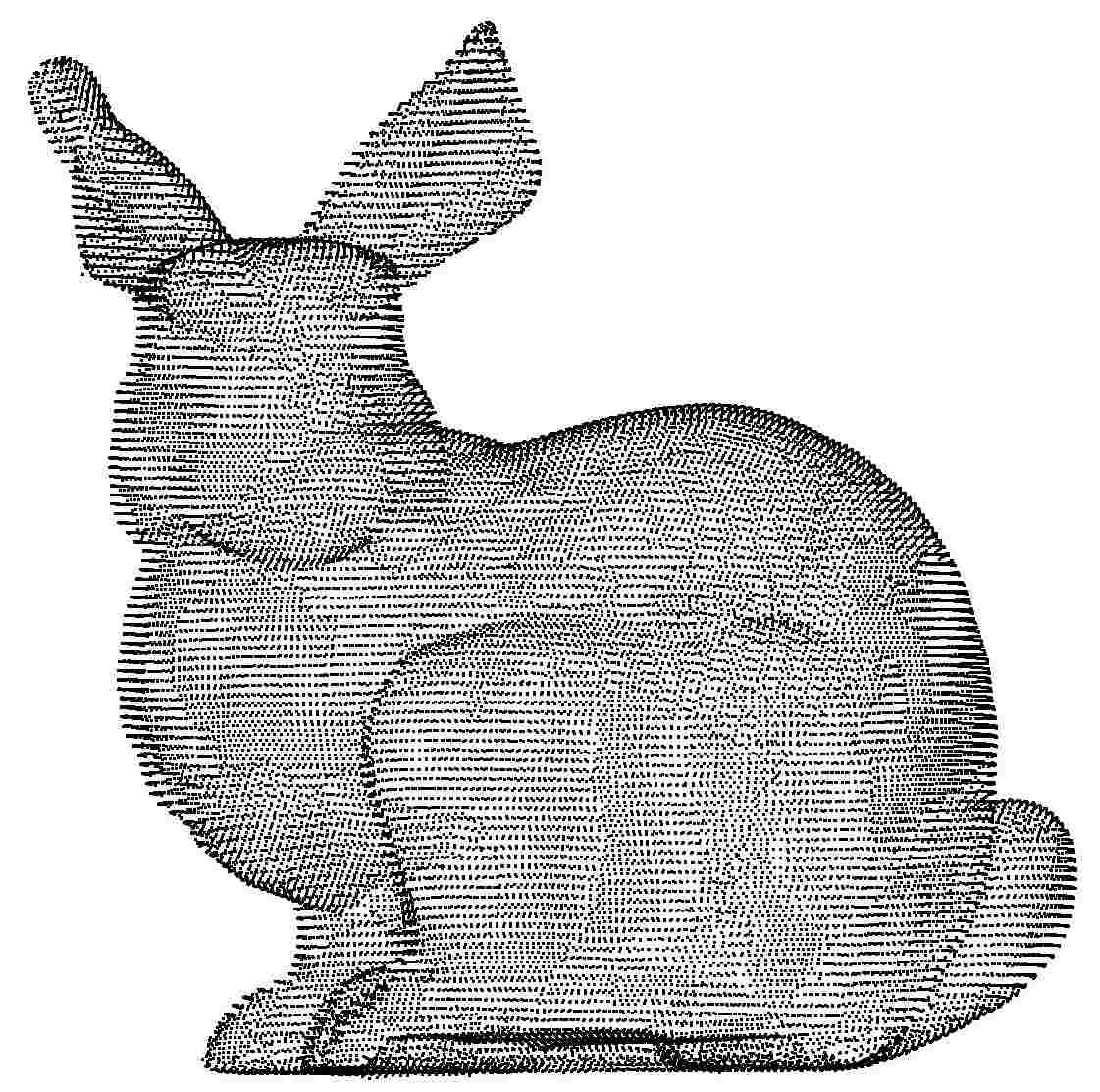}}  &   \hspace{3mm} \raisebox{-\totalheight}{\includegraphics[width= 0.9\linewidth]{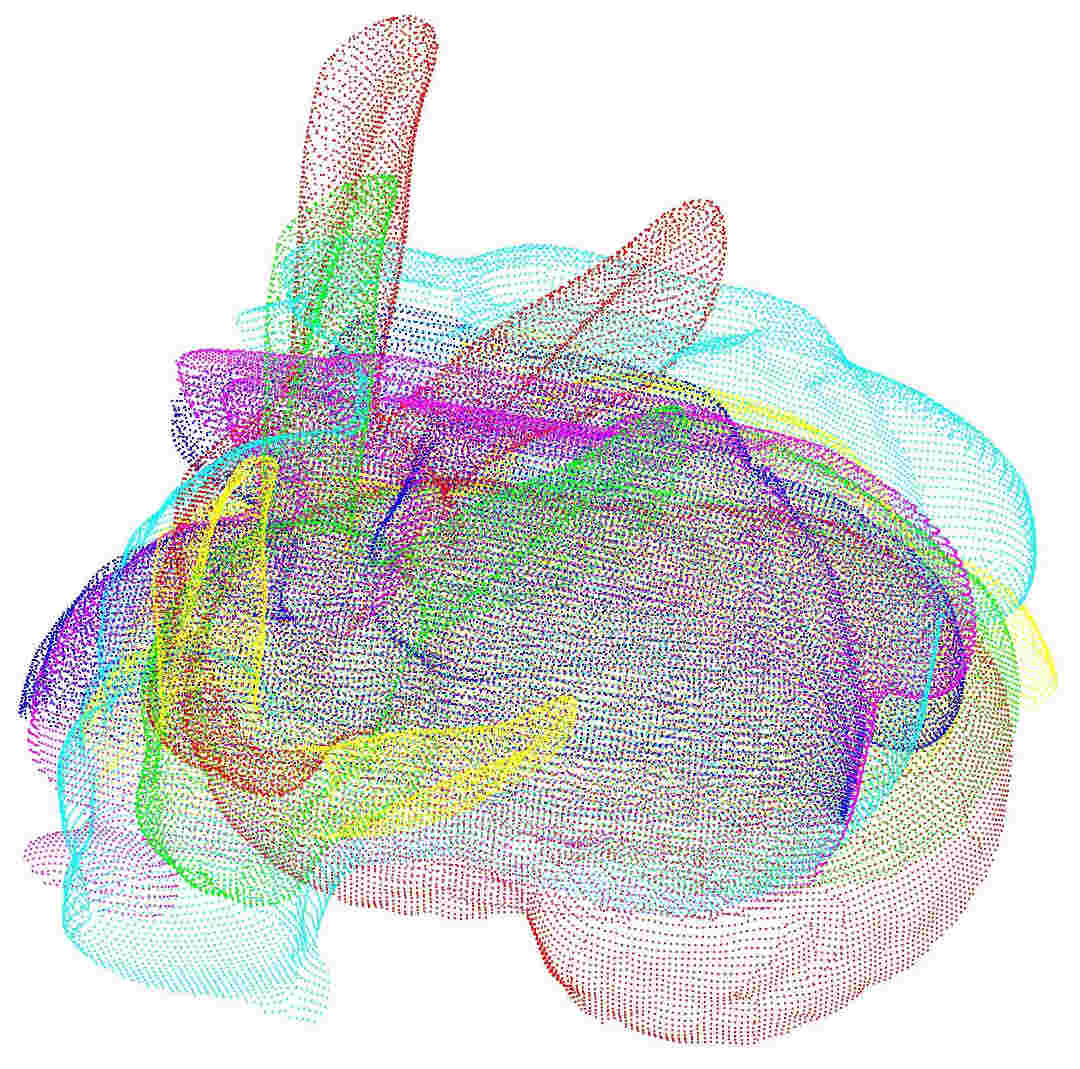}} \vspace{0.5mm}&     \raisebox{-0.5\totalheight}{\includegraphics[width= 0.9\linewidth]{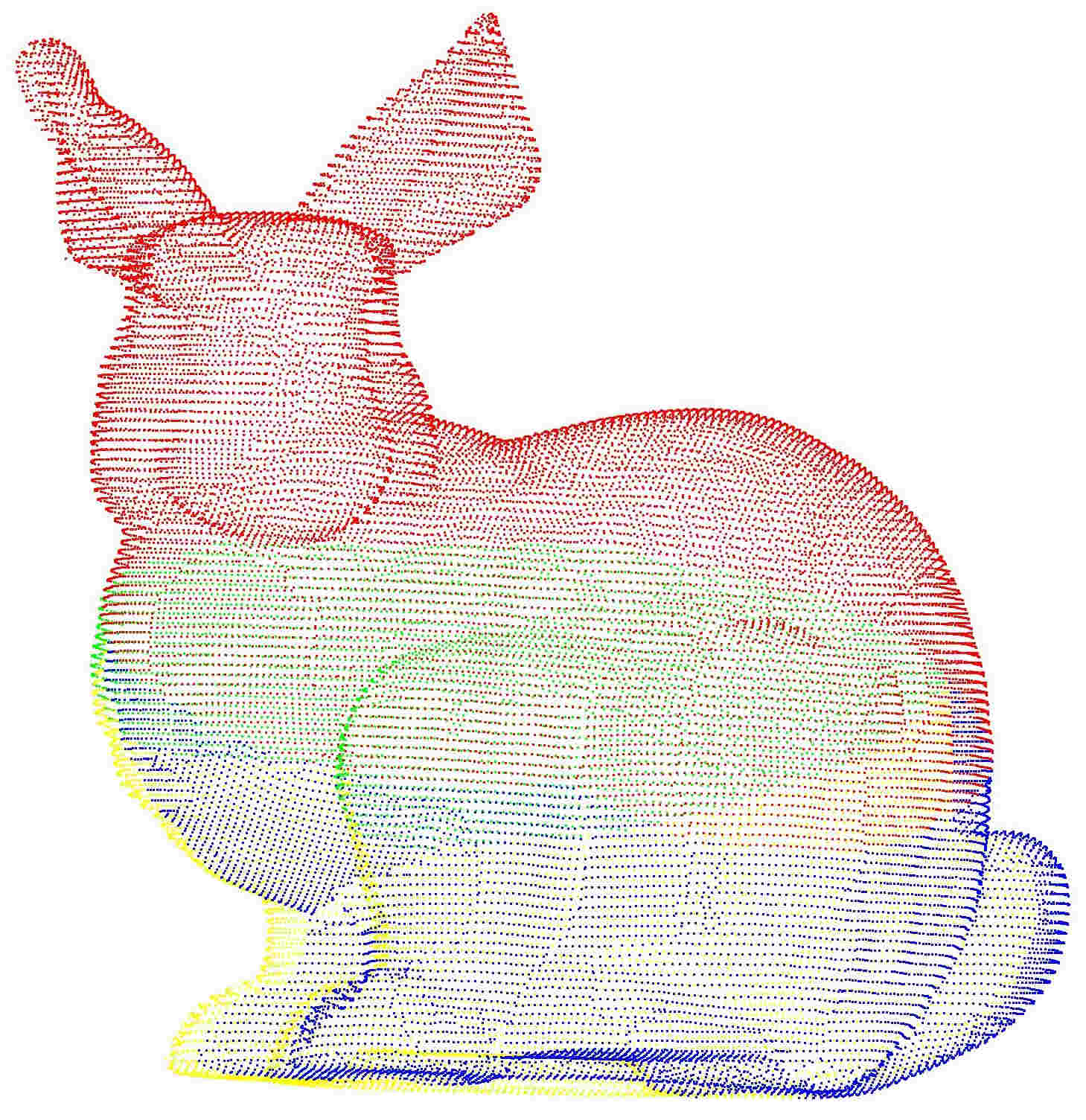}} \\ \hline
 \cline{1-3} 
 
\hspace{3mm} \raisebox{-0.1\totalheight}{\includegraphics[width= 0.9\linewidth]{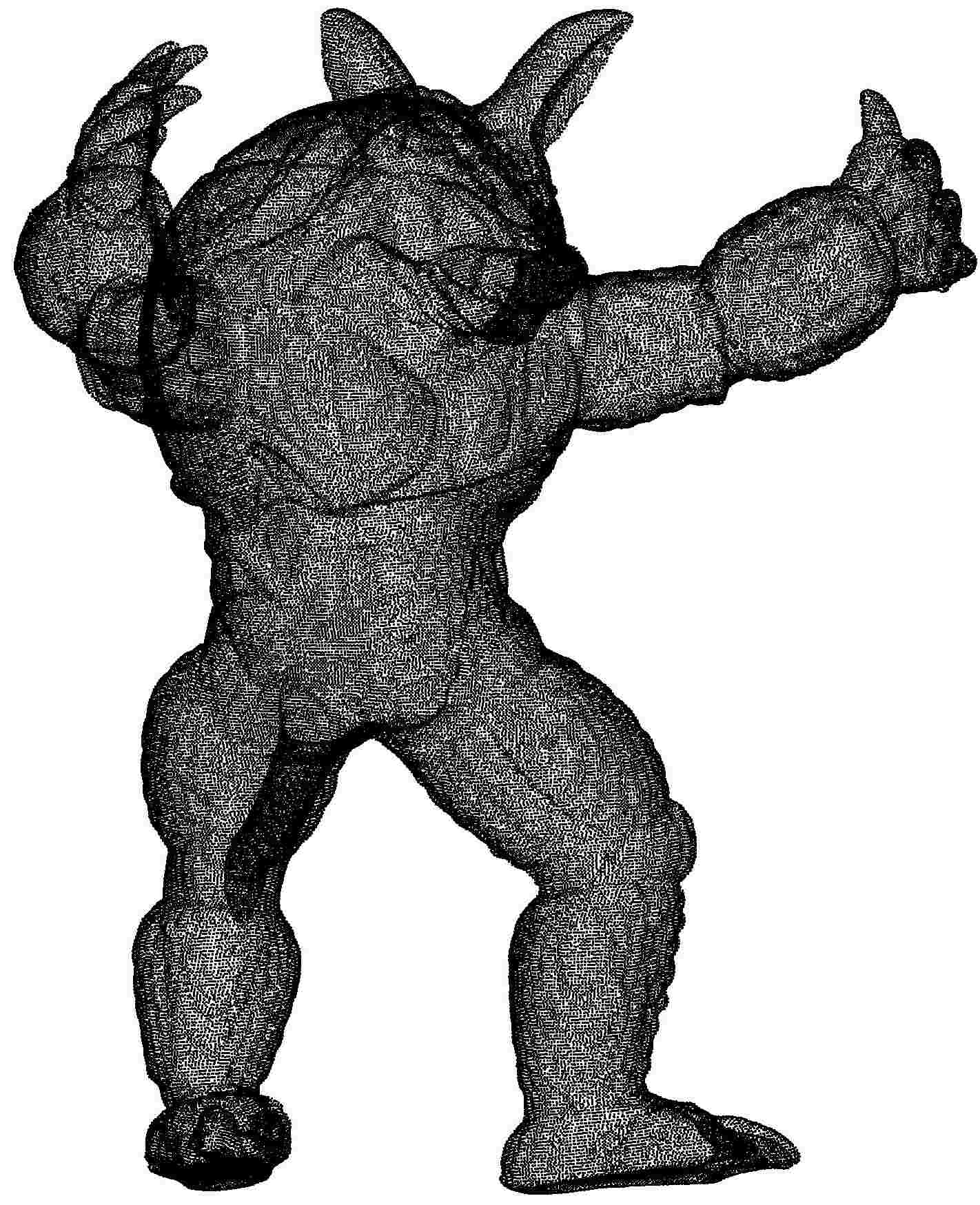}}  &  \hspace{2mm}  \raisebox{-\totalheight}{\includegraphics[width= 0.9\linewidth]{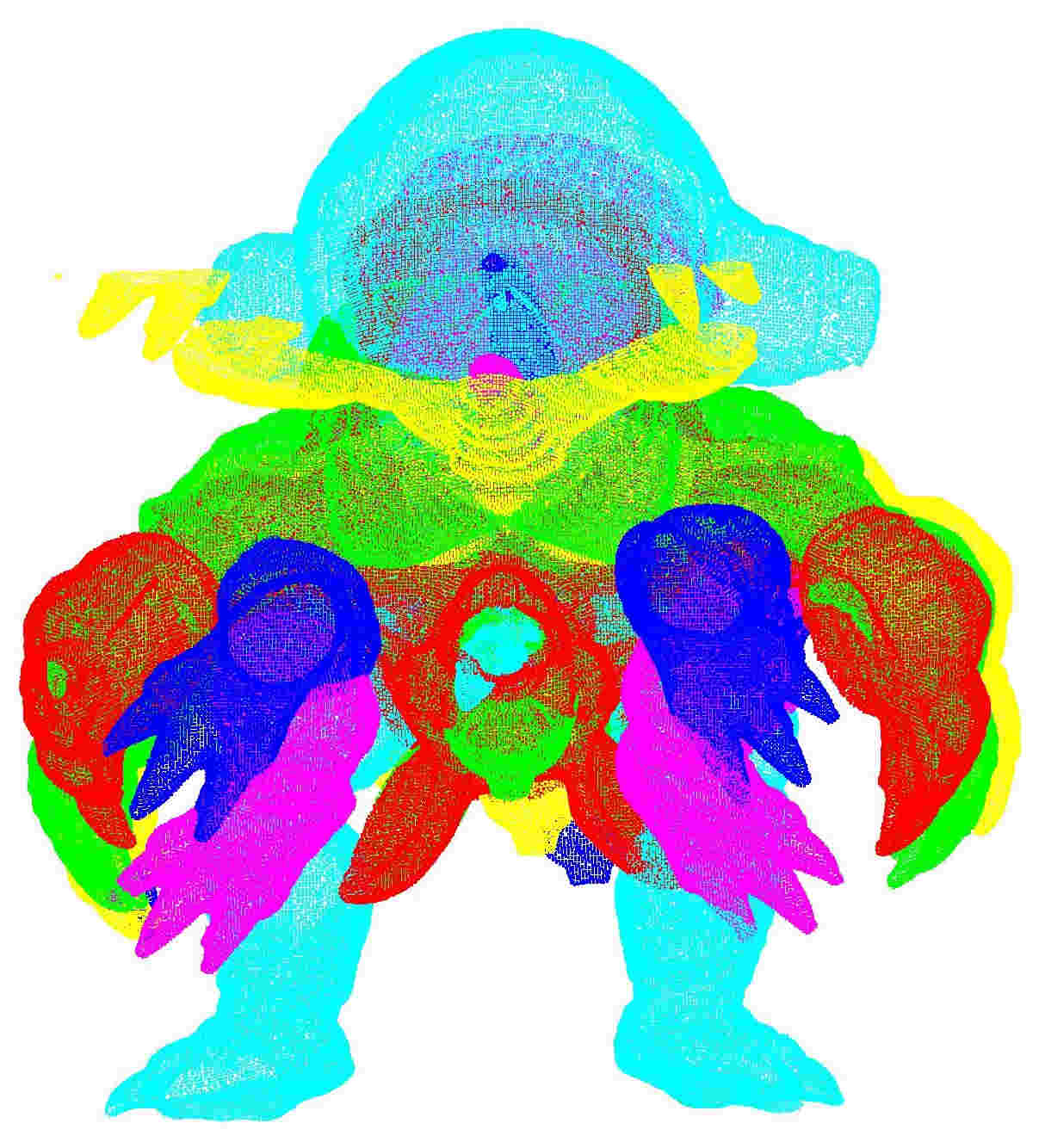}} \vspace{0.5mm}&  \hspace{3mm}   \raisebox{-0.1\totalheight}{\includegraphics[width= 0.9\linewidth]{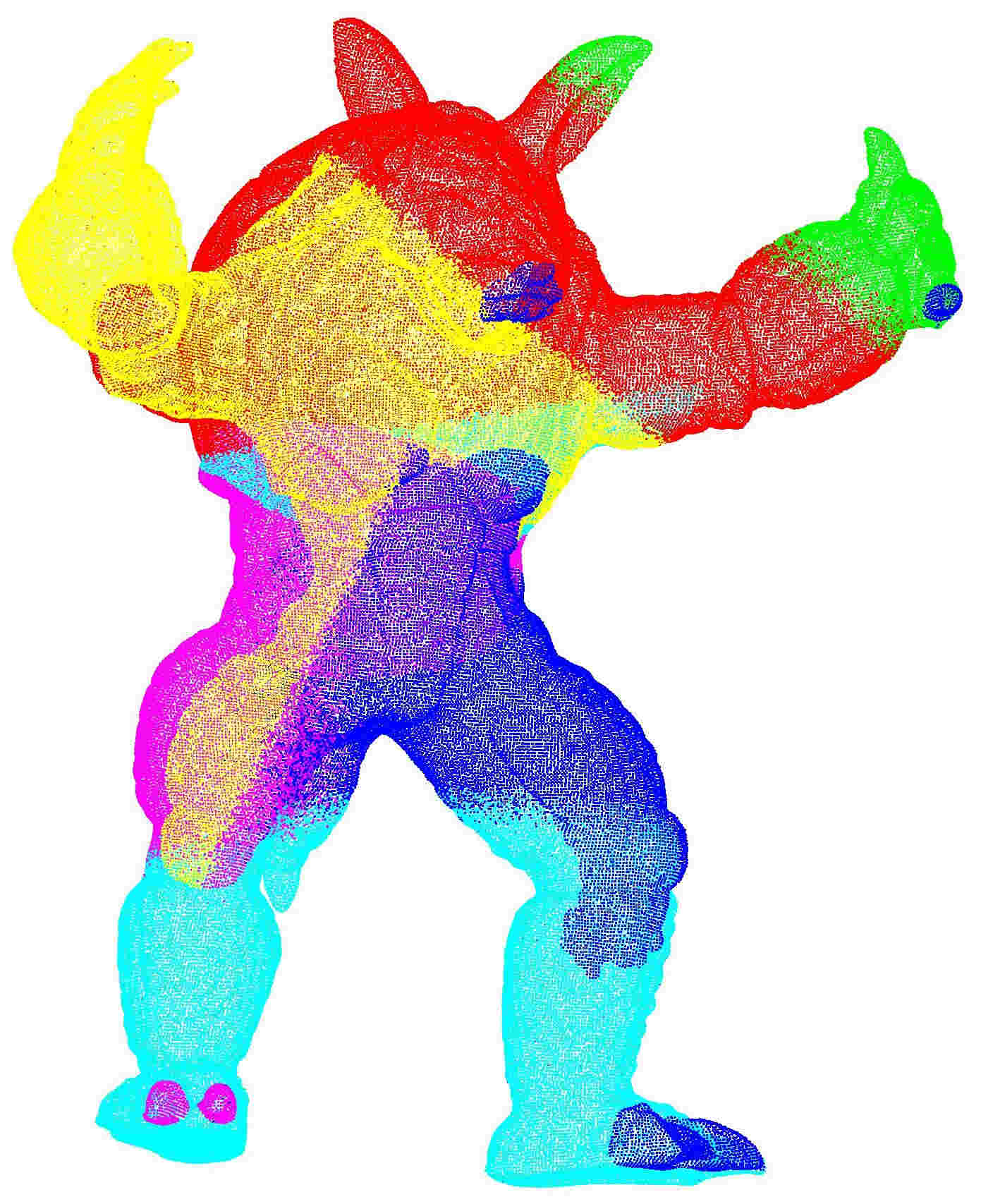}}         \\ \hline
 
\hspace{5mm}\raisebox{-0.5\totalheight}{\includegraphics[width=0.6\linewidth]{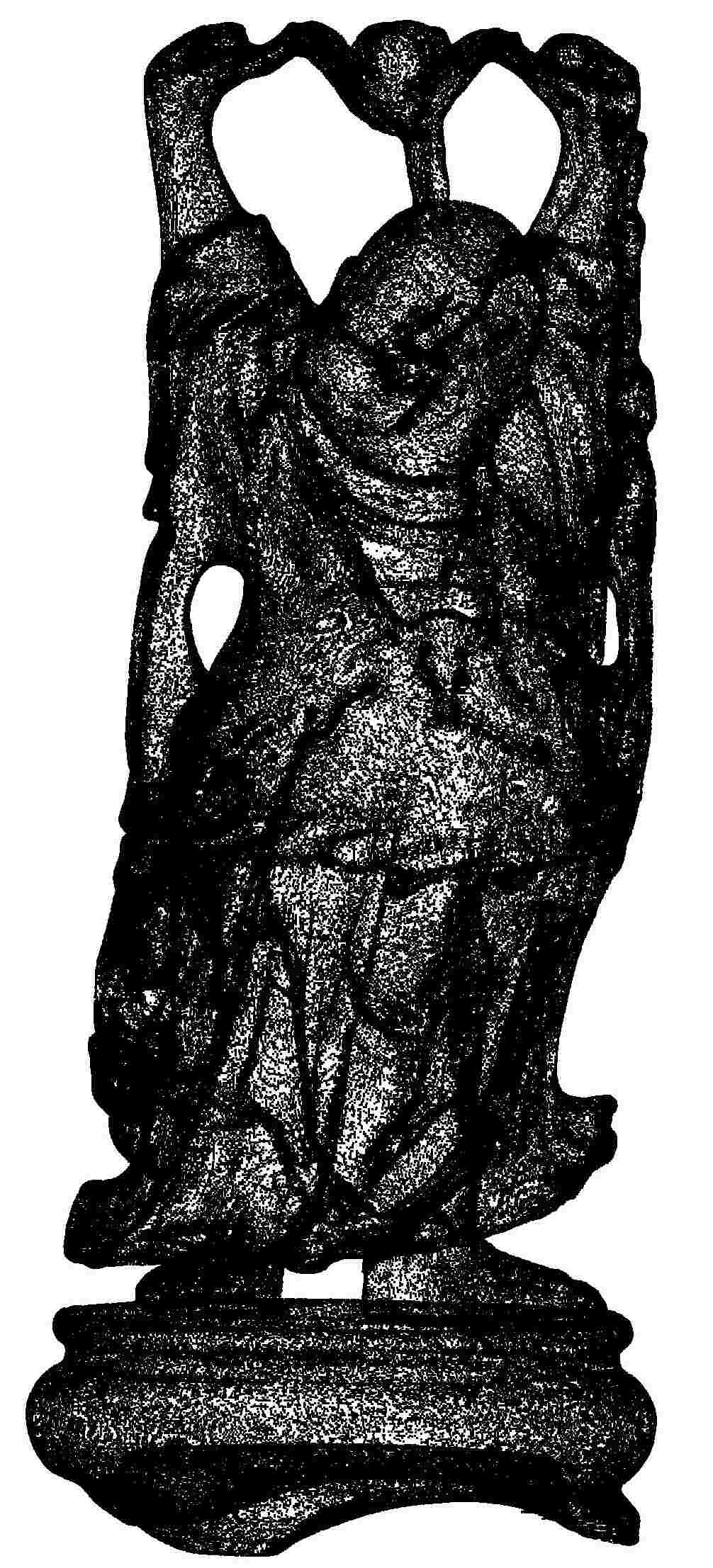}}  &    \raisebox{-0.5\totalheight}{\includegraphics[width= \linewidth]{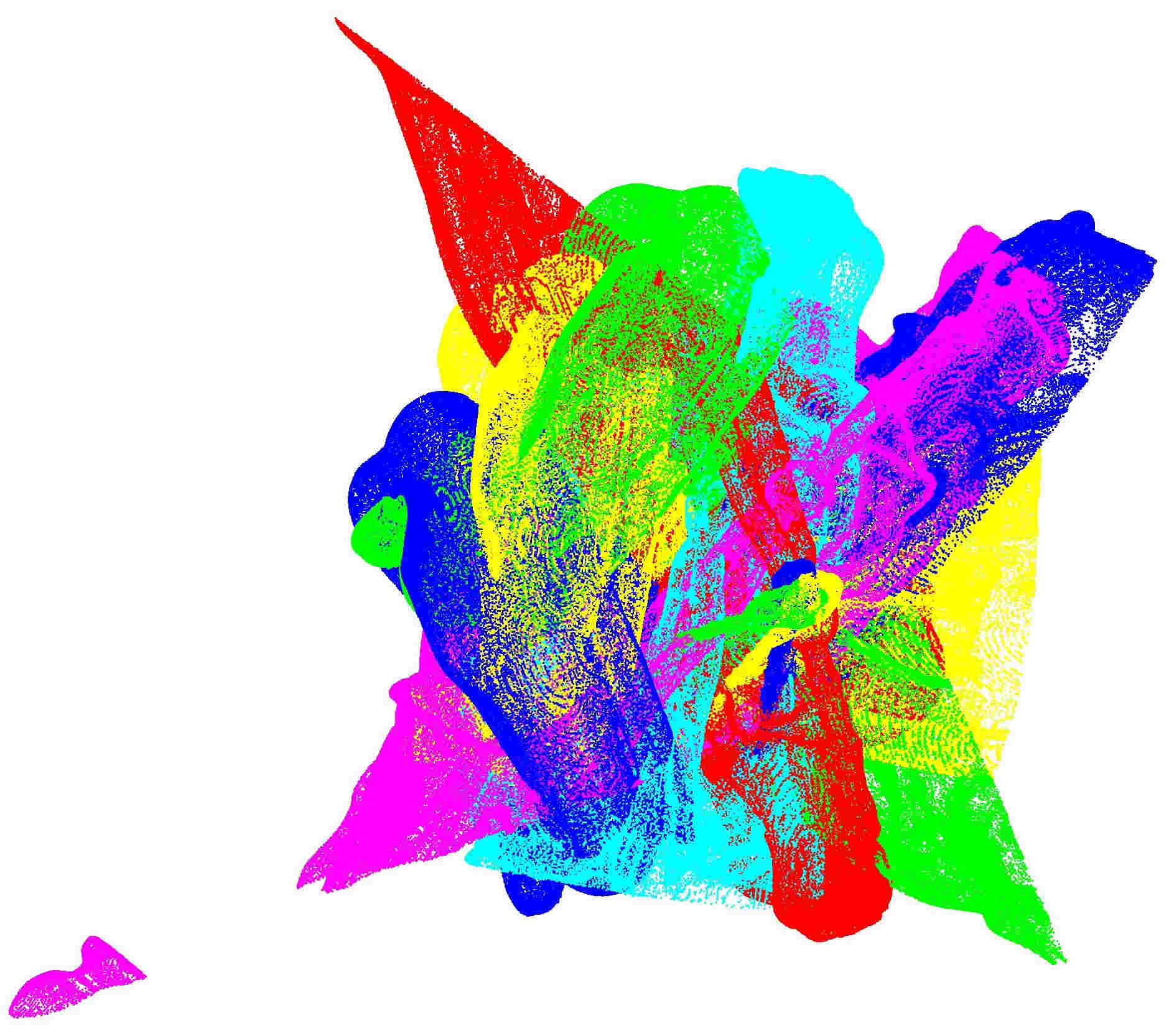}} \vspace{0.5mm}&  \hspace{5mm}   \raisebox{-0.5\totalheight}{\includegraphics[width= 0.6\linewidth]{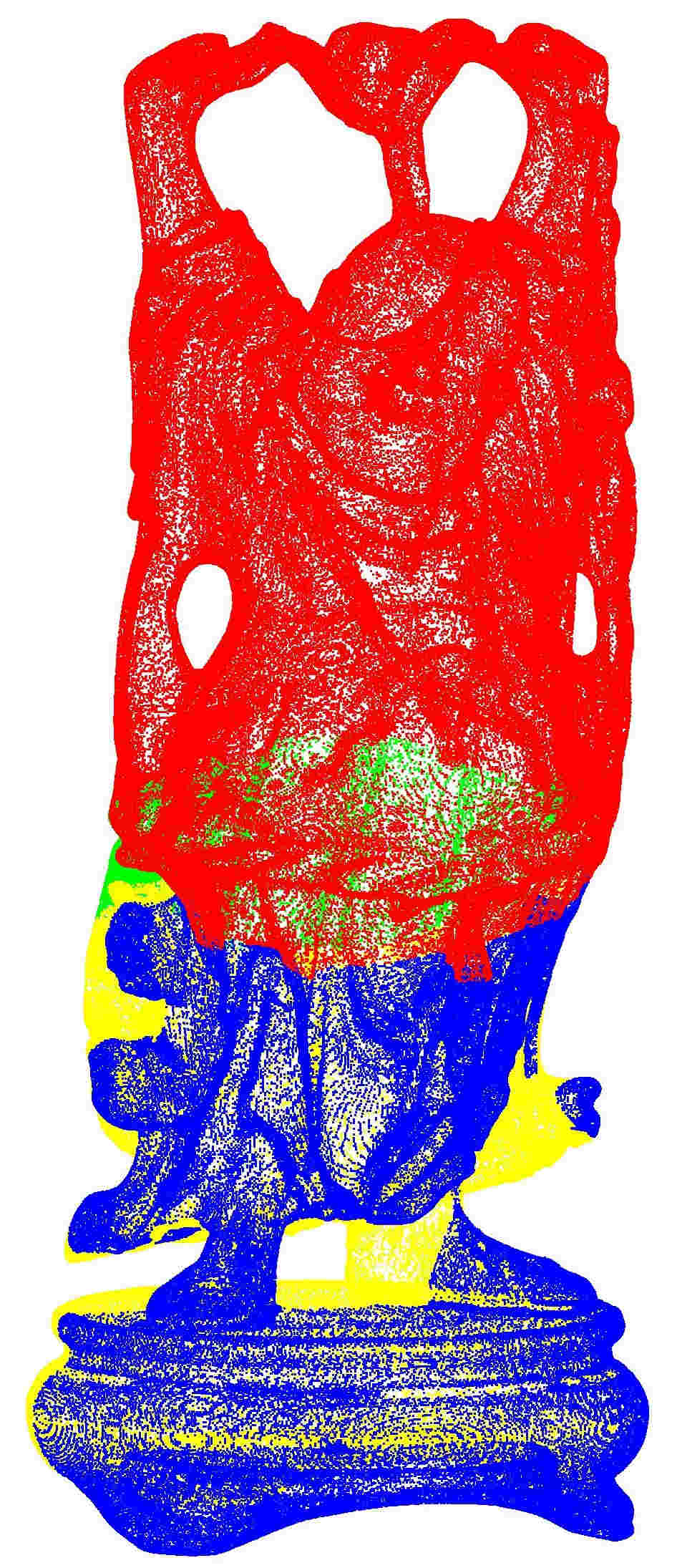}}  \\ \hline

\raisebox{-0.5\totalheight}{\includegraphics[width= \linewidth]{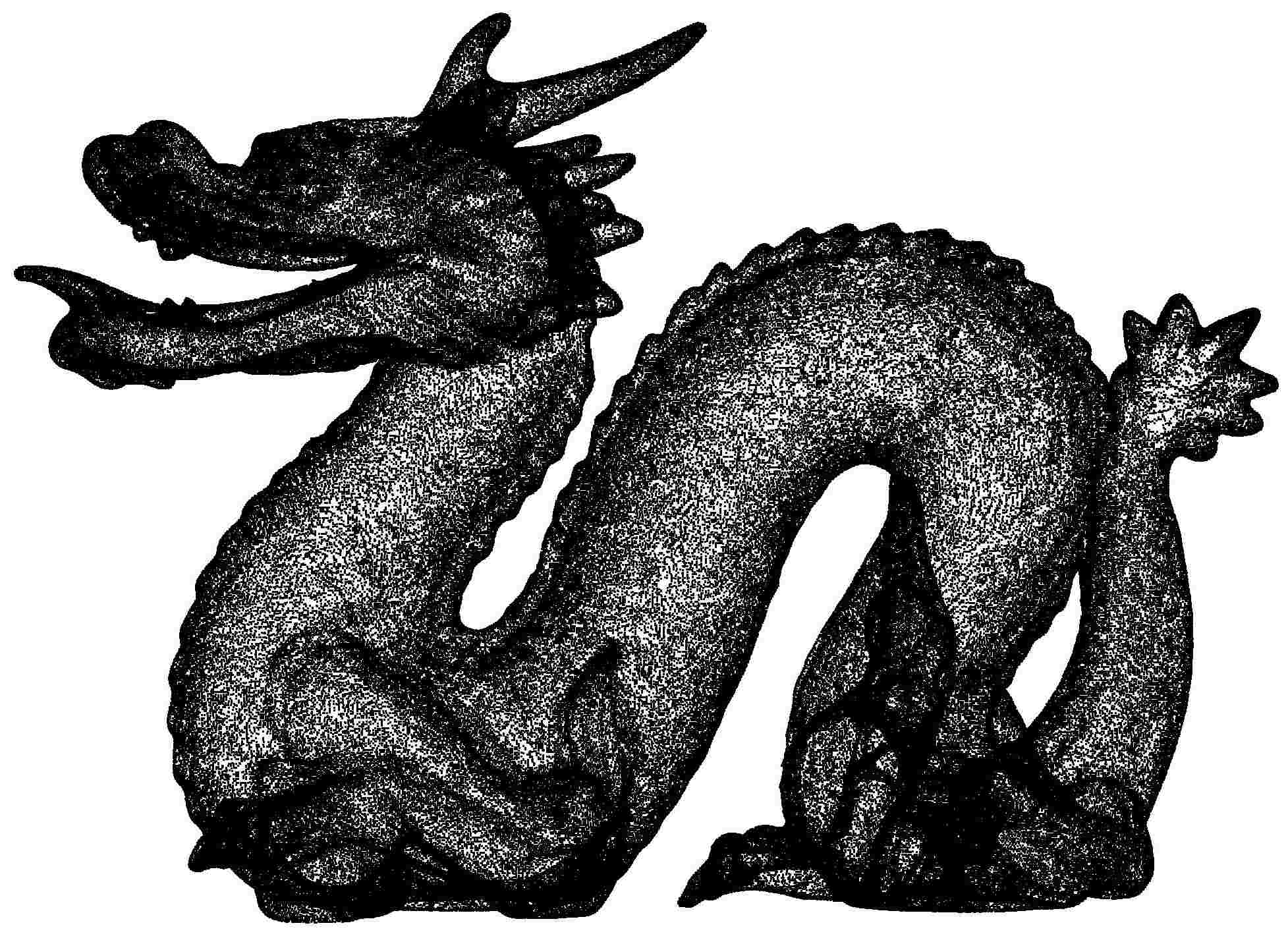}}  &    \raisebox{-0.5\totalheight}{\includegraphics[width= \linewidth]{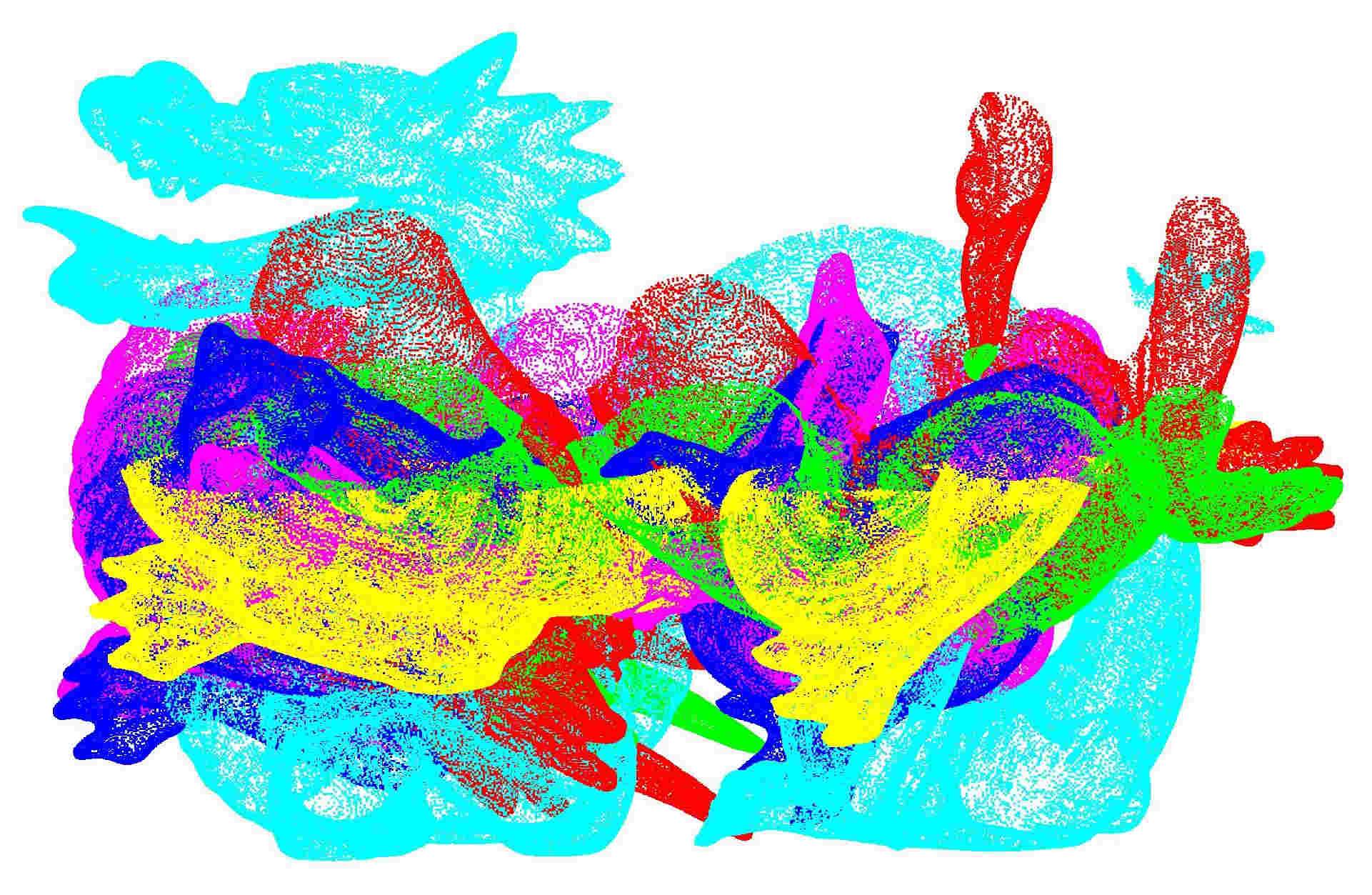}} \vspace{0.5mm}&     \raisebox{-0.5\totalheight}{\includegraphics[width= \linewidth]{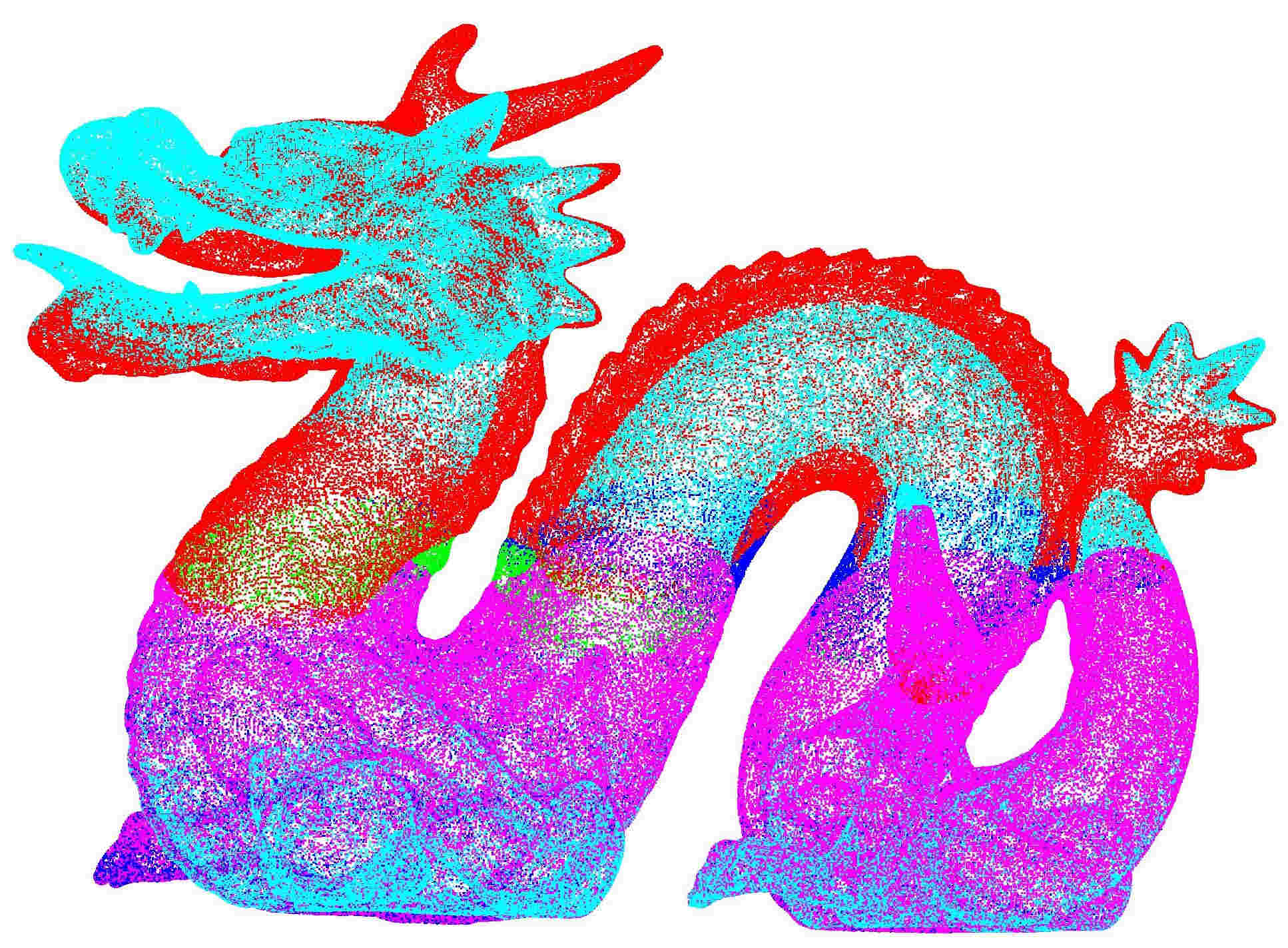}}  \\ \hline
\end{tabular}\\~\\
\label{NRR}
\end{table}

Going back to the questions posed at the start of this section, we have empirically noticed the following:
\begin{itemize}
\item The iterates of Algorithm \ref{algo} converge to a fixed point for any arbitrary initialization when $\rho >0$. The convergence is generally fast if $\rho \in [1,10]$ and spectral initialization is used.  

\item For special cases where we know the global minimum of \eqref{ncvxSDP}, the objective remarkably converges to the global minimum if we use the spectral initialization. Generally, the algorithm always seemed to work in the noiseless setting.

\item The algorithm behaves stably with perturbations in the coordinates and correspondences.
\end{itemize}


For completeness, we compare the proposed algorithm with \cite{Ahmed2017}, where the optimization in \eqref{manopt} is performed over $\mathbb{O}(d)$. We extract ten point sets from \textit{Bunny} and randomly shuffle $60\%$ of the correspondences. We feed this data into the algorithms  and check the determinants of the computed transforms. The results of an experiment are presented in Table \ref{DC}. Notice that the  transforms computed by our algorithm are indeed rotations, whereas the algorithm in \cite{Ahmed2017} returns a mix of rotations and reflections. Moreover, the reconstruction error is much higher for the latter (see Figure \ref{ON}).

\begin{figure*}[!htpb]
\center
\subfloat[\cite{Ahmed2017}.]{\includegraphics[width=0.35\linewidth]{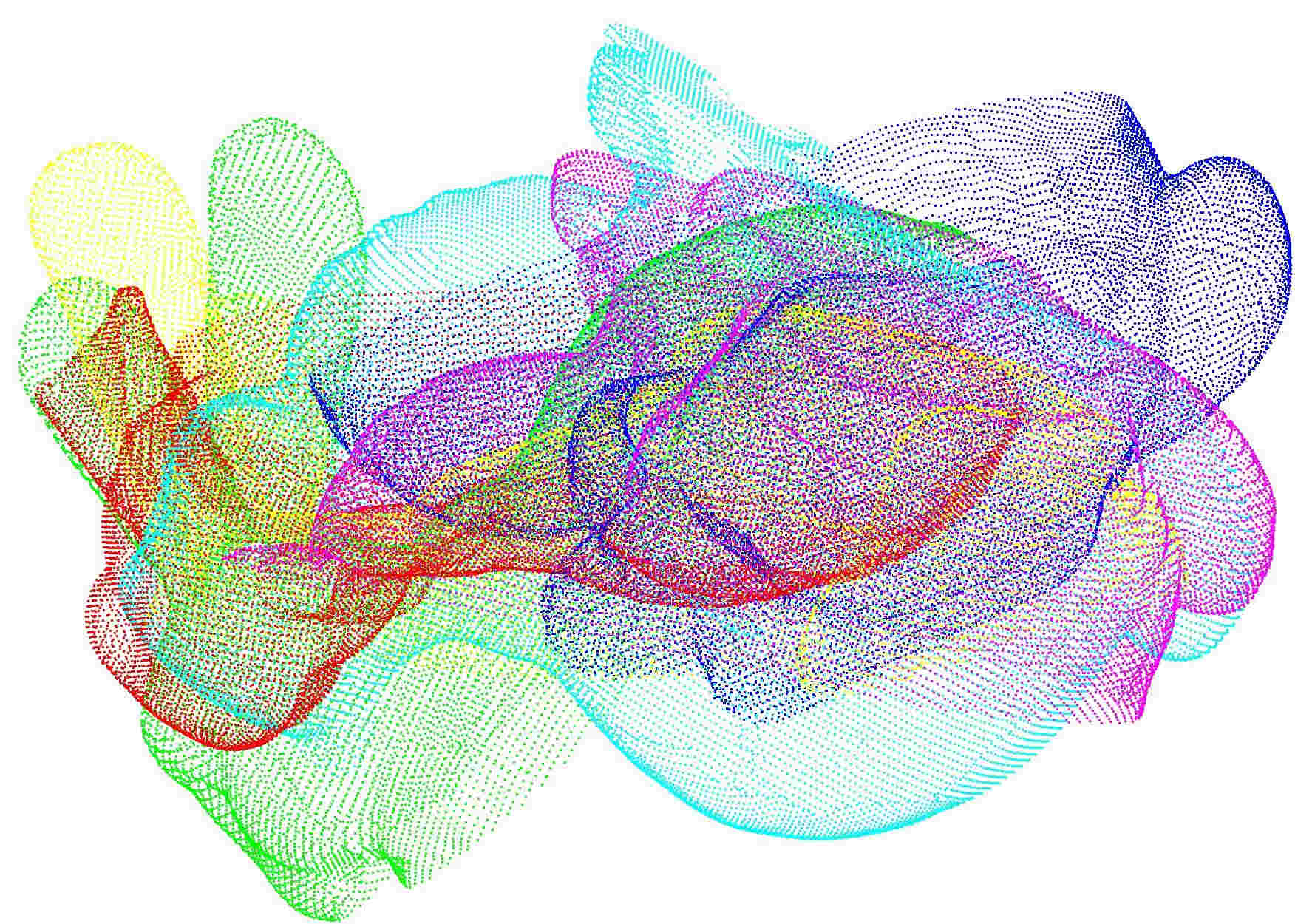}}\hspace{10.em}
\subfloat[Proposed.]{\includegraphics[width=0.28\linewidth]{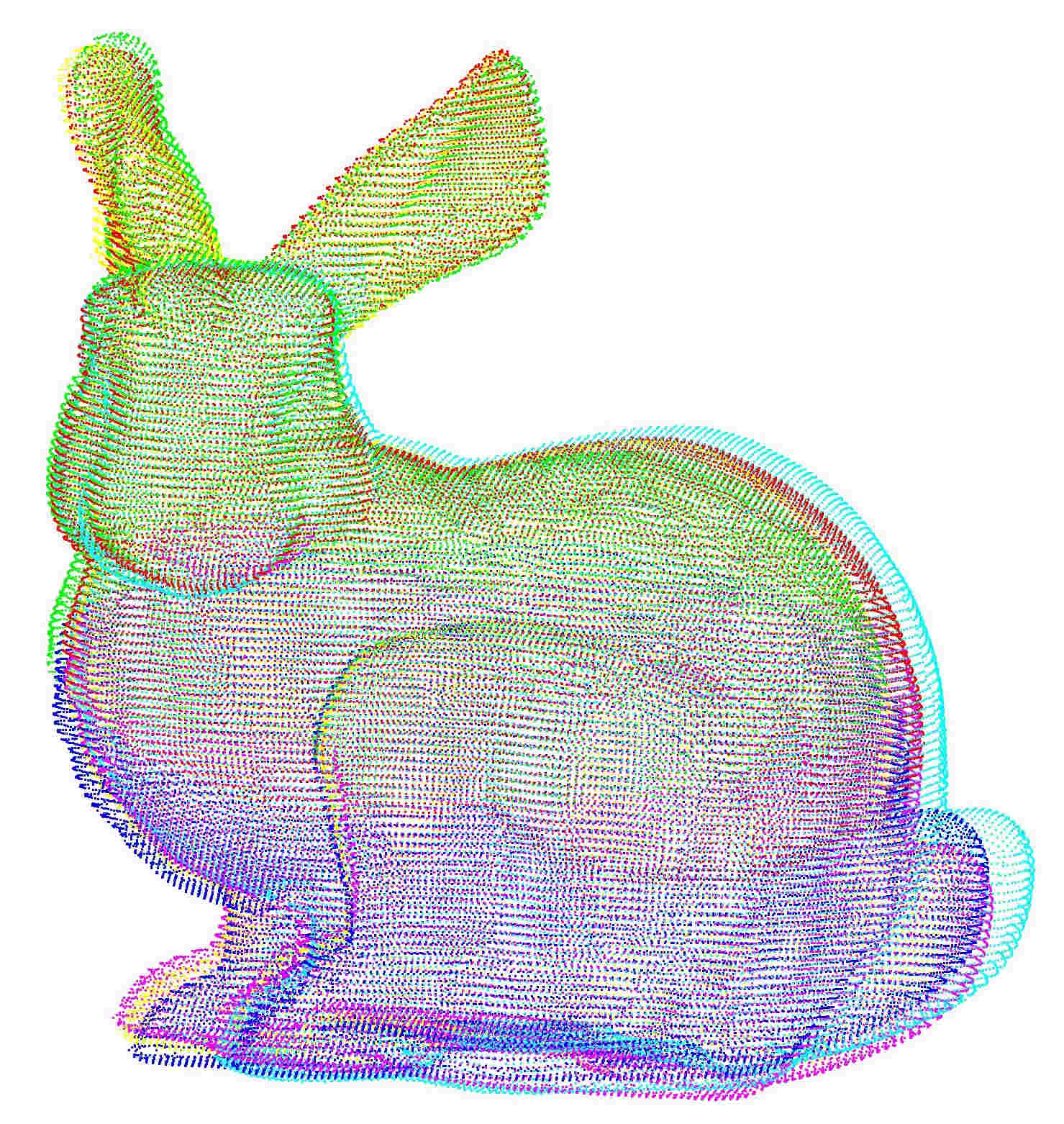}}
\caption{Reconstruction of \textit{Bunny} from known correspondences with $60$\% outliers. The poor reconstruction in (a)  is due to reflections.} 
\label{ON}
\end{figure*}

\begin{figure}[!htpb]
\center
\subfloat[ \textit{Armadillo}.]{\includegraphics[height=0.35\linewidth,width=0.3\linewidth]{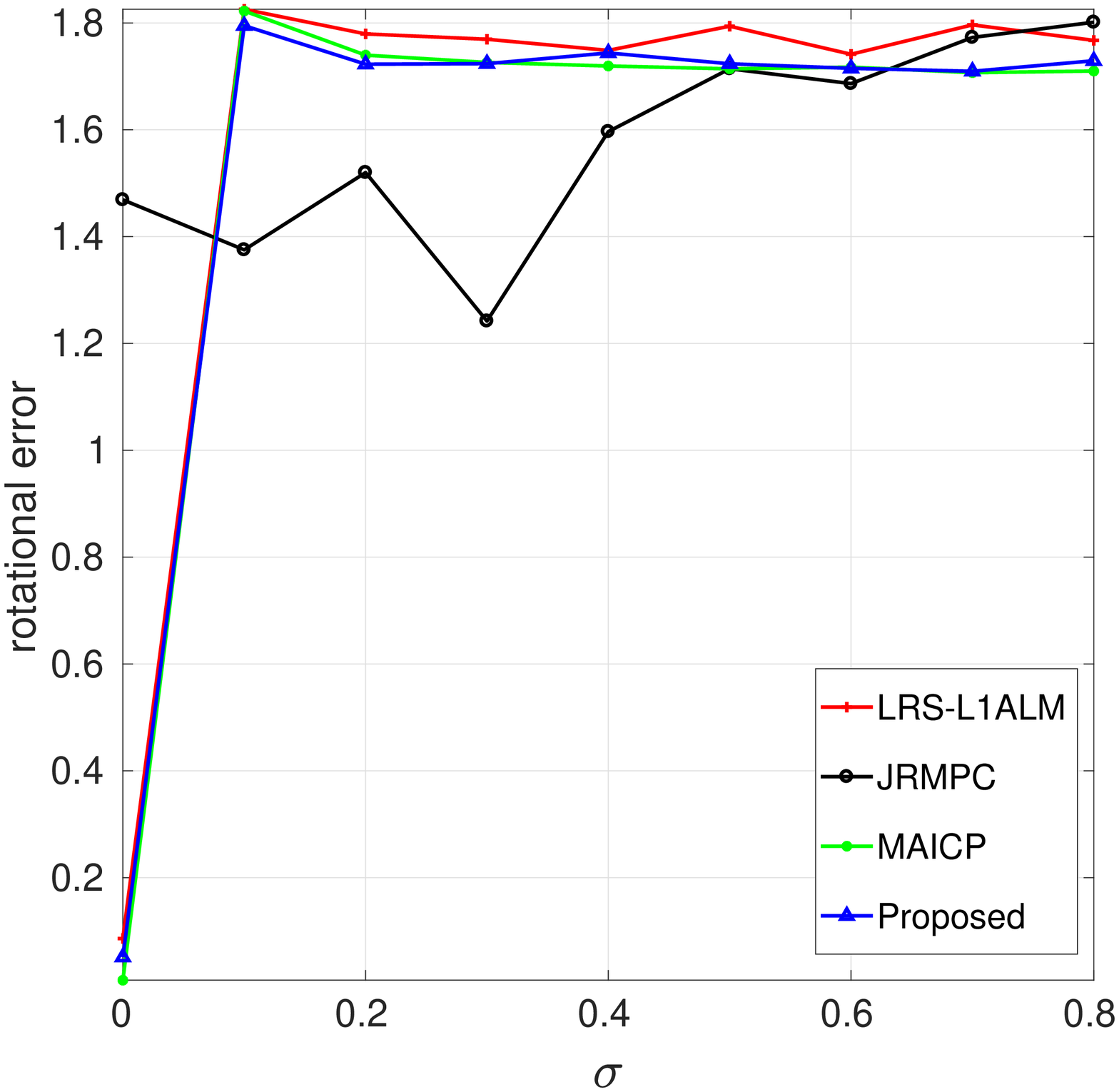}} 
\subfloat[ \textit{Dragon}.]{\includegraphics[height=0.35\linewidth,width=0.3\linewidth]{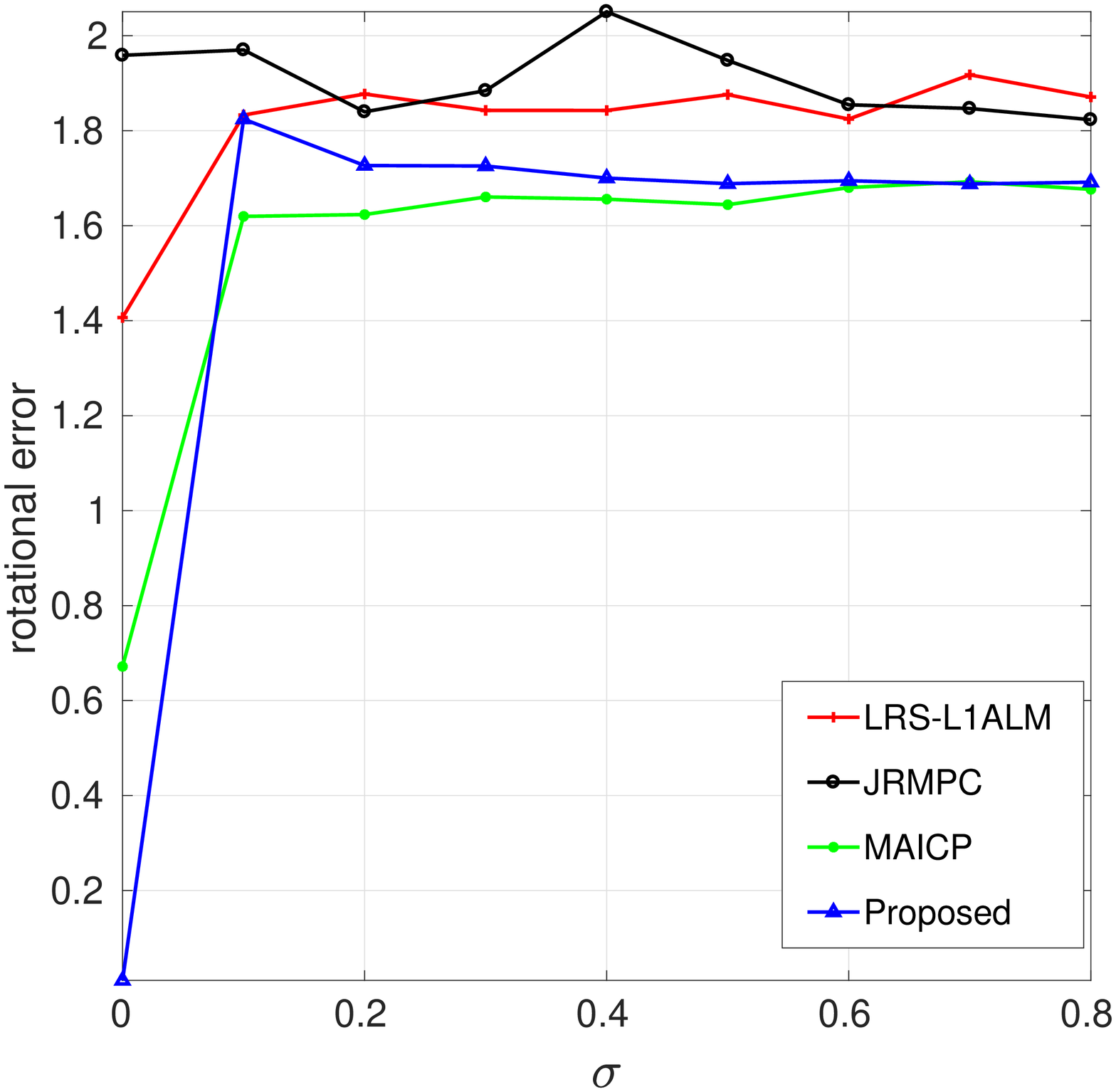}}
\caption{Comparison with existing methods for different noise models.} 
\label{FR}
\end{figure}

\begin{table*}
\centering

\begin{tabular}{|m{3.5cm}|m{3.5cm}|m{3.5cm}|m{3.5cm}|}
\cline{1-4}
\hspace{1cm} Model        & \hspace{1cm} Scans & \hspace{1cm} Proposed & \hspace{1cm} MAICP  \\ \hline
\raisebox{-\totalheight}{\includegraphics[width= \linewidth]{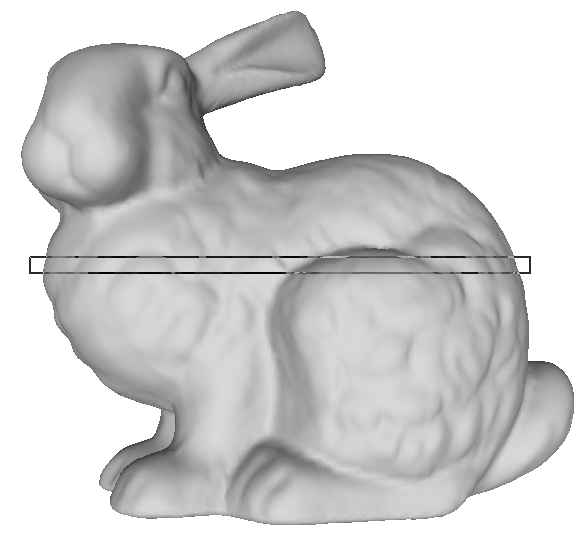}}  &    \raisebox{-\totalheight}{\includegraphics[width= 0.6\linewidth]{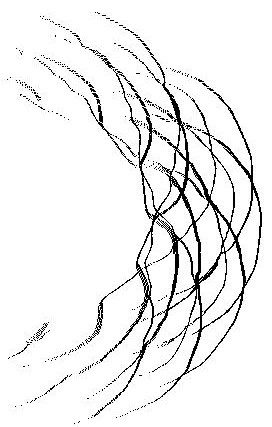}} \vspace{0.5mm}&     \raisebox{-\totalheight}{\includegraphics[width= \linewidth]{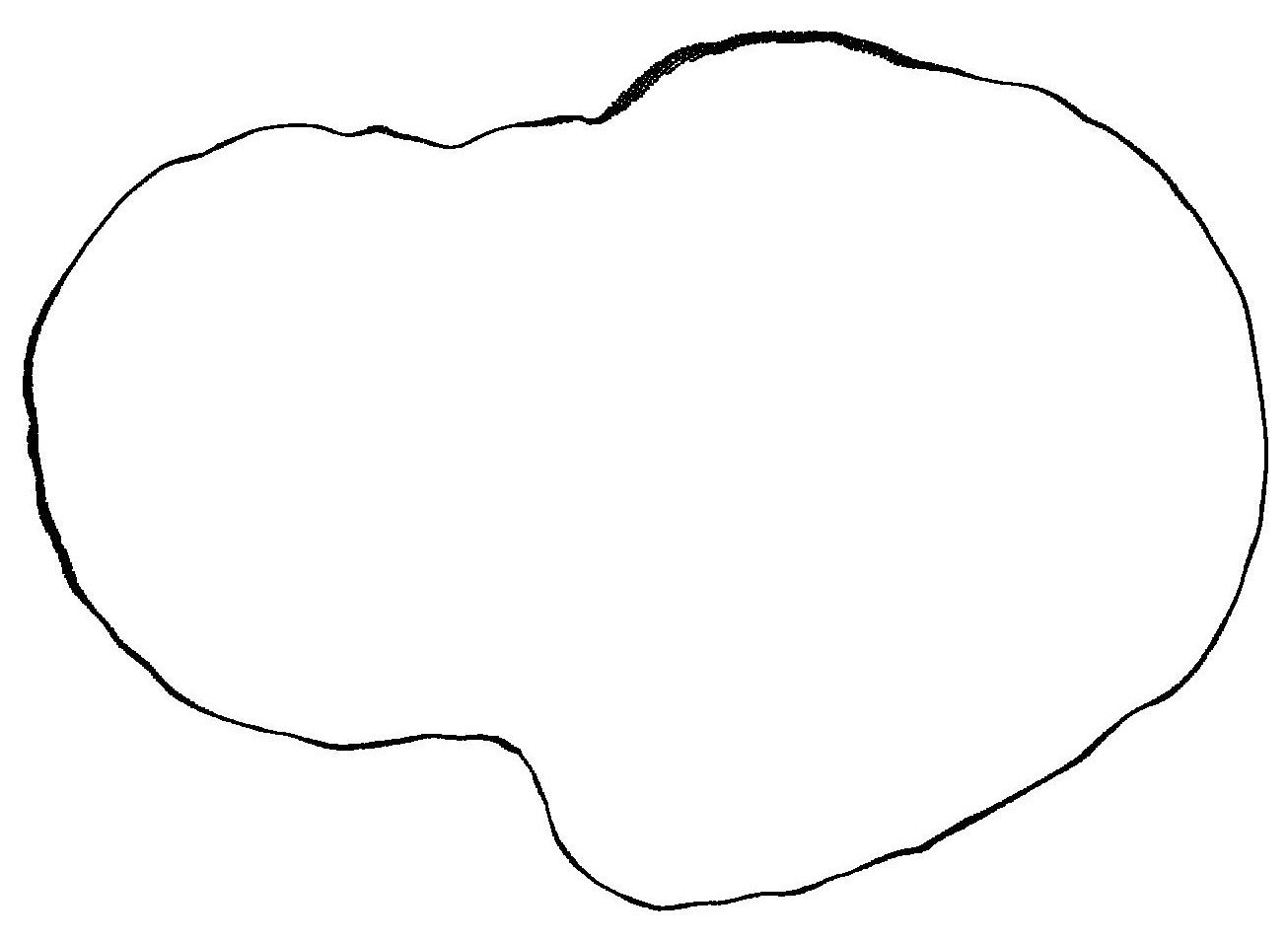}} \vspace{0.5mm} &                                  \raisebox{-\totalheight}{\includegraphics[width= \linewidth]{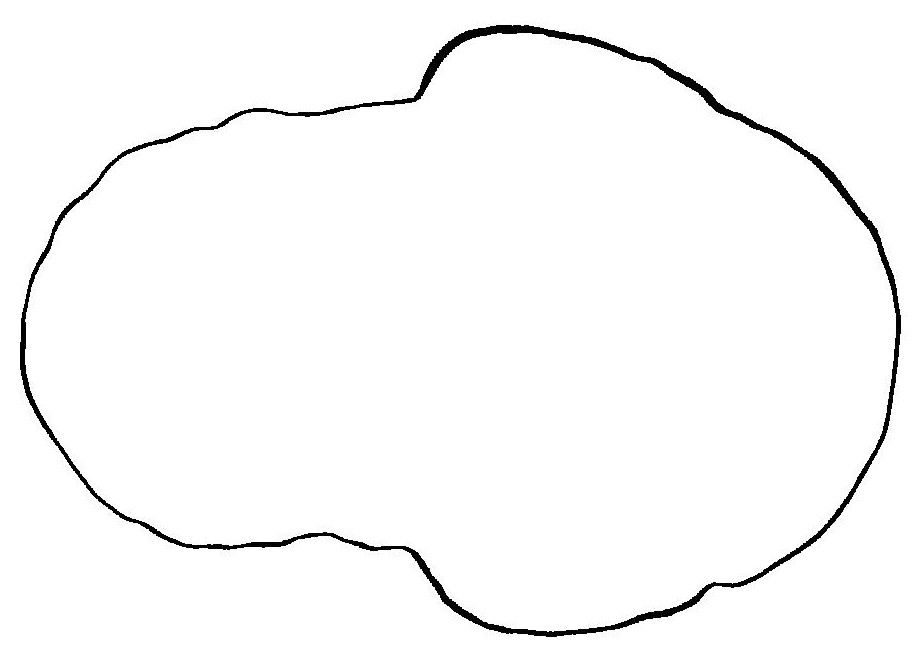}}  \\ 
 \cline{1-4} 
\raisebox{-0.5\totalheight}{\includegraphics[width= \linewidth]{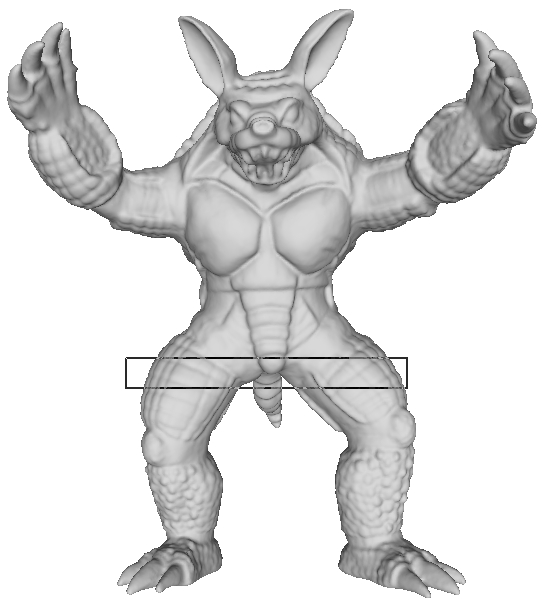}}  &    \raisebox{-\totalheight}{\includegraphics[width= \linewidth]{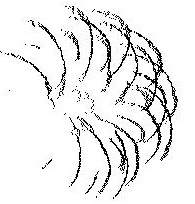}} \vspace{0.5mm}&     \raisebox{-\totalheight}{\includegraphics[width= \linewidth]{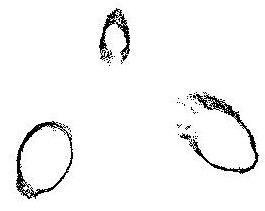}} \vspace{0.5mm} &                                  \raisebox{-\totalheight}{\includegraphics[width= \linewidth]{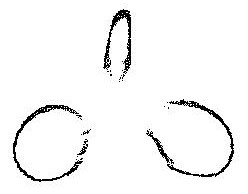}}        \\ \hline
 
 \hspace{8mm}\raisebox{-0.5\totalheight}{\includegraphics[width= 0.5\linewidth, height= 1.25\linewidth]{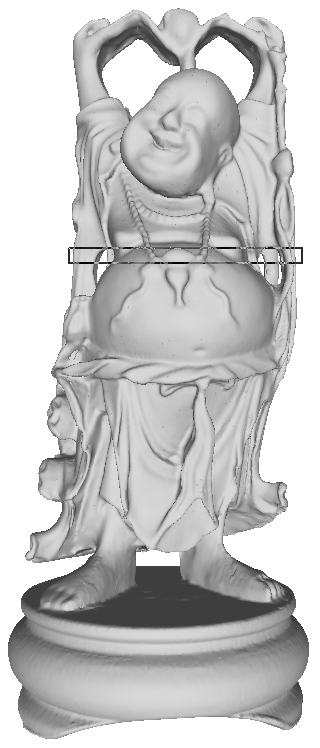}}  &    \raisebox{-\totalheight}{\includegraphics[width= 0.8\linewidth]{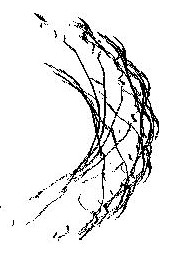}} \vspace{0.5mm}&     \raisebox{-\totalheight}{\includegraphics[width= \linewidth]{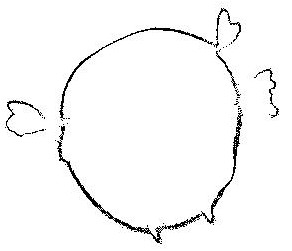}} \vspace{0.5mm} &                                  \raisebox{-\totalheight}{\includegraphics[width= \linewidth]{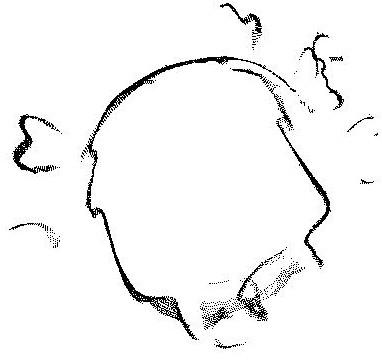}} \\ \hline
 
 \raisebox{-\totalheight}{\includegraphics[width= \linewidth]{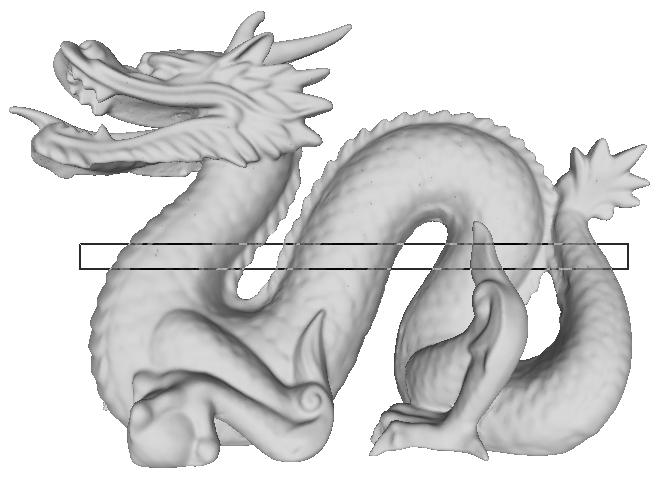}}  &    \raisebox{-\totalheight}{\includegraphics[width= \linewidth]{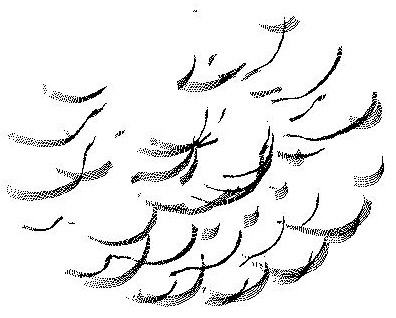}} \vspace{0.5mm}&     \raisebox{-\totalheight}{\includegraphics[width= \linewidth]{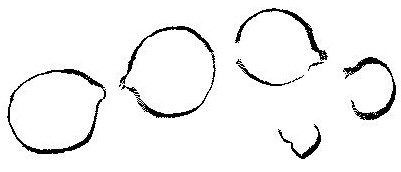}} \vspace{0.5mm} &                                  \raisebox{-\totalheight}{\includegraphics[width= \linewidth]{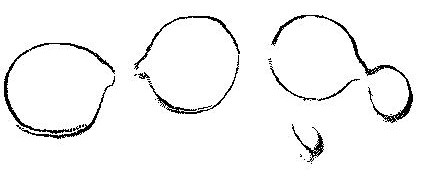}} \\ \hline
 
\end{tabular}
\caption{Visual comparison of reconstructions based on their cross-sections. }
\label{FRC}
\end{table*}

\section{Shape Matching}

We now apply our method for matching $2$D shapes (\cite{shapematchmalik}). Without getting into a rigorous analysis, we simply present few results to demonstrate that our method can be used for shape matching. 
In particular, we wish to match a collection of shapes of a $2$D model without having to compute all pairwise similarities. We have used the models \textit{plane, car} and \textit{bicego} from the hmm-gdb database\footnote{Download from http://visionlab.uta.edu/shape\_data.htm (accessed on $13$ Nov, $2018$).}. The scans are first centered and arbitrarily labeled. Picky-ICP is then used to estimate the correspondences between successive pairs of scans. Finally, the scans are registered using our method. The results are shown in Figure \ref{planedata}.   Notice that the scans do not match perfectly. This is expected since the scans are originally deformed and also because we use only rigid transforms. Nevertheless, the overall matching appears to be reasonably good.

\begin{figure}
	  \center
	      \subfloat[Plane.]{\includegraphics[width=0.35\linewidth]{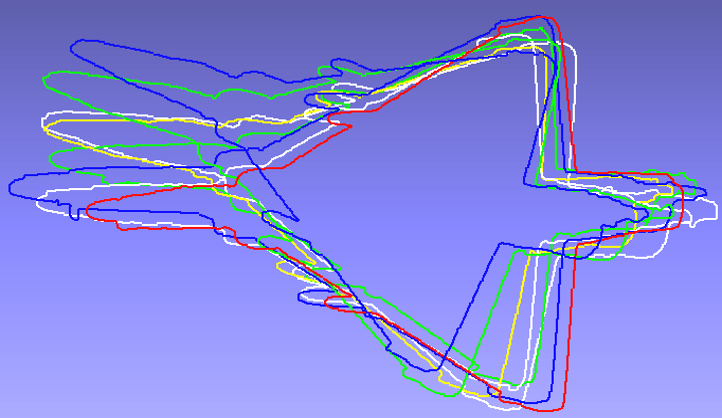}}
	      \subfloat[After registration.]{\includegraphics[width=0.35\linewidth]{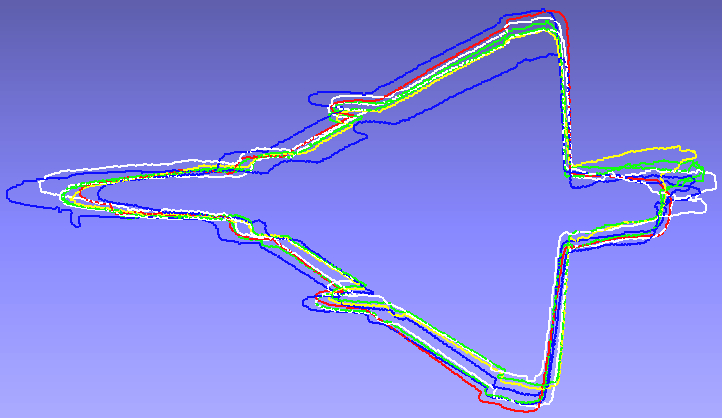}} \\
	      \subfloat[Car.]{\includegraphics[width=0.4335\linewidth]{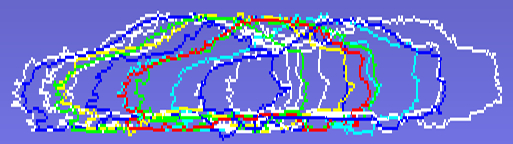}}  
	      \subfloat[After registration.]{\includegraphics[width=0.4\linewidth]{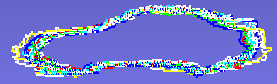}}  \\
	      \subfloat[Bicego.]{\includegraphics[width=0.445\linewidth]{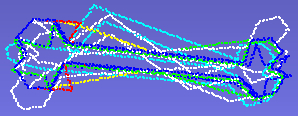}} 
	      \subfloat[After registration.]{\includegraphics[width=0.42\linewidth]{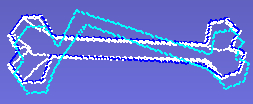}}
	      \caption{Registration of 2D profiles using our method.} 
	  \label{planedata}
\end{figure}

\section{Multiview Registration}
\label{sec:MR}

We next use the proposed algorithm for the registration of 3D scans (\cite{Sharp2002}). 
The important consideration here is that the point-to-point correspondences need to be estimated from the scan data.
We propose to use two-scan registration for the same which is discussed next.

\subsection{Correspondence Estimation}
\label{CF}

As mentioned, the scans extracted from a model are represented using a mesh (\cite{3Dscanrep}).
We treat each scan as a point cloud, where the points are simply the mesh vertices.
We first determine which pairs of scans overlap and the correspondences between them. 
This information is supplied to our registration algorithm.
Note that, while we determine the correspondences in a pairwise manner, the registration algorithm takes all the pairwise correspondences into account. To find pairwise correspondences, we can use  ICP or its fast variant (\cite{Besl1992,Rusinkiewicz2001}).
However, we noticed in our experiments that these are sensitive to outliers. 
After trying different methods, we found that the correspondences obtained using Picky-ICP \cite{Zinsser2003} give good reconstructions for our registration algorithm. In ICP, multiple points from scan $\mathcal{P}_i$ are often assigned to a single point from some target scan $\mathcal{P}_j$. 
However, the correspondences between $\mathcal{P}_i$  and $\mathcal{P}_j$ should ideally be one-to-one. 
In Picky-ICP, the correspondences between $\mathcal{P}_i$  and $\mathcal{P}_j$ are first estimated as in ICP. Multiple assignments are then resolved by selecting the point in $\mathcal{P}_i$ (among several candidates) that is closest to the matching point in $\mathcal{P}_j$ (ties are randomly broken). Let $(d_k)$ be the distances between corresponding points and $s$ be their standard deviation. Pairs for which $d_k$ is within a certain factor of $s$ are considered as overlapping points, and the remaining points are discarded. In our case, we set the factor as three. We simply used PickyICP as a black box and have not engineered anything on our own. Needless to say, if a better method is used for  finding correspondences, the performance of our registration is expected to improve.

\subsection{Comparisons}
\label{exp}

\begin{table}
\centering
\caption{Comparison of rotation error  and Matlab timings (in seconds) on an Intel quad-core $3.4$ GHz machine with $32$ GB  memory}.
\resizebox{0.7\textwidth}{!}{\begin{tabular}{|l|l|l|l|l|l|l|l|l|}
\hline
\multicolumn{1}{|c|}{\multirow{2}{*}{Dataset}} & \multicolumn{2}{c|}{LRS}                                 & \multicolumn{2}{c|}{JRMPC}                               & \multicolumn{2}{c|}{MAICP}                               & \multicolumn{2}{c|}{Our method}                          \\ \cline{2-9} 
\multicolumn{1}{|c|}{}                         & \multicolumn{1}{c|}{rot err} & \multicolumn{1}{c|}{time} & \multicolumn{1}{c|}{rot err} & \multicolumn{1}{c|}{time} & \multicolumn{1}{c|}{rot err} & \multicolumn{1}{c|}{time} & \multicolumn{1}{c|}{rot err} & \multicolumn{1}{c|}{time} \\ \hline
Bunny                                          & 2.293                        & 0.2                       & 2.384                        & 93.8                      & 1.689                        & 562.0                     & 1.723                        & 0.2                       \\ \hline
Armadillo                                      & 0.309                        & 0.1                       & 1.47                         & 93.4                      & 0.008                        & 521.3                     & 0.051                        & 0.2                       \\ \hline
Buddha                                         & 0.314                        & 0.3                       & 0.926                        & 373.1                     & 0.093                        & 637.4                     & 0.054                        & 0.3                       \\ \hline
Dragon                                         & 0.615                        & 0.2                       & 0.883                        & 150.6                     & 0.015                        & 719.6                     & 0.027                        & 0.2                       \\ \hline
\end{tabular}}
\label{fullcleanscanall}
\end{table}

We report results on four datasets from four datasets: \textit{Bunny} (\cite{bunny}), \textit{Buddha}, \textit{Dragon} (\cite{buddhandragon}) and \textit{Armadillo} (\cite{armadillo}). We also compare with  some recent methods for multiview registration: MAICP (\cite{Pooja2014}), LRS (\cite{Arrigoni2016}), and JRMPC (\cite{Evangelidis2014}). These methods have already been demonstrated to perform better than \cite{Sharp2002,Torsello2011,Benjemaa1999,Bennamoun2001,Bernard2015}. Codes for LRS and JRMPC are available online and that of MAICP was provided by the authors. All the competing methods were run using default parameters. 

	  \begin{figure*}
	  \center
	      \subfloat[\textit{Armadillo}.]{\includegraphics[width=0.3\linewidth]{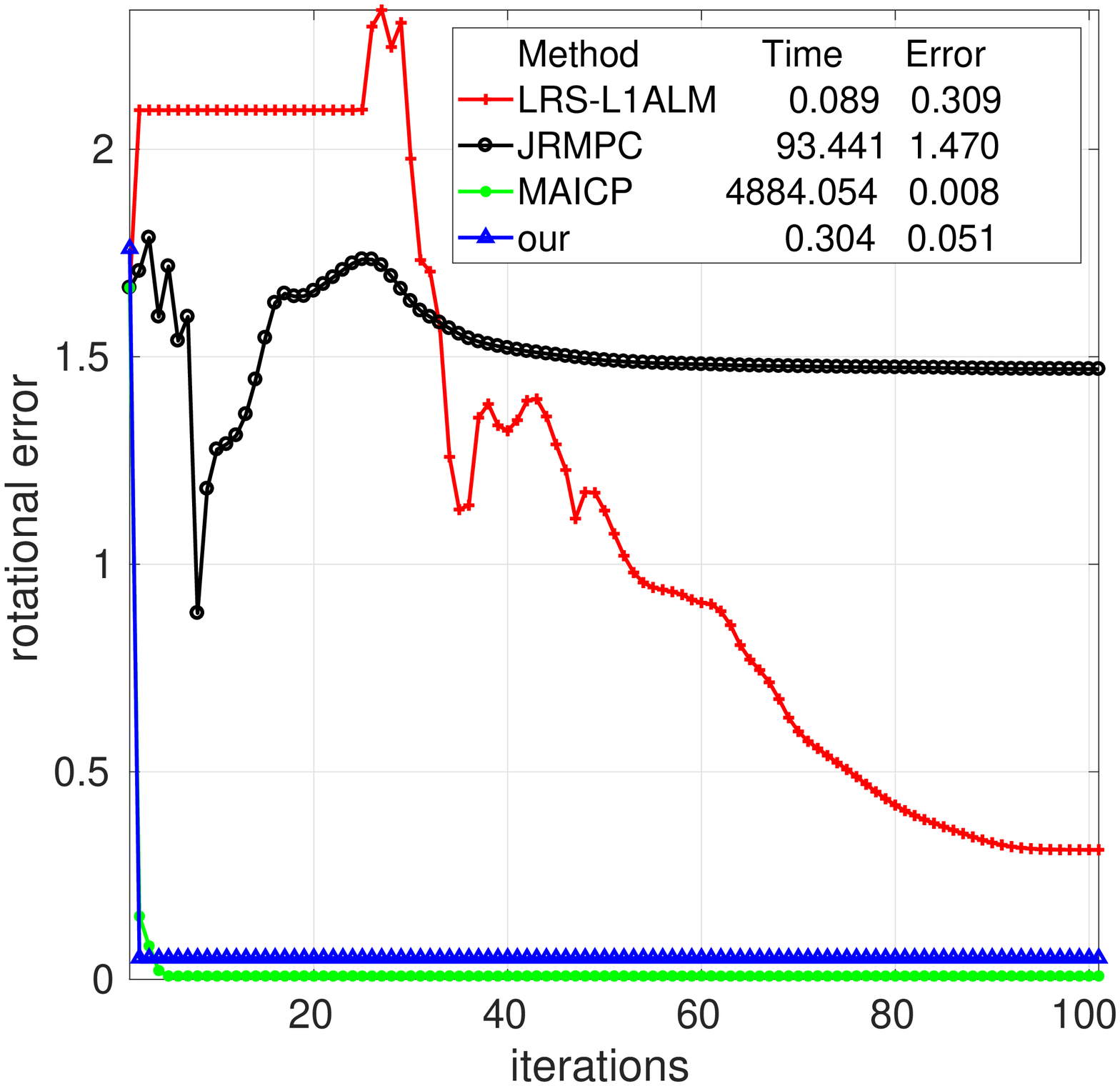}} \vspace{0.2em}
	      \subfloat[\textit{Buddha}.]{\includegraphics[width=0.3\linewidth]{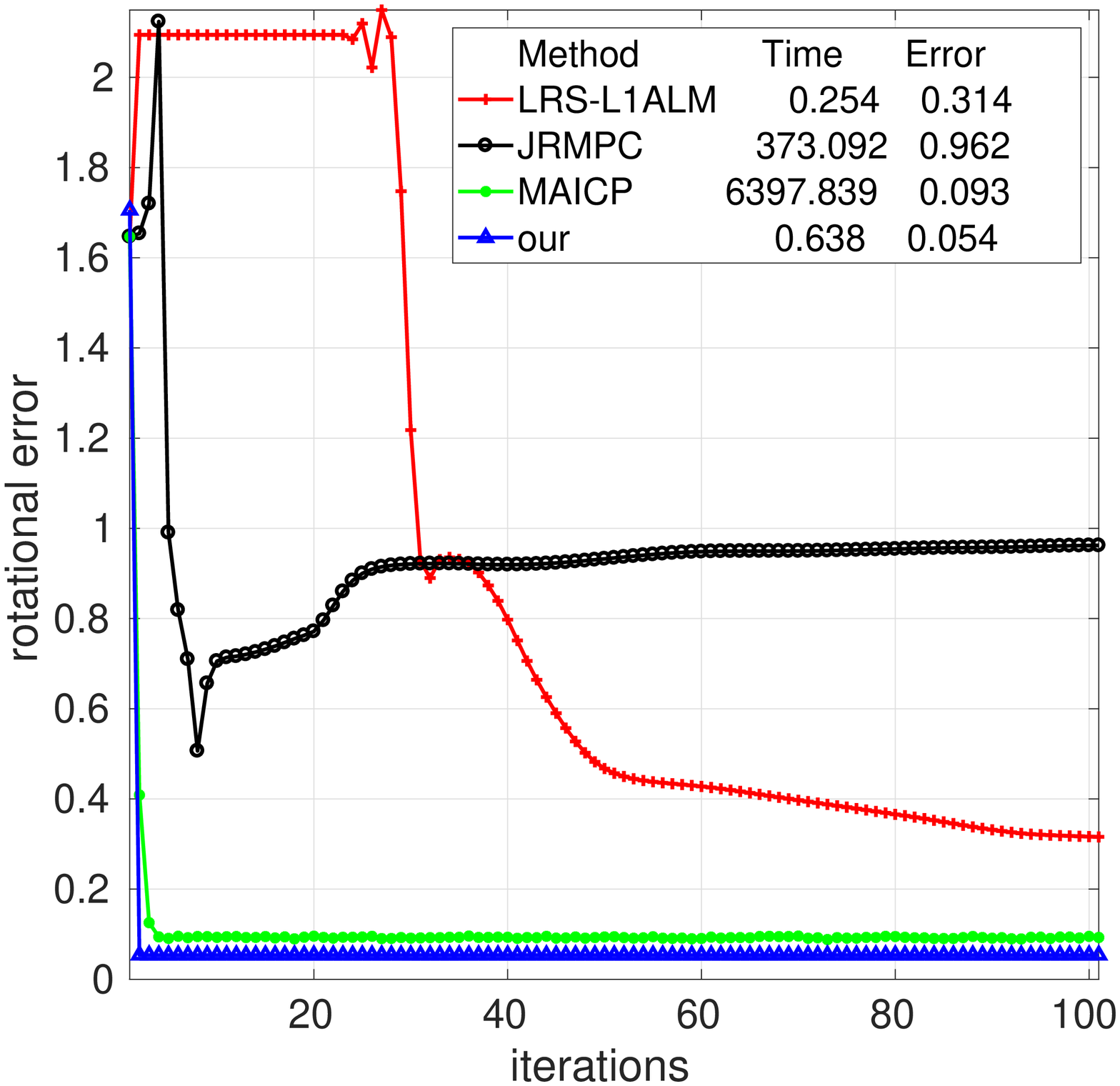}}      
	      
	      \subfloat[\textit{Dragon}.]{\includegraphics[width=0.3\linewidth]{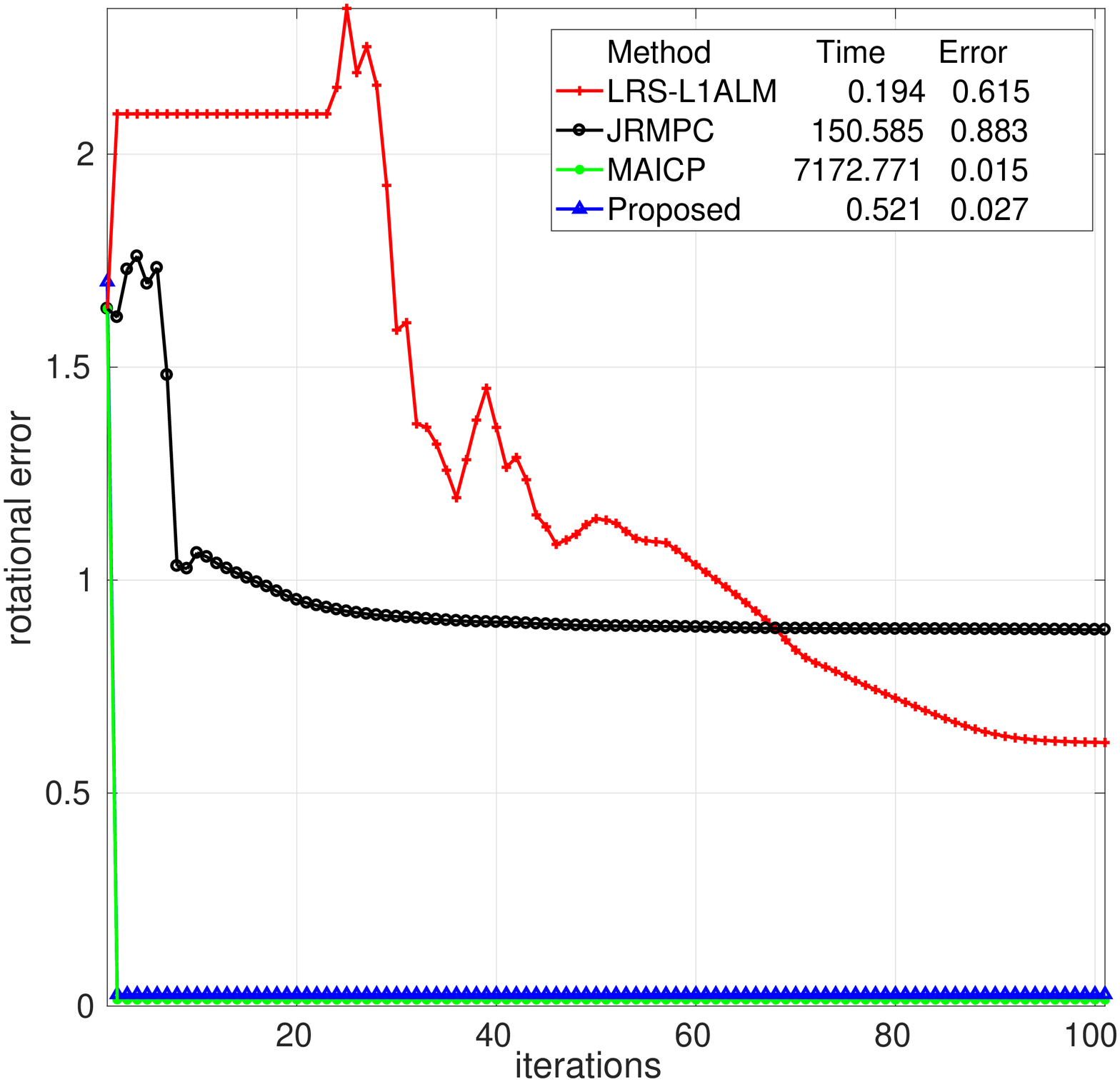}}  \vspace{0.2em}
	      \subfloat[\textit{Bunny}.]{\includegraphics[width=0.3\linewidth]{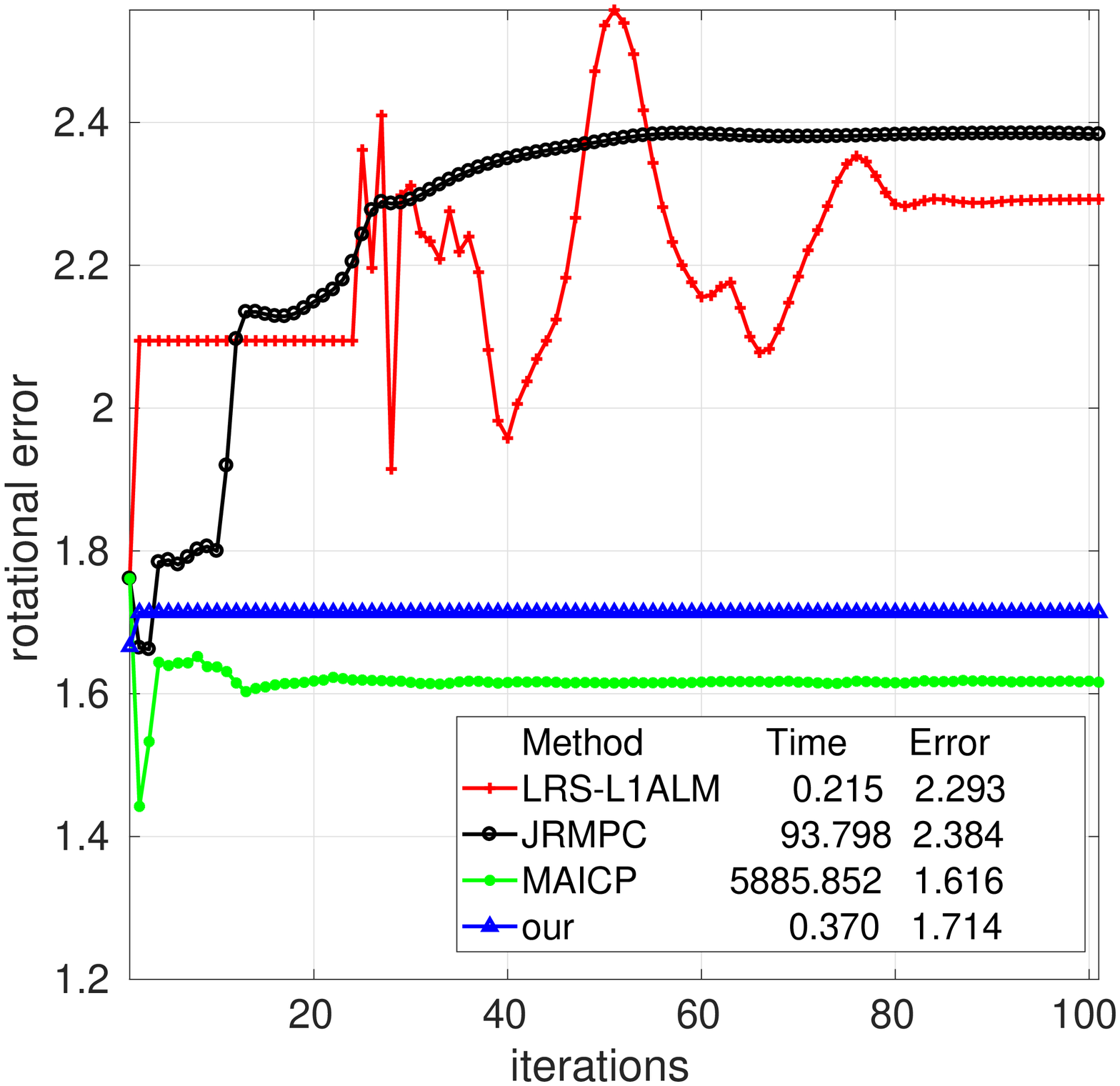}} 
	      \caption{Evolution of rotation error with iterations. Compared to LRS and MAICP, the rotation error for our method and MAICP falls off rapidly. But the execution time for our method is significantly lower than MAICP. Note that MAICP and our algorithm are run for $100$ iterations here for comparison. Consequently the time of execution are also reported. However, in practice, both MAICP and our method needs to be run for fewer iterations as reported in Table \ref{fullcleanscanall}.} 
	  \label{errperiter}
	  \end{figure*}  
	  
\underline{\textbf{Comparison 1}}. We first compare the reconstructions using the scans from Stanford dataset. For a fair comparison, we have initialized all methods using PickyICP \cite{Zinsser2003}. The reconstruction accuracy is assessed using the error metric in \eqref{reconErr}. The results are presented in Table \ref{fullcleanscanall}. It is evident that our method performs much better than LRS and JRMPC. The performance is generally comparable to MAICP. Note that the execution time for our method is significantly less. This aspect is especially important when registering several scans. 
For a visual comparison, the cross-sections of the reconstructions are compared in Table \ref{FRC}. Since the rotation errors for LRS and JRMPC are large, we have only shown the cross-sections for MAICP and our method in Table \ref{FRC}. The number of scans and angular differences are as follows: \textit{Bunny} ($12$ scans, $30$ degree), \textit{Armadillo} ($12$ scans, $30$ degree), \textit{Buddha} ($15$ scans, $24$ degree), and \textit{Dragon} ($15$ scans, $24$ degree). Notice that the cross-section for \textit{Buddha} is much better for our reconstruction. Following JRMPC, LRS and MAICP, we have tried comparing the convergence rate (of the rotation error) using the following protocol:
	 \begin{enumerate}
	  \item read the full scans.
	  \item initialize all methods using PickyICP.
	  \item run LRS, JRMPC, MAICP and our method on the scans.
	  \item record the rotation error at each iteration.
	 \end{enumerate}
	  The results are shown in Figure \ref{errperiter}. Notice that the rotation error decreases quickly for our method and MAICP. The rotation error at convergence is comparable for MAICP and our method. However, the timing is significantly lower in our case. Note that $100$ iterations were used for all methods for comparison, though fewer iterations are required for our method and MAICP (cf. Figure \ref{errperiter}). Thus only $10$ iterations were run MAICP and proposed method in Table \ref{fullcleanscanall}.
	  
	  \begin{figure*}[!htpb]
	  \center
	      \subfloat[\textit{Scans}.]{\includegraphics[width=0.275\linewidth]{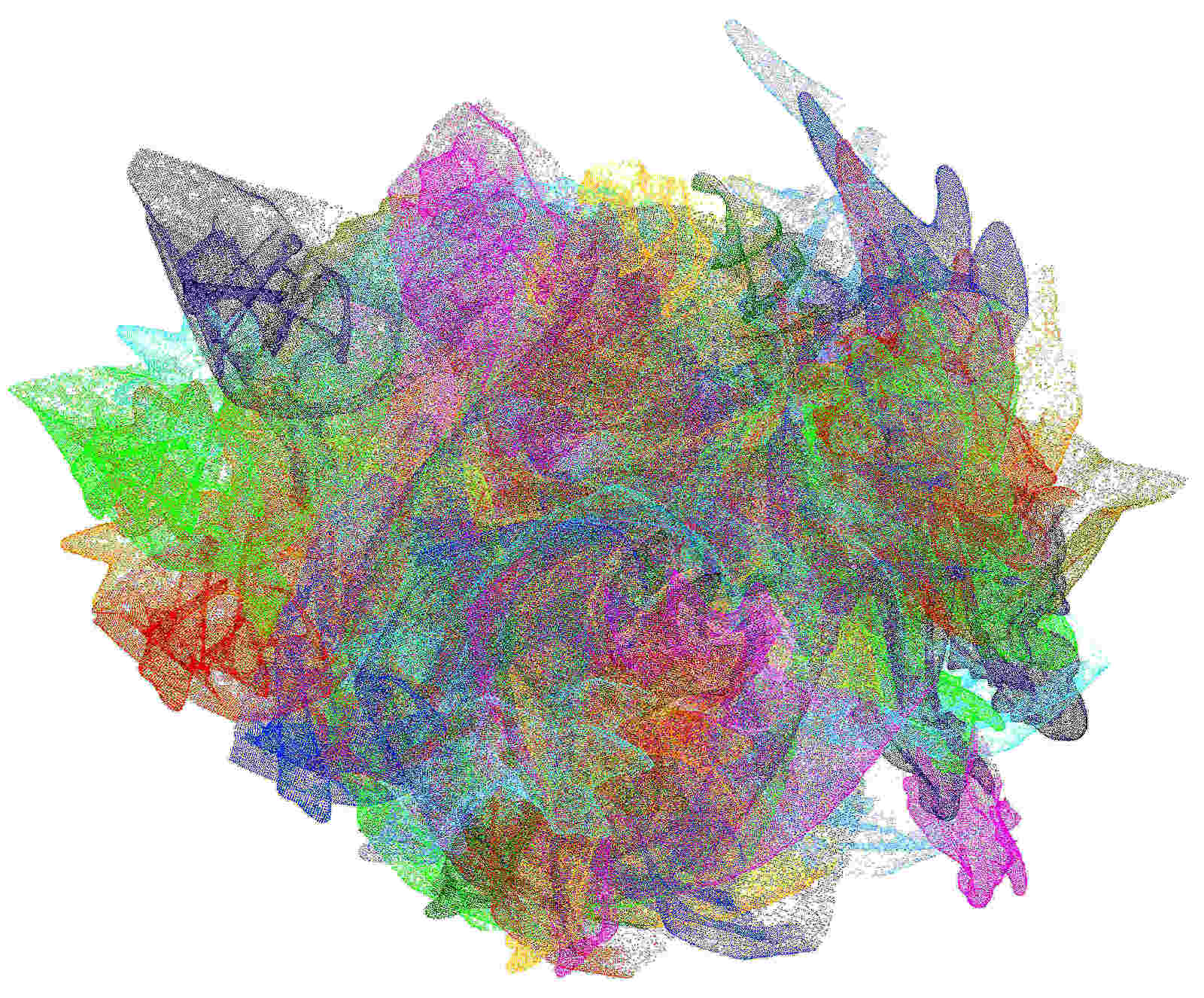}} \hspace{1mm}
	      \subfloat[\textit{MAICP}.]{\includegraphics[width=0.26\linewidth]{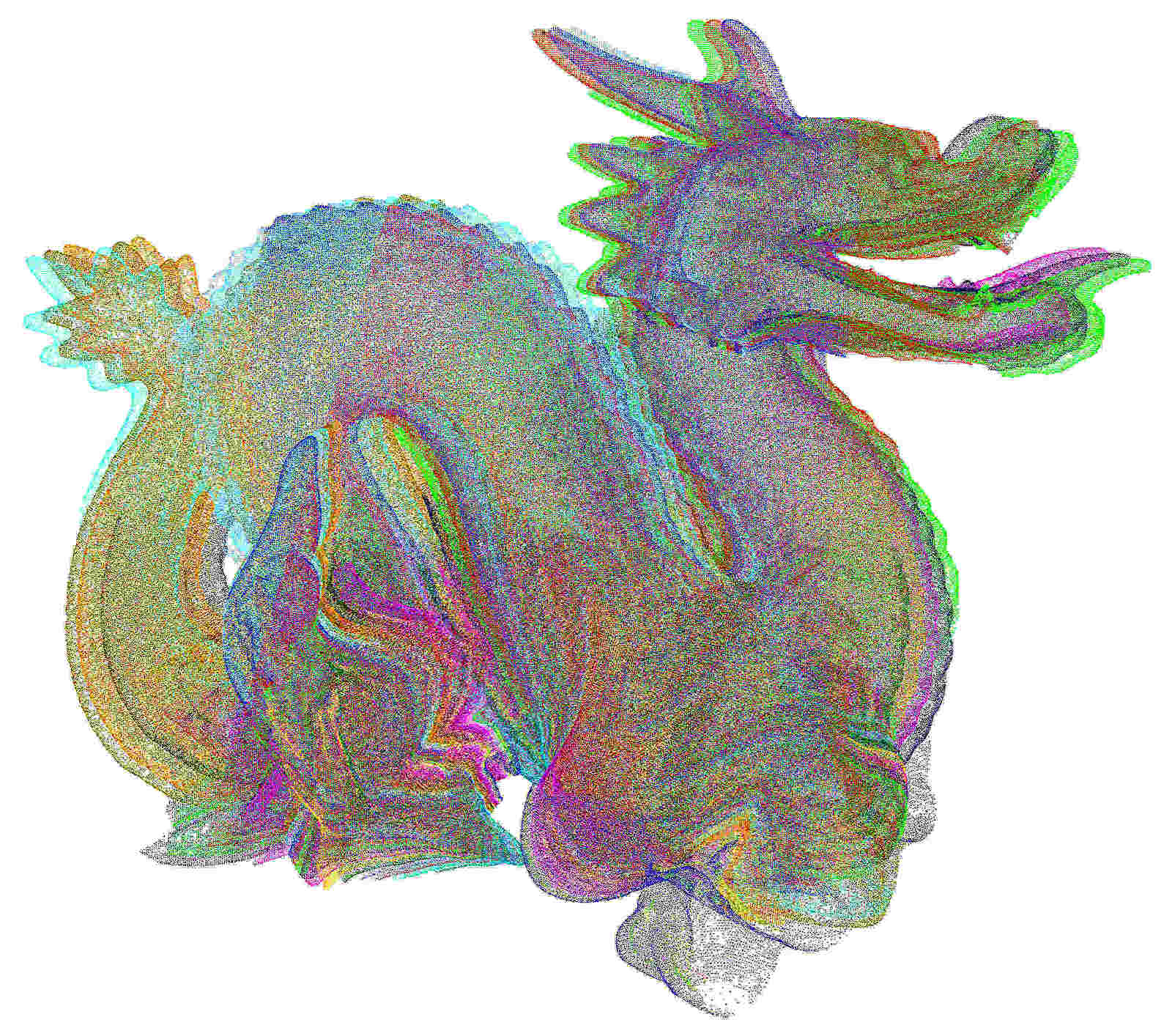}} \hspace{1mm}
	      \subfloat[\textit{Proposed}.]{\includegraphics[width=0.325\linewidth]{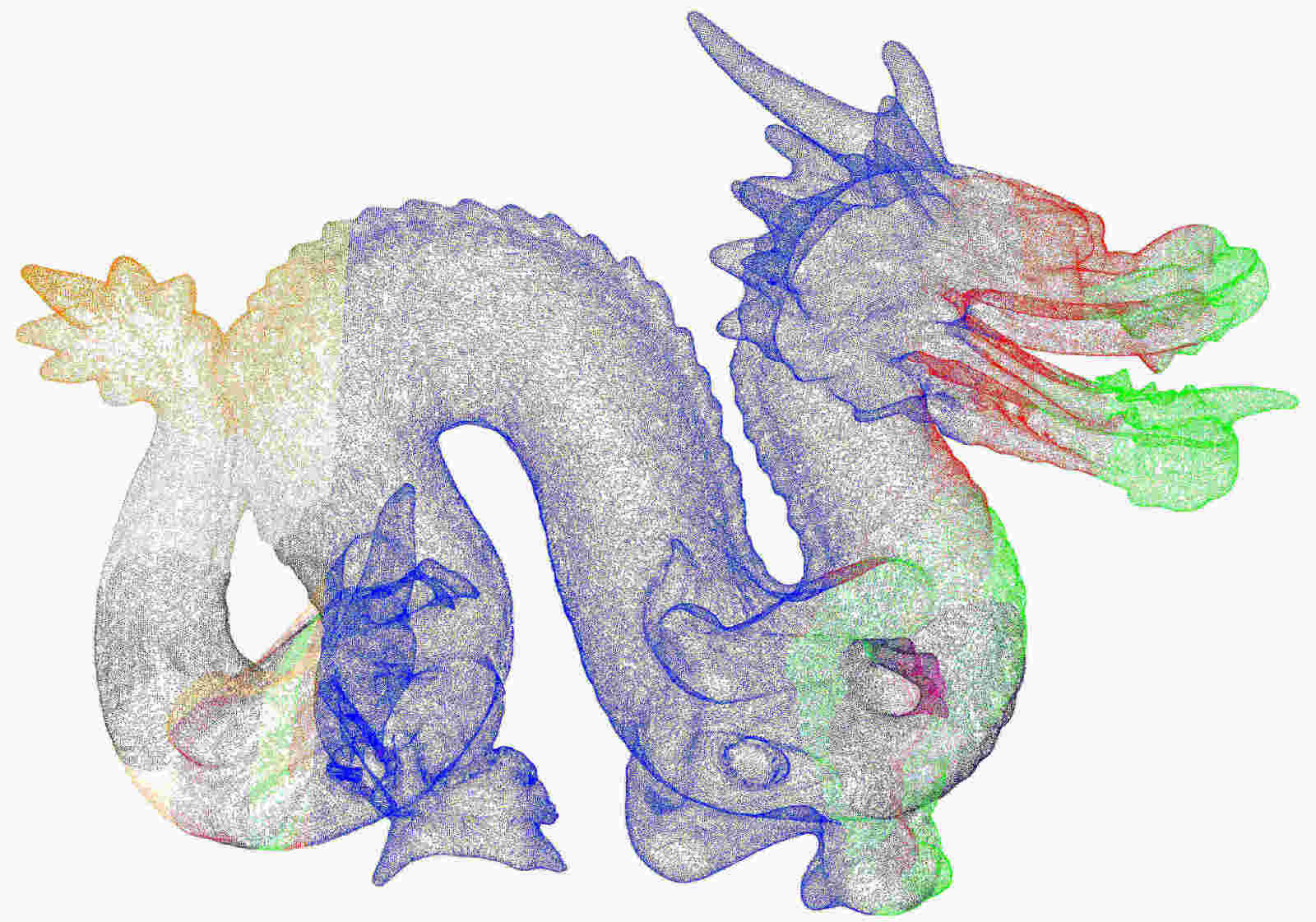}} \\      
	      \subfloat[\textit{Scans}.]{\includegraphics[width=0.24\linewidth]{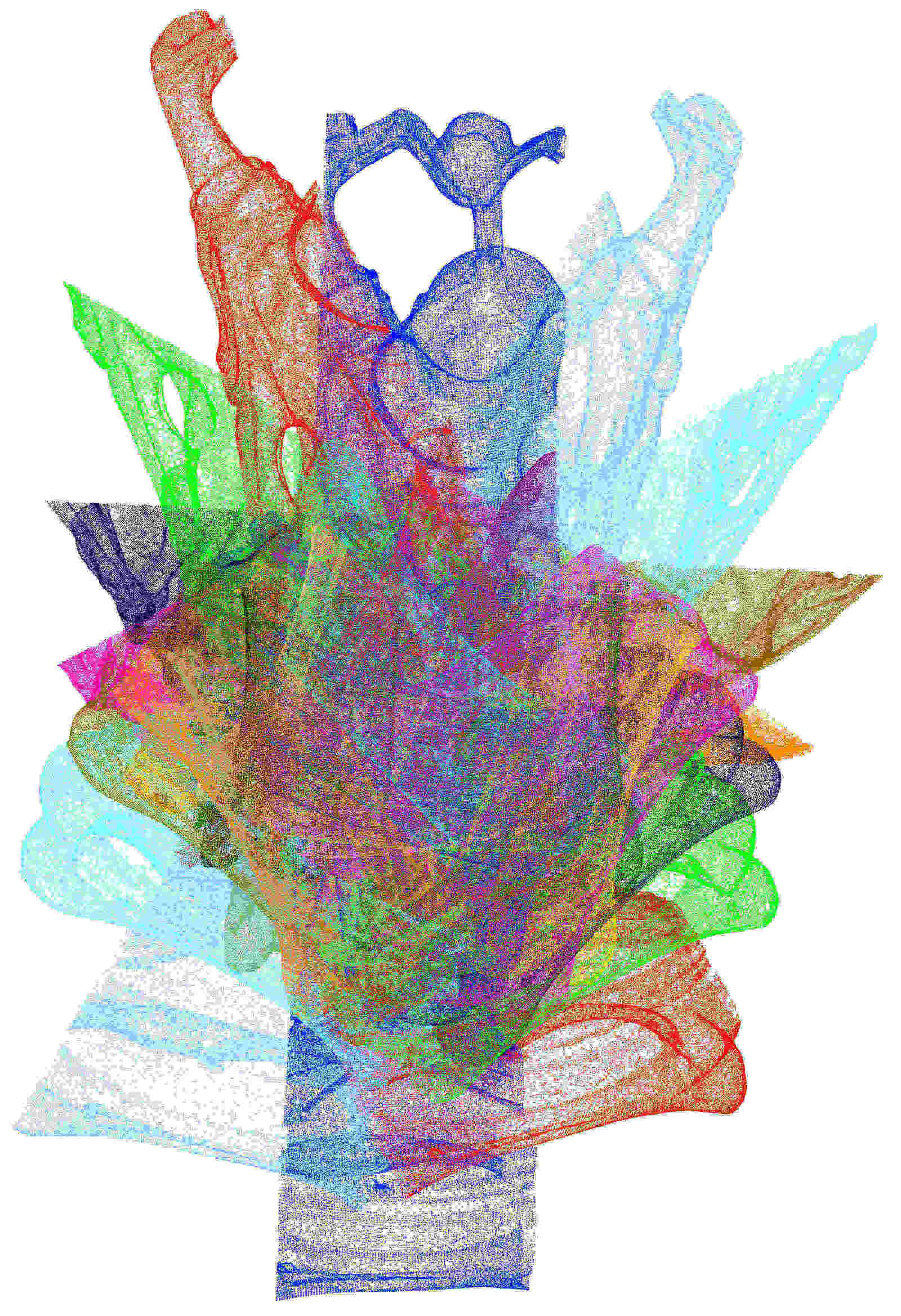}}\hspace{1mm}
	      \subfloat[\textit{MAICP}.]{\includegraphics[width=0.154\linewidth]{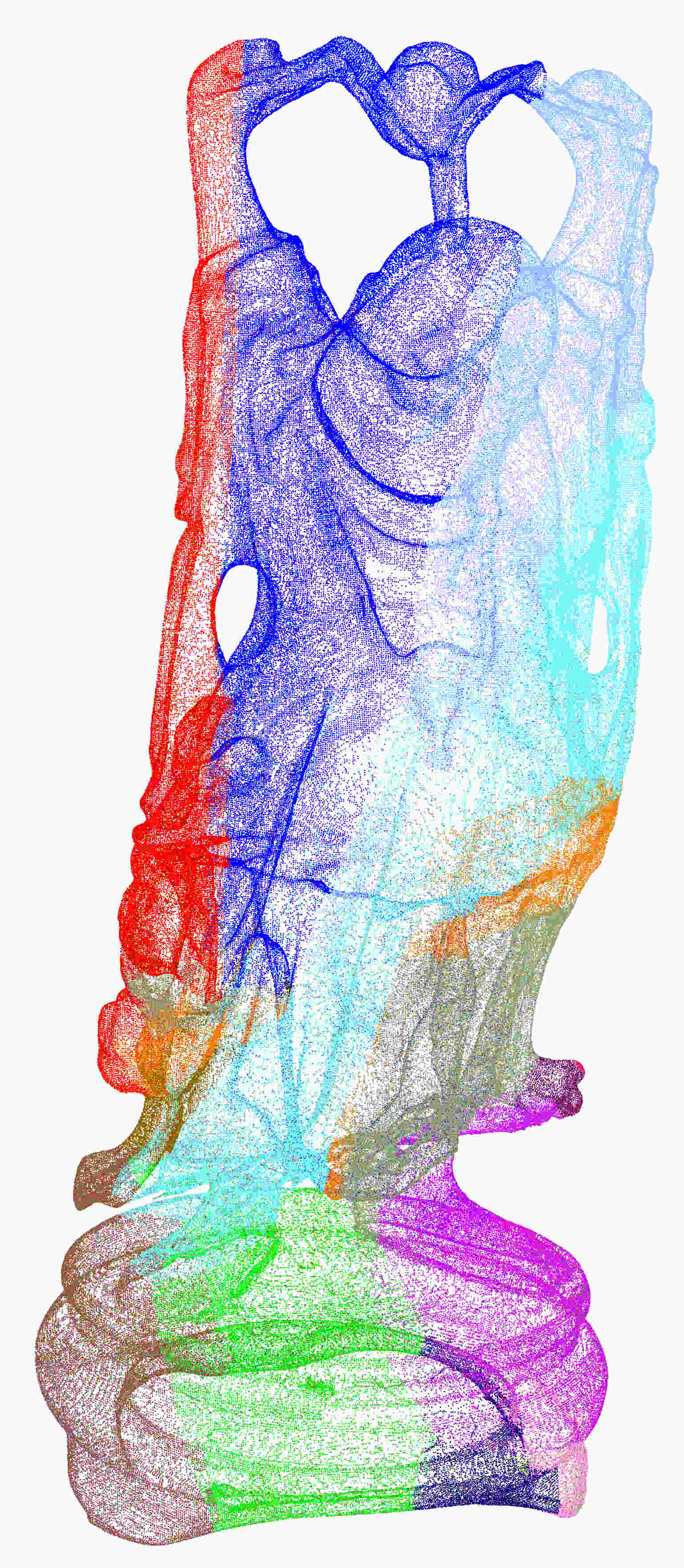}} \hspace{1mm}
	      \subfloat[\textit{Proposed}.]{\includegraphics[width=0.154\linewidth]{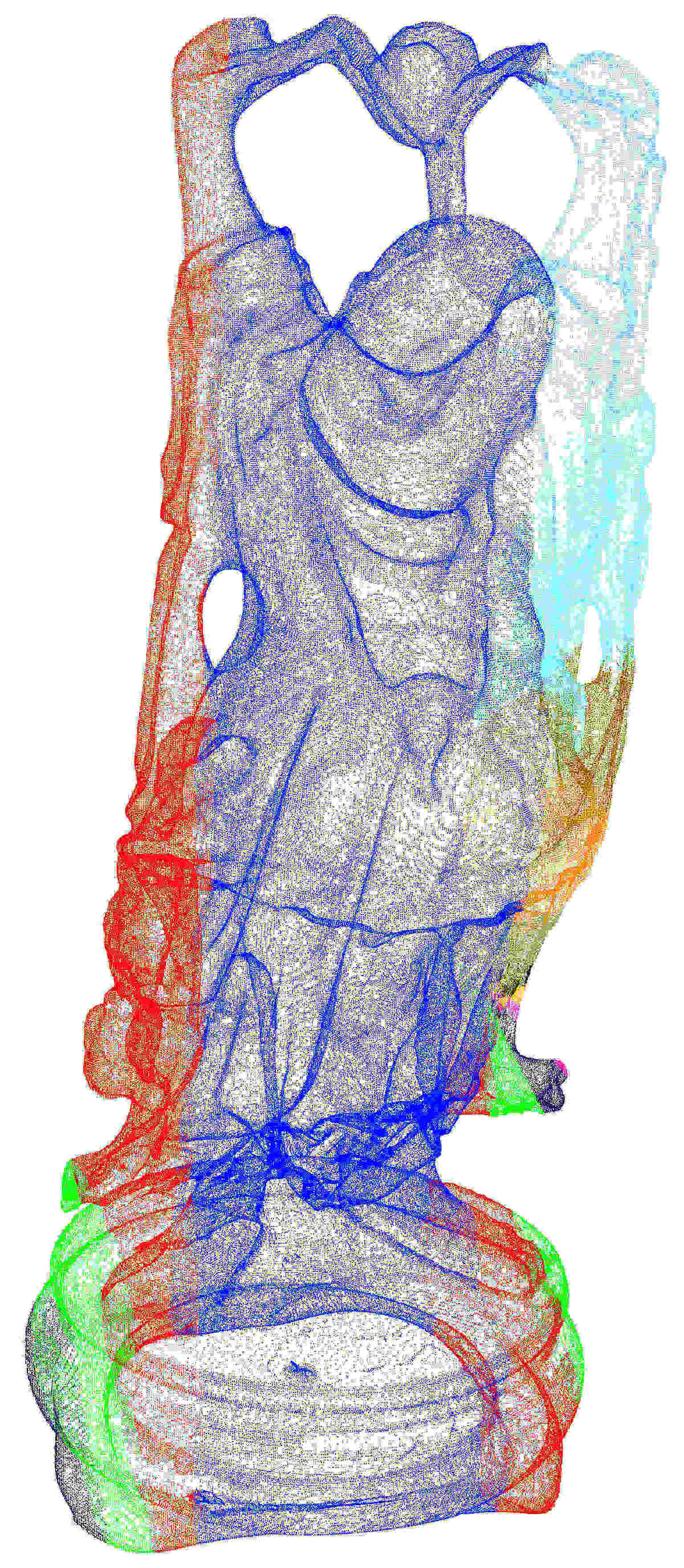}} 	      
	      \caption{Registration using full scans: $24$ scans were randomly generated from \textit{Dragon} and \textit{Buddha} . Rotation between the scans is $15$ degrees. Each scan was randomly rotated about an arbitrary axis in the range $[0\mbox{-}1]$ degrees. The points in each scan were further perturbed with Gaussian noise (standard deviation: $0.01$ and $0.03$ respectively). Notice that the reconstruction  is better for \textit{Dragon} compared to MAICP. This is evident  from the scans colors---the individual scans do not overlap well for MAICP. This suggests that the scans can still be transformed for better alignment. On the other hand, the reconstructions are almost identical for \textit{Buddha}. However, we are much faster.} 
	  \label{randomdata1}
	  \end{figure*}  
	  
\underline{\textbf{Comparison 2}}.  To test robustness, we perturb the individual scans by Gaussian noise and register the scans with all the methods to test for robustness. Some typical results are reported in Figure \ref{FR}. The performance of the proposed method is generally comparable to MAICP. But again, the timings are significantly less for our method.	 
 
\underline{\textbf{Comparison 3}}. Finally, we carry out experiments with random noise (in the rotations and coordinates) using the following protocol: 
	  \begin{enumerate}\label{genrandom}
	   \item read the model.
	   \item center the model by subtracting its centroid and rotate it about the $x$-axis by angles $\theta, 2\theta, \ldots$ for some fixed $\theta$ (this mimics the collection of point sets using a turn table). For each rotation, the points above the $x$-$y$ plane are formed into a point set.
	   \item  perturb the coordinates with additive Gaussian noise.
	  \end{enumerate}
	  The trials are carried out $25$ times for each noise setting and the resulting errors are averaged.  
	  The angle of rotation $\theta$ determines the overlap between successive scans. For all experiments, we fixed $\theta$ to be $15$ degrees.  We generate $24$ scans for \textit{Dragon} and \textit{Buddha}. Standard deviation for Gaussian noise is $0.01$, while the true rotations are randomly perturbed in the range $[0\mbox{-}1]$ degrees. Reconstruction results from MAICP and our method are shown in Figure \ref{randomdata1} for one of the trials. PickyICP is used for initialization in both methods. 
	  The averaged rotation error is slightly better for \textit{Dragon} in our case (our: $1.39$, MAICP: $1.45$). The visual quality also appears to be better for our method (see Figure \ref{randomdata1}). The timings are $1.49$ sec and $23.2$ min for our method and MAICP. 
	  The results are almost identical for \textit{Buddha} (see  Figure \ref{randomdata1}). Average rotation errors for our method and MAICP are $1.486$ and $1.485$. However, the timing is significantly lower for our method (our: $1.39$ sec, MAICP: $21.9$ min).

	  	\begin{figure*}
	  \center
	      \subfloat[\textit{Scans}.]{\includegraphics[width=0.185\linewidth]{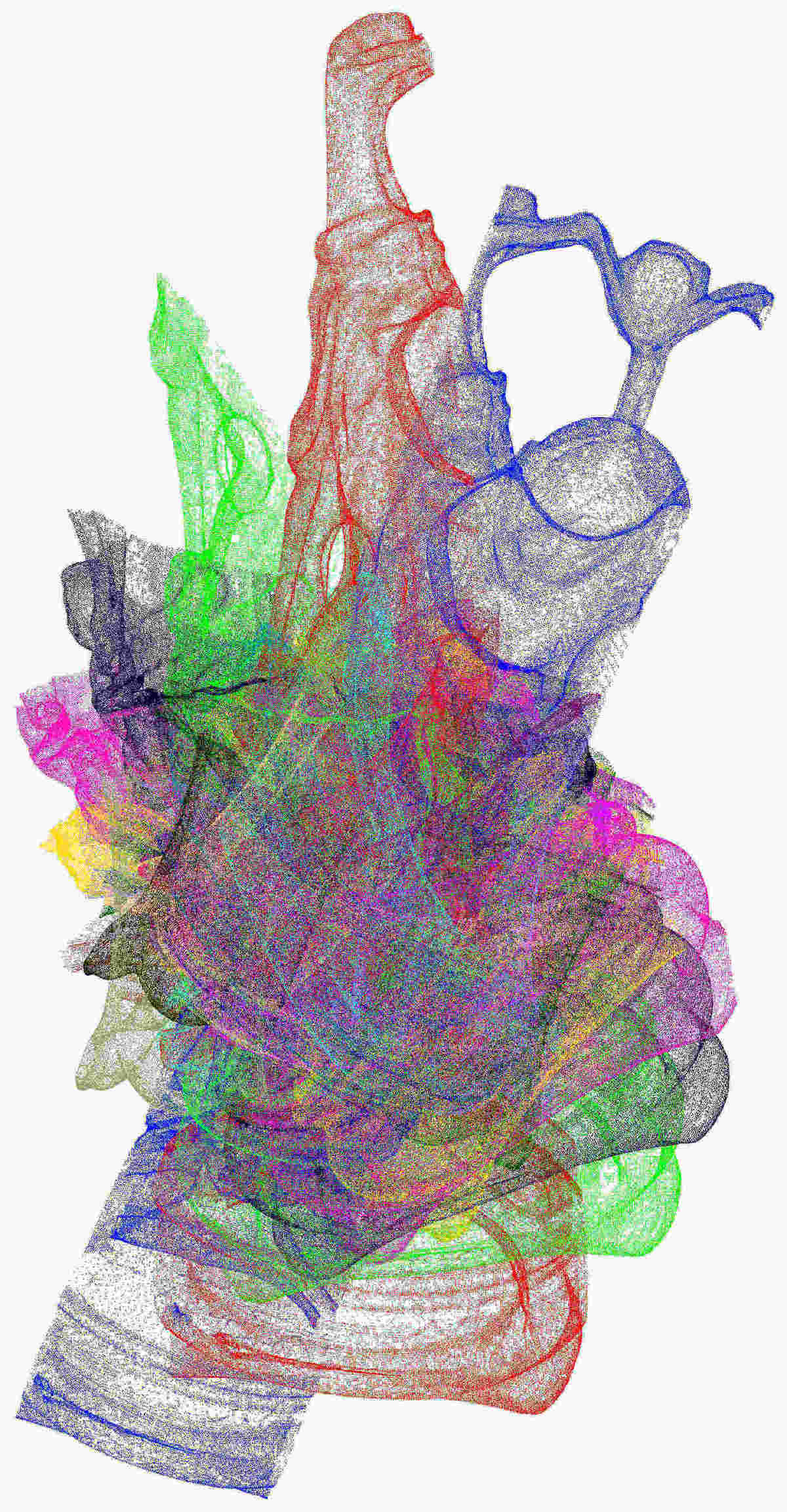}} \hspace{1mm}
	      \subfloat[\textit{MAICP}.]{\includegraphics[width=0.185\linewidth]{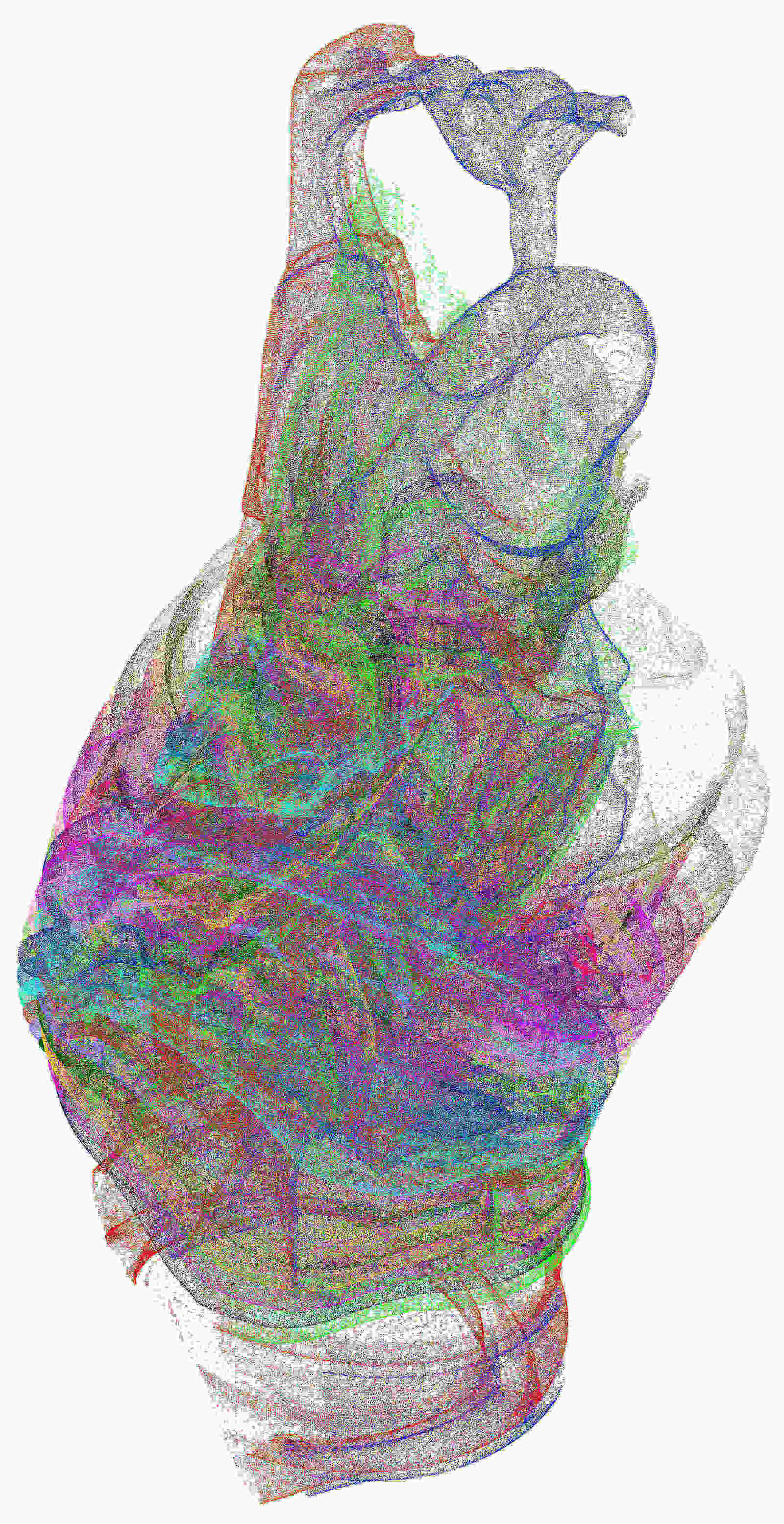}} \hspace{1mm}
	      \subfloat[\textit{Proposed}.]{\includegraphics[width=0.14\linewidth]{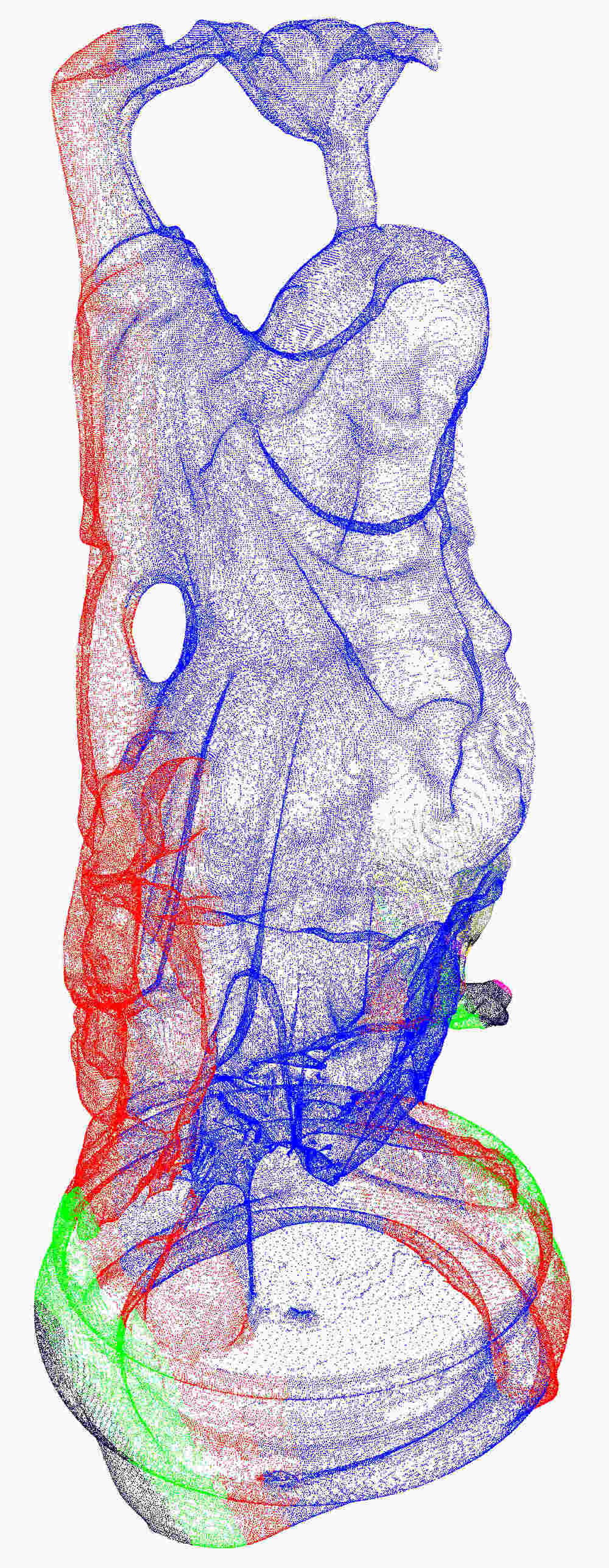}}  \\
	      \subfloat[\textit{Scans}.]{\includegraphics[width=0.3\linewidth]{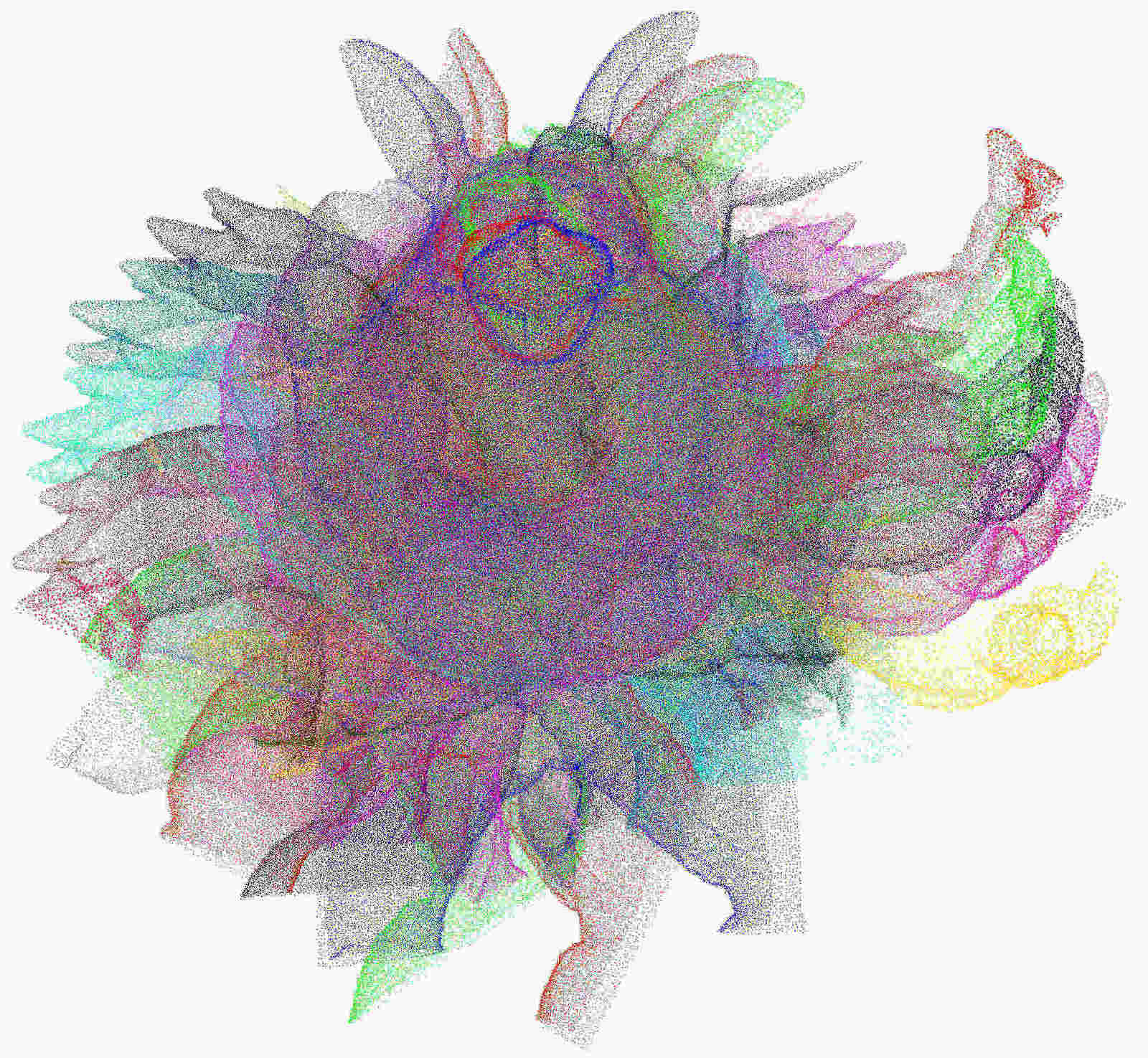}} \hspace{1mm}
	      \subfloat[\textit{MAICP}.]{\includegraphics[width=0.275\linewidth]{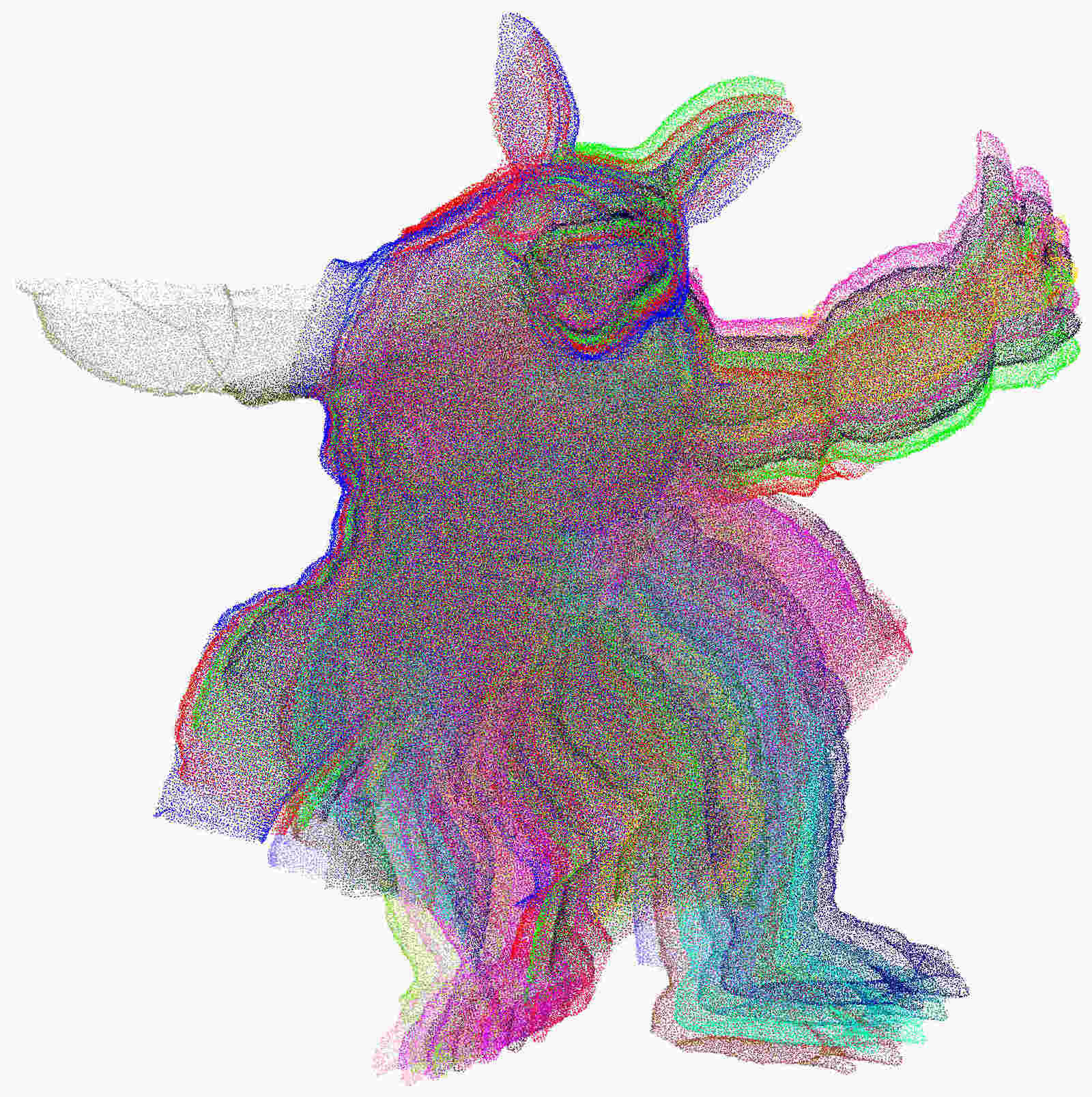}} \hspace{1mm}
	      \subfloat[\textit{Proposed}.]{\includegraphics[width=0.24\linewidth]{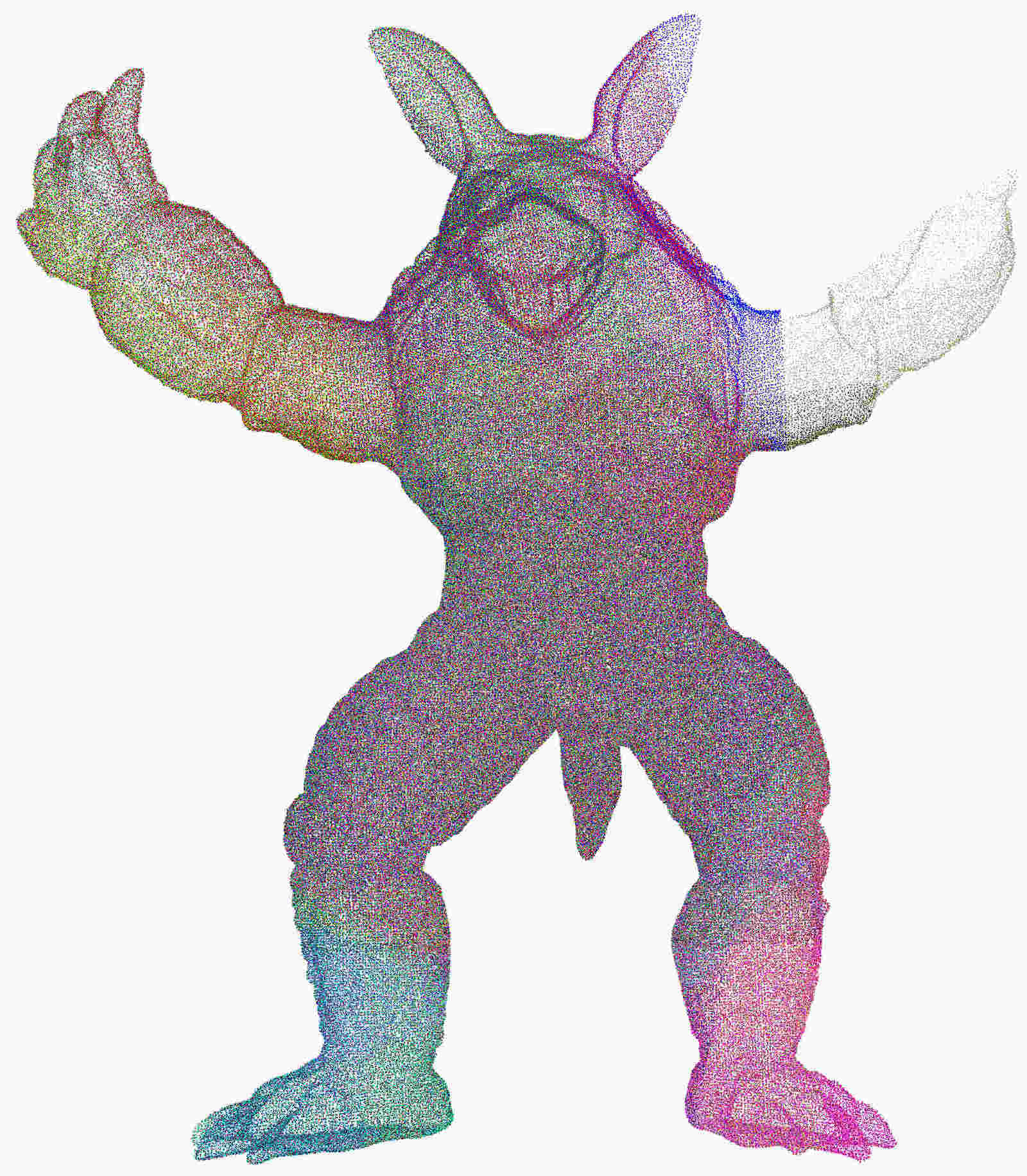}} 
	      
	      \caption{Registration using partial scans: $20$ scans were randomly generated from \textit{Buddha} and \textit{Armadillo}. The rotation between scans is $15$ degrees. Each scan was randomly rotated around an arbitrary axis in the range $[0\mbox{-}2]$ degrees. Standard deviation of Gaussian noise for \textit{Armadillo} is $0.03$. No translational noise is added to \textit{Buddha} to test the reconstruction with just rotation noise. The scans are registered using MAICP and our method. The reconstructions are visually better for the our method.} 
	  \label{randomdata2}
	  \end{figure*}

	We next simulate scenarios where scans covering the entire model may not be available. For this purpose, we generate $20$ scans from \textit{Buddha} and \textit{Armadillo}. The scans for \textit{Armadillo} are perturbed with Gaussian noise having standard deviation $0.03$. No Gaussian noise is added for \textit{Buddha} in order to verify reconstruction with just rotation noise. Each scan is randomly rotated in the range $[0\mbox{-}2]$ degrees. As earlier, we use $25$ trials and average the rotation error. PickyICP is used to initialize both methods. 
	Average rotation errors for \textit{Buddha} are $1.9$ and $2.2$ for our method and MAICP. For \textit{Armadillo}, the corresponding rotation errors are $1.59$ and $1.63$ respectively. The registration results for one of the trials is shown in Figure \ref{randomdata2}. Notice that our reconstruction result is much better than MAICP for both models. The  possible reason for this is that MAICP uses correspondences between the first and and the last scan. Since full scans are not considered, there are very few overlapping points between them.  Again, our method is significantly faster: $1.06$ sec and $0.43$ sec for \textit{Buddha} and \textit{Armadillo} compared to $18.4$ min and $17.6$ min for MAICP.

\section{Conclusion}
\label{sec:Conc}

We proposed a novel optimization algorithm for registering multiple points sets in a globally consistent fashion using rotations and translations.
We empirically analyzed the algorithm and showed that it works well on both simulated and real scan data.
An intriguing property of the proposed ADMM solver is that it converges to the global minimum with a good initialization. 
We could certify this for cases where the global minimum can be computed by some other means, namely, for clean data (where the minimum is zero) and for two scans. 
We applied the proposed algorithm for 2D shape matching and 3D multiview registration. For the later we compared its performance with recent methods. 
The reconstruction accuracy of the proposed algorithm is comparable to MAICP but we are much faster. The results also suggest that the overall algorithm is quite robust to noise.

\section{Appendix}

\subsection{Proof of Theorem \ref{thm1}}

It follows from ($\mathrm{P3}$) that $\mathrm{rank}(\G) \geq d$. Along with ($\mathrm{P2}$), we conclude that $\mathrm{rank}(\G) = d$. 
Therefore, using ($\mathrm{P1}$) and the spectral theorem for symmetric matrices, we can write 
\begin{equation*}
\G= \lambda_1 \boldsymbol{u}_1 \boldsymbol{u}_1^\top + \cdots +  \lambda_d \boldsymbol{u}_d \boldsymbol{u}_d^\top ,
\end{equation*}
where $\lambda_i > 0$ and $\boldsymbol{u}_1,\ldots,\boldsymbol{u}_d$ is a orthonormal basis of $\mathbb{R}^{dm}$. Define $\R \in \mathbb{R}^{d \times dm}$ to be
\begin{equation*}
\R = \big[ \surd \lambda_1 \boldsymbol{u}_1 \quad \cdots \quad \surd \lambda_d \boldsymbol{u}_d \big]^\top, 
\end{equation*}
and let $\R = [ \R_1 \cdots \R_m ]$, where $\R_i \in  \mathbb{R}^{d \times d}$.  By construction, $\G=\R^\top \R$ and, in particular,
\begin{equation*}
\G_{ij}=\R_i^\top\! \R_j \qquad (i,j \in [m]).
\end{equation*}
Therefore, it follows from ($\mathrm{P3}$) that $\R_i ^\top \R_i = \mathbf{I}$. Furthermore, we conclude from ($\mathrm{P4}$) that for $i=1,\ldots, m-1$,
\begin{equation*}
\mathrm{det} (\G_{i,i+1})=\mathrm{det} (\R_i) \mathrm{det}(\R_{i+1}) =1.
\end{equation*}
We deduce that for $i \in [m]$, $\mathrm{det} (\R_i) =1$ or $\mathrm{det} (\R_i) =-1$. In the latter case, we simply pick a global reflection $\mathbf{Q}$ with $\mathrm{det}(\mathbf{Q})=-1$, and reassign $\R_i \leftarrow \mathbf{Q} \R_i$. This gives us the desired $\R_1,\ldots,\R_m \in \mathbb{SO}(d)$ such that \eqref{defG} holds.

\subsection{Proof of Theorem \ref{thm3}}

We can write $\A= \mathbf{U} \Lambda \mathbf{U}^\top$, where 
\begin{equation*}
\Lambda = \mathrm{diag}(\lambda_1,\ldots,\lambda_{dm}) \ \text{ and }  \ \mathbf{U} = [\boldsymbol{u}_1 \cdots \boldsymbol{u}_{dm}].
\end{equation*}
Let $\Psi$ denote matrices in $\mathbb{S}_+^{dm}$ with rank at most $d$. Any $\X \in \Psi$ can be represented as $\X = \bV \Gamma \bV^\top$,  
\begin{equation*}
\Gamma= \mathrm{diag}(\mu_1,\ldots,\mu_{dm})   \ \text{ and }  \ \bV \in \mathbb{O}(dm),
\end{equation*}
where $\mu_1 \geq \cdots \geq \mu_{dm} \geq 0$ and at most $d$ of these are positive. 

Note that $\lVert  \X - \A \rVert_{\mathrm{F}} = \lVert  \mathbf{K} \Gamma  \mathbf{K}^\top -  \Lambda \rVert_{\mathrm{F}}$, where $\mathbf{K} = \bU^\top \bV$. As a result, 
\begin{equation}
\label{red}
\underset{\X \in \Psi}{\text{min}} \ \lVert  \X - \A \rVert_{\mathrm{F}}^2 = \underset{\Gamma \in \Phi,\mathbf{K} \in  \mathbb{O}(d)}{\text{min}} \  \lVert  \mathbf{K} \Gamma  \mathbf{K}^\top -  \Lambda \rVert_{\mathrm{F}}^2,
\end{equation}
where $\Phi$ denotes non-negative diagonal matrices with rank at most $d$. For fixed $\Lambda \in \Phi$, it can be shown that the minimum over $\mathbf{K} \in  \mathbb{O}(d)$ is attained when $\mathbf{K} = \mathbf{I}$, that is, when $\bV = \bU$. In particular, this reduces \eqref{red} to
\begin{equation}
\label{OPT}
 \underset{\Gamma \in \Phi}{\text{min}} \  \lVert  \Gamma - \ \Lambda \rVert_{\mathrm{F}}^2 =  \underset{\Gamma \in \Phi}{\text{min}} \ \sum_{i=1}^{dm} (\mu_i - \lambda_i)^2,
\end{equation}
The minimizer of \eqref{OPT} is given by the projection $\Gamma = \Pi_{\Phi}(\Lambda)$, that is, $\mu_i = \max(0,\lambda_i) \text{ for }1 \leq i \leq d$, and $\mu_i=0 \text{ for } d+1 \leq i \leq dm$.

\subsection{Proof of Theorem \ref{thm4}}

The projection problem in question is
\begin{equation*}
\Pi_{\Omega}(\A) = \underset{\X \in \mathbb{SO}(d)}{\text{argmin}} \ \lVert \X - \A \rVert_{\mathrm{F}}^2 =  \underset{\X \in \mathbb{SO}(d)}{\text{argmax}} \ \langle \A , \X \rangle,
\end{equation*}
where we have used the fact that $\X \in \mathbb{SO}(d)$. Since $\mathbb{SO}(d)$ is compact, there exists $\X_0 \in \mathbb{SO}(d)$ such that
\begin{equation*}
 \underset{\X \in \mathbb{SO}(d)}{\text{max}} \ \langle \A , \X \rangle =  \langle \A , \X_0 \rangle = \mathrm{trace}(\A^\top \! \X_0).
\end{equation*}
We claim that $\mathbf{P}=\A^\top \! \X_0 \in \mathbb{S}^d$. Indeed, consider an arbitrary anti-symmetric matrix $\mathbf{M} \in \mathbb{R}^{d \times d}$ such that $\mathbf{M}^\top =- \mathbf{M}$, and define
\begin{equation*}
f(t) = \mathrm{trace}(\mathbf{P} e^{t \mathbf{M}}) \qquad (t \in \mathbb{R}).
\end{equation*}
Since $e^{t \mathbf{M}} \in \mathbb{SO}(d)$, it follows that for all $t \in \mathbb{R}$,
\begin{equation*}
f(0) = \mathrm{trace}(\mathbf{P}) \geq f(t).
\end{equation*}
Hence, 
\begin{equation}
\label{deriv}
f'(0) =  \mathrm{trace}(\mathbf{P} \mathbf{M})=0.
\end{equation}
Since \eqref{deriv} holds for any anti-symmetric $\mathbf{M}$, it easily follows that $\mathbf{P} \in \mathbb{S}^d$. 

Having shown that $\mathbf{P}$ is symmetric, let $\mathbf{P} = \mathbf{Q} \Lambda \mathbf{Q}^\top$, where $\Lambda$ is a diagonal matrix and $\mathbf{Q} \in \mathbb{O}(d)$. Then 
\begin{equation*}
\mathbf{Q} \Lambda^2 \mathbf{Q} =   \mathbf{P}^2 = \X_0^\top \! \A \A^\top \! \X_0 = \mathbf{K} \mathrm{diag}(\sigma_1,\ldots,\sigma_d )^2 \mathbf{K}^\top,
\end{equation*}
where $\mathbf{K} = \X_0^\top \! \bU \in \mathbf{O}(d)$. Since $\mathbf{P}$ and $\mathbf{P}^2$ commute, they can be diagonalized in the same basis. In particular, 
\begin{equation*}
\mathbf{P} = \mathbf{K}  \mathrm{diag}(s_1\sigma_1,\ldots,s_d\sigma_d )  \mathbf{K}^\top, \qquad (s_i \in \{-1,1\}).
\end{equation*}
In terms of this representation, we have
\begin{equation}
\label{opt}
 \mathrm{trace}(\A^\top \! \X_0) =  \mathrm{trace}(\mathbf{P}) = \sum_{i=1} s_i \sigma_i.
\end{equation}
Since $\sigma_i \geq 0$,  \eqref{opt} is maximum when each $s_i=1$. However, since $\X_0 \in \mathbb{SO}(d)$, it follows that
\begin{equation*}
\prod_{i=1}^d s_i \cdot \prod_{i=1}^d \sigma_i  =   \mathrm{det}(\mathbf{P}) =  \mathrm{det}(\A)  \mathrm{det}(\X_0) = \mathrm{det}(\bU \bV^\top) \prod_{i=1}^d \sigma_i .
\end{equation*}
Therefore, if $\mathrm{det}(\A) \neq 0$, we must have 
\begin{equation}
\label{prod}
\prod_{i=1}^d s_i  = \mathrm{det}(\bU \bV^\top) \in \{-1,1\},
\end{equation}
 since $\bU \bV^\top \in \mathbb{O}(d)$. If $\mathrm{det}(\bU \bV^\top)=1$, then \eqref{opt} is maximum when $s_i=1$ for all $i$. However, if $\mathrm{det}(\bU \bV^\top)=-1$, then it follows from \eqref{prod} that the maximizer is $s_i=1 \text{ for } 1 \leq i \leq d-1$, and $s_d=-1$ (this is where we use the fact that $\sigma_i \geq \sigma_d$ for all $i$). The case $\mathrm{det}(\A) = 0$ can be worked out similarly.

\bibliographystyle{model2-names}
\bibliography{citations}

\end{document}